\documentclass[twoside]{article}
\usepackage[accepted]{aistats2018}
\usepackage[utf8]{inputenc} 
\usepackage[T1]{fontenc}    
\usepackage{url}            
\usepackage{booktabs}       
\usepackage{amsfonts}       
\usepackage{nicefrac}       
\usepackage{microtype}      
\usepackage{natbib}
\usepackage{hyperref}

\usepackage{amssymb,amsmath,amsthm,amsfonts}

\newtheorem{theorem}{Theorem}
\newtheorem{lemma}{Lemma}
\newtheorem{definition}{Definition}
\newtheorem{proposition}{Proposition}

\usepackage{graphicx}
\usepackage{amsmath}
\usepackage{subcaption}
\usepackage{algorithm}
\usepackage{algcompatible}
\usepackage{xcolor}

\newcommand{\tbeta}{\tilde{\beta}}
\newcommand{\A}{\mathcal{A}}
\newcommand{\tA}{\tilde{A}}
\newcommand{\X}{\mathcal{X}}
\newcommand{\R}{\mathbb{R}}
\newcommand{\tx}{\widetilde{x}}
\newcommand{\ty}{\widetilde{y}}
\newcommand{\bbeta}{\bar{\beta}}

\newcommand{\E}{\mathcal{E}}
\newcommand{\B}{\mathcal{B}}

\begin{document}

%

%

\twocolumn[

\aistatstitle{Random Warping Series: A Random Features Method for Time-Series Embedding}


\aistatsauthor{ 
Lingfei Wu \And Ian En-Hsu Yen \And  Jinfeng Yi}

\aistatsaddress{ 
IBM Research \And  Carnegie Mellon University \And IBM Research}

\aistatsauthor{ 
Fangli Xu \And Qi Lei \And  Michael J. Witbrock}

\aistatsaddress{ 
College of William and Mary \And University of Texas at Austin \And IBM Research }

]

\begin{abstract}
Time series data analytics has been a problem of substantial interests for decades, and Dynamic Time Warping (DTW) has been the most widely adopted technique to measure dissimilarity between time series. A number of global-alignment kernels have since been proposed in the spirit of DTW to extend its use to kernel-based estimation method such as support vector machine. However, those kernels suffer from diagonal dominance of the Gram matrix and a quadratic complexity w.r.t. the sample size. In this work, we study a family of alignment-aware positive definite (p.d.) kernels, with its feature embedding given by a distribution of \emph{Random Warping Series (RWS)}. The proposed kernel does not suffer from the issue of diagonal dominance while naturally enjoys a \emph{Random Features} (RF) approximation, which reduces the computational complexity of existing DTW-based techniques from quadratic to linear in terms of both the number and the length of time-series. We also study the convergence of the RF approximation for the domain of time series of unbounded length. Our extensive experiments on 16 benchmark datasets demonstrate that RWS outperforms or matches state-of-the-art classification and clustering methods in both accuracy and computational time. Our code and data is available at { \url{https://github.com/IBM/RandomWarpingSeries}}.
\end{abstract}

\section{Introduction}

Over the last two decades, time series classification and clustering have received considerable interests in many applications such as 
genomic research \citep{leslie2002spectrum}, 
image alignment \citep{peng2015piefa,peng2015circle}, 
speech recognition \citep{cuturi2007kernel, shimodaira2001dynamic},
and motion detection \citep{li2011time}. 
One of the main challenges in time series data stems from the fact that there are no explicit features in sequences \citep{xing2010brief}. Therefore, a number of feature representation methods have been proposed recently, among which the approaches deriving features from phase dependent intervals \citep{deng2013time, baydogan2013bag}, phase independent shapelets \citep{ye2009time, rakthanmanon2013fast}, and dictionary based bags of patterns \citep{senin2013sax, schafer2015boss} have gained much popularity due to their highly competitive performance \citep{bagnall2016great}. However, since the aforementioned approaches only consider the local patterns rather than global properties, the effectiveness of these features largely depends on the underlying characteristics of sequences that may vary significantly across applications. More importantly, these approaches may typically not be a good first choice for large scale time series due to their quadratic complexity in terms both of the number $N$ and (or) length $L$ of time series. 

Another family of research defines a distance function to measure the similarity between a pair of time series. Although Euclidean distance is a widely used option and has been shown to be competitive with other more complex similarity measures \citep{wang2013experimental}, various elastic distance measures designed to address the temporal dynamics and time shifts are more appropriate \citep{xing2010brief, kate2016using}. Among them, dynamic time warping (DTW) \citep{berndt1994using} is the standard elastic distance measure for time series. Interestingly, an 1NN classifier with DTW has been demonstrated as the gold standard benchmark, and has been proved difficult to beat consistently \citep{wang2013experimental, bagnall2016great}. Recently, a thread of research has attempted to directly use the pair-wise DTW distance as features \citep{hayashi2005embedding, gudmundsson2008support, kate2016using, QiYi2016}. However, the majority of these approaches have quadratic complexity in both number and length of time series in terms of the computation and memory requirements. 

Despite the successes of various explicit feature design, kernel methods have great promise for learning non-linear models by implicitly transforming a simple representations into a high-dimension feature space \citep{rahimi2007random, chen2016efficient, wu2016revisiting, yen2014sparse}. The main obstacle for applying kernel method to time series is largely due to two distinct characteristics of time series, (a) variable length; and (b) dynamic time scaling and shifts. Since elastic distance measures, such as DTW, take into account these two issues, there have been several attempts to apply DTW directly as a similarity measure in a kernel-based classification model \citep{shimodaira2001dynamic, gudmundsson2008support}. Unfortunately, the DTW distance does not correspond to a \emph{valid positive-definite } (p.d.) kernel and thus direct use of DTW leads to an indefinite kernel matrix that neither corresponds to a loss minimization problem nor giving a convex optimization problem \citep{bahlmann2002online, cuturi2007kernel}. 
To overcome these difficulties, a family of global alignment kernels have been proposed by taking softmax over all possible alignments in DTW to give a p.d. kernel \citep{cuturi2007kernel, cuturi2011fast, marteau2015recursive}. However, the effectiveness of the global alignment kernels is impaired by the diagonal dominance of the resulting kernel matrix. Also, the quadratic complexity in both the number and length of time series make it hard to scale. 

In this paper, inspired by the latest advancement of  kernel learning methodology from distance \citep{wu2018d2ke}, we study \emph{Random Warping Series} (RWS), a generic framework to generate vector representation of time-series, where we construct a family of p.d. kernels from an explicit feature map given by the DTW between original time series and a distribution of random series. To admit an efficient computation of the kernel, we give a random features approximation that uniformly converges to the proposed kernel using a finite number of random series drawn from the distribution. The RWS technique is fully parallelizable, and highly extensible in the sense that the building block DTW can be replaced by recently proposed elastic distance measures such as CID \citep{batista2014cid} and DTDC \citep{gorecki2014non}. With a number $R$ of random series, RWS can substantially reduce the computational complexity of existing DTW-based techniques from $O(N^2L^2)$ to $O(NRL)$ and memory consumption from $O(NL + N^2)$ to $O(NR)$. We also extend existing analysis of random features to handle time series of unbounded length, showing that $R=\Omega(1/\epsilon^2)$ suffices for the uniform convergence to $\epsilon$ precision of the exact kernel. We evaluate \emph{RWS} on 16 real-world datasets on which it consistently outperforms or matches state-of-the-art baselines in terms of both testing accuracy and runtime. In particular, RWS often achieves orders-of-magnitude speedup over other methods to achieve the same accuracy. 


\section{DTW and Global Alignment Kernels}
We first introduce the widely-used technique DTW and nearest-neighbor DTW (1NN-DTW), and then illustrate the existing global alignment kernels for time series and their disadvantages. 

\textbf{Time Series Alignment and 1NN-DTW.}
Let $\X$ be the domain of input time series, and $\{x_i\}_{i=1}^N$ be the set of time series, where the length of each time series $|x_i| \leq L$, taking numeric values in $\R$. A special challenge in time series lies in the fact that the series could have different lengths, and a signal could be generated with time shifts and different scales, but with a similar pattern. To take these factors into account, an alignment (also called a warping function) is often introduced to provide a better distance/similarity measure between two time series $x_i = (x_i^1, \ldots, x_i^n)$ and $x_j = (x_j^1, \ldots, x_j^m)$ of lengths $n$ and $m$ respectively. Specifically, an alignment $a = (a_1,a_2)$ of length $|a| = p$ between two time series $x_i$ and $x_j$ is a pair of increasing vectors $(a_1,a_2)$ such that $1 = a_1(1) \leq \ldots \leq a_1(p) = n$ and $1 = a_2(1) \leq \ldots \leq a_2(p) = m$ with unitary increments and no simultaneous repetitions. The set of all alignments between $x_i$ and $x_j$ is defined as $\A(x_i,x_j)$. In the literature of DTW \citep{berndt1994using}, the DTW distance between $x_i$ and $x_j$ is defined as follows in its simplest form: 
\begin{equation}\label{eq:DTW_Definition}
\begin{split}
& S(x_i,x_j) = \min_{a\in\A(x_i,x_j)} \tau(x_i,x_j;a), \ \ \\
& \text{where} \ \ \tau(x_i,x_j;a) = \sum_{t=1}^{|a|} \tau(x_i(a_1(t)),x_j(a_2(t))).
\end{split}
\end{equation}
Here $\tau(x_i,x_j;a)$ is a dissimilarity measure between $x_i$ and $x_j$ under alignment $a$. Typically, \emph{Dynamic Programming} (DP) is employed to find the optimal alignment $a^*$ and then compute DTW distance. 
The dissimilarity function $\tau$ could be defined as any commonly used distance such as the squared Euclidean distance. To accelerate the computation and improve the performance, a Sakoe and Chiba band is often used to constrain the search window size for DTW \citep{sakoe1978dynamic,rakthanmanon2012searching}. 

DTW has been widely used for time series classification in combination with the 1NN algorithm, and this combination has been shown to be exceptionally difficult to beat \citep{wang2013experimental, bagnall2016great}. However, there are two disadvantages of 1NN-DTW. First, this method incurs the high computational cost of $O(N^2)$ complexity for computing DTW similarity between all pairs of time series, where each evaluation of DTW without constraints takes $O(L^2)$ computation. Second, 
Nearest-Neighbor methods often suffers from the problems of high variance.
For example, if a class label is determined by a small portion of time series, a Nearest-Neighbor identification on the basis of similarity with the whole time series will be ineffective due to noise and irrelevant information.

\textbf{Existing Global Alignment Kernels.}
To take the advantage of DTW in other prediction methods based on \emph{Empirical Risk Minimization} (ERM) such as SVM and Logistic Regression, a thread of research has been trying to derive a \emph{valid p.d. kernel} that resembles $S(x_i,x_j)$. 
A framework for designing such kernel is the \emph{time series global-alignment kernel} proposed in \citep{cuturi2007kernel} and further explored in \citep{cuturi2011fast}. The kernel replaces the minimum in \eqref{eq:DTW_Definition} with a soft minimum that sums over all possible DTW alignments between two series $x_i$, $x_j$:
\begin{equation}\label{global_align_kernel}
\begin{split}
k(x_i,x_j) & :=\sum_{a\in\A(x_i,x_j)} \exp(-\tau(x_i,x_j;a)) \\
& :=\sum_{a\in\A(x_i,x_j)}\prod_{t=1}^{|a|} \kappa(x_i[a_1(t)],x_j[a_2(t)])
\end{split}
\end{equation}
where $\kappa(.,.)$ is some local similarity function induced from the divergence $\tau$ as $\kappa = \exp(-\tau)$. The function \eqref{global_align_kernel} is a p.d. kernel when $\kappa(.,.)$ satisfies certain conditions \citep{cuturi2007kernel}. However, it is known that a soft minimum can be orders of magnitude larger than the minimum when summing over exponentially many terms, which results in a serious  \emph{diagonally dominant problem} for the kernel \eqref{global_align_kernel}. In other words, the kernel value between a series to itself $k(x_i,x_i)$ is orders of magnitude larger than other values $k(x_i,x_j)$. Thus in practice, one must take the log of the kernel \eqref{global_align_kernel} even though such operation is known to break the p.d. property \citep{cuturi2011fast}. In addition, the evaluation of kernel \eqref{global_align_kernel} requires running DP over all pairs of samples and thus gives a high complexity of $O(N^2L^2)$.

\section{Novel Time-Series Kernels via Alignments to Random Series}
\begin{figure}
\centering
\includegraphics[scale=0.35]{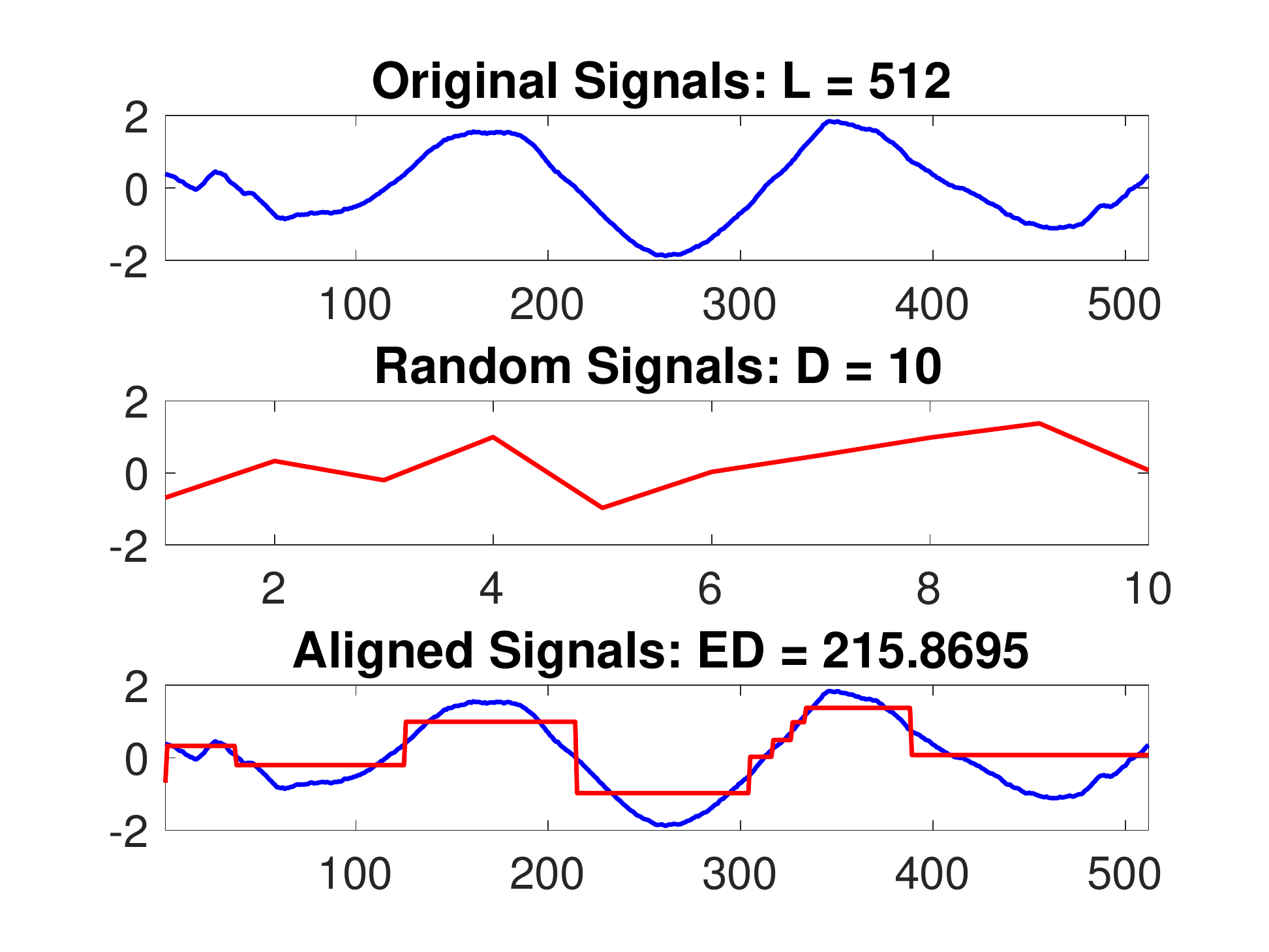}
\vspace{-2mm}
\caption{Example of the DTW alignment between original time series of length $L = 512$ and random time series of length $D = 10$.}
\label{fig:DTW_Alignment_BirdChicken}
\vspace{-4mm}
\end{figure}

In this section, we study a new approach to build a family of p.d. kernels for time series based on DTW, inspired by the latest advancement of  kernel learning methodology from distance \citep{wu2018d2ke}. 




Formally, the kernel is defined by integrating a feature map over a distribution of random time series $p(\omega)$, with each feature produced by alignments between original time series $x$ and random series $\omega$:
\begin{equation}\label{new_kernel}
\begin{split}
& k(x,y) =\int_{\omega} p(\omega) \phi_{\omega}(x) \phi_{\omega}(y) d\omega, \ \\
& \text{where} \ \ \phi_{\omega}(x):= \sum_{a\in\A(\omega, x)} p(a|\omega) \tau(\omega, x;a).
\end{split}
\end{equation}
The kernel \eqref{new_kernel} enjoys several advantages. First, \eqref{new_kernel} is a \emph{p.d. kernel} by its construction.

\begin{proposition}
The kernel \eqref{new_kernel} is positive definite, that is,
$
\sum_{i=1}^N\sum_{j=1}^N c_ic_jk(x_i,x_j) \geq 0
$
for any $\{c_i\mid c_i\in \R \}_{i=1}^N$ and any $\{x_i\mid x_i\in \X \}_{i=1}^N$.
\end{proposition}
\begin{proof}
By definition \eqref{new_kernel}, we have
\begin{align*}
\sum_{i=1}^N\sum_{j=1}^N c_ic_jk(x_i,x_j) 
& = \sum_{i=1}^N\sum_{j=1}^N c_ic_j \int_{\omega} p(\omega)  \phi_{\omega}(x_i) \phi_{\omega}(y_j) d\omega \\
& = \int_{\omega} p(\omega) \sum_{i=1}^N\sum_{j=1}^N c_ic_j\phi_{\omega}(x_i)\phi_{\omega}(x_j) d\omega \\
& = \int_{\omega} p(\omega) \left(\sum_{i=1}^N c_i\phi_{\omega}(x_i)\right)^2 d\omega\geq0.
\end{align*}
\end{proof}
\vspace{-4mm}

Secondly, by choosing 
\begin{equation}\label{mini_condprob}
p(a|\omega)=\left\{\begin{array}{ll}
1 , & a^{*}=\text{arg}\min_{a}\; \tau(\omega,x;a) \\
0 , & o.w. 
\end{array}\right.
\end{equation}
one can avoid the \emph{diagonal dominance problem} of the kernel matrix, since the kernel value between two time series $k(x,y)$ depends only on the correlation of $\tau(\omega,x;a_x)$ and $\tau(\omega,y;a_y)$ under their \emph{optimal alignments}. It thus avoids the dominance of the diagonal terms $k(x,x)$ caused by the summation over exponentially many alignments. We can interpret the random series $\omega$ of length $D$ as the possible \emph{shapes} of a time series, defined by $D$ segments, each associated with a random number. Figure \ref{fig:DTW_Alignment_BirdChicken} gives an example of a random series $\omega$ of length $D=10$, which divides a time series $x$ into $D$ segments and outputs a dissimilarity score as the feature $\phi_{\omega}(x)$. The third advantage of \eqref{new_kernel} is its computational efficiency due to a simple \emph{random features approximation}. Although the kernel function \eqref{new_kernel} seems hard to compute, we show that there is a low-dimensional representation of each series $T(x)$, by which one can efficiently find an approximate solution to that of the exact kernel \eqref{new_kernel} within $\epsilon$ precision. This is in contrast to the global-alignment kernel \eqref{global_align_kernel}, where although one can evaluate the kernel matrix exactly in $O(N^2L^2)$ time, it is unclear how to efficiently find a low-rank approximation.

\subsection{Computation of Random Warping Series}

Although the kernel \eqref{new_kernel} does not yield a simple analytic form, it naturally yields a random approximation of the form using a simple MC method,
$$
k(x,y) \ \approx \ \ \big \langle T(x), T(y) \big \rangle = \frac{1}{R} \sum_{i=1}^R \big \langle \phi_{\omega_i}(x), \phi_{\omega_i}(y) \big \rangle.
$$
The feature vector $T(x)$ is computed using dissimilarity measure $\tau({\{\omega_i\}}_{i=1}^R, x)$, where ${\{\omega_i\}}_{i=1}^R$ is a set of random series of variable length $D$ with each value drawn from a distribution $p(\omega)$. In particular, the function $\tau$ could be any elastic distance measure but without loss of generality we consider DTW as our similarity measure since it has proved to be the most successful metric for time series \citep{wang2013experimental, xi2006fast}. 

Algorithm \ref{alg:RWS_features} summarizes the procedure to generate feature vectors for raw time series. There are several comments worth making here. First of all, the distribution of $p(\omega)$ plays an important role in capturing the global properties of original time-series. Since we explicitly define a kernel from this distribution, it is flexible to search for the best distribution that fits data well for underlying applications. In our experiments, we find the Gaussian distribution is generally applicable for time series from various applications. Specifically, the parameter $\sigma$ stems from a distribution $p(\omega)$ that should well capture the characteristics of time series $\{x_{i}\}_{i=1}^N$. Second, as shown in Figure \ref{fig:DTW_Alignment_BirdChicken}, a short random warping series could typically identify the local patterns as well as global patterns in raw time series. It suggests that there are some optimal alignments that allow short random series to segment raw time series to obtain discriminatory features. In practice, there is no prior information for this optimal alignment and thus we choose to uniformly sample the length of random series between $[D_{min},D_{max}]$ to give an unbiased estimate of $D$, where $D_{min}=1$ is used in our experiments. Additional benefits lie in the fact that random series with variable lengths may simultaneously identify multi-scale patterns hidden in the raw time series. 

\begin{algorithm}[t]
\caption{RWS Approximation: An Unsupervised Feature Representation for Time Series}
\begin{algorithmic}[1]
    \STATEx {\bf Input:} Time series $\{x_i\}_{i=1}^N, 1 \leq |x_i| \leq L$, $D_{min}$, $D_{max}$, $R$, $\sigma$ associated to $p(\omega)$.
    \STATEx {\bf Output:} Feature matrix $T_{N \times R}$ for time series
    \FOR {$j = 1, \ldots, R$}
        \STATE Draw $D$ uniformly from $[D_{min}, D_{max}]$. Generate random time series $\omega_j$ of length $D_j$ with each value drawn from distribution $p(\omega)$ normalized by $\sigma$.
        \STATE Compute a feature vector $T(:,j) = \phi_{\omega_i}(\{x_i\}_{i=1}^N)$ using DTW with or without a window size.  
    \ENDFOR
    \STATE Return feature matrix $T_{N \times R} = \frac{1}{\sqrt{R}} [T(:,1:R)]$
\end{algorithmic}
\label{alg:RWS_features}
\end{algorithm}

In addition to giving a practical way to approximate the proposed kernel, applying these random series also enjoys the double benefits of reduced computation and memory consumption. Compared to the family of global alignment kernels \citep{cuturi2007kernel, cuturi2011fast}, computing the dense kernel matrix $K \in \mathbb{R}^{N \times N}$ requires $O(N^2)$ times evaluation of DTW which usually takes $O(L^2)$ complexity based on DP. It also needs $O(NL + N^2)$ to store the original time series and resulting kernel matrix. In contrast, our RWS approximation only requires linear complexity of $O(NRL)$ computation and $O(NR)$ storage size, given $D$ is a small constant. This dramatic reduction in both computation and memory storage empowers much more efficient training and testing when combining with ERM classifiers such as SVM.

\subsection{Convergence of Random Warping Series}

In the following, we extend standard convergence analysis of Random Features \citep{rahimi2007random} from a kernel between two fixed-dimensional vectors to a kernel function measuring similarity between two time series of variable lengths. Note \citep{wu2018d2ke} has proposed a general analysis for any distance-based kernel through covering number w.r.t. the distance, which however, does not apply directly here since DTW is not a distance metric.

Let $(A,B)$ be $l\times D$ and $l\times L$ matrices that map each element of $\omega$ and $x$ to an element of a DTW alignment path. The feature map of RWS can be expressed as
\begin{equation}\label{RF}
\phi_{\omega}(x) := \min_{(A,B)\in\A(\omega,x)} \tau(A\omega,Bx) 
=\sum_{i=1}^l \tau([A\omega]_i, [Bx]_i).
\end{equation}
Note that in practice one can often convert a similarity function into a dissimilarity function to fit into the above setting. The goal is to approximate the kernel
$
k(x,y):=\int_{\omega} p(\omega) \phi_{\omega}(x)\phi_{\omega}(y) d\omega
$
via a sampling approximation
$
s_R(x,y)=\frac{1}{R}\sum_{i=1}^R \phi_{\omega_i}(x)\phi_{\omega_i}(y)
$
with $\omega_i$ $\sim$ $p(\omega)$. Note we have $E[s_R(x,y)]=E_{\omega_i}[\phi_{\omega_i}(x)\phi_{\omega_i}(y)]=k(x,y)$. The question is how many samples $R$ are needed to guarantee 
\begin{equation}\label{kernel_approx_error}
\left|s_R(x,y)-k(x,y)\right| \leq \epsilon \;\;\forall x,y\in \X
\end{equation}
In the standard analysis of RF, the required sample size is
$
\Omega(\frac{d}{\epsilon^2}\log\frac{\sigma_p \text{diam}(\X)}{\epsilon})
$
where $\X$ comprises all $d$-dimensional vectors of diameter diam($\X$). The standard analysis does not apply to our case for two reasons: (a) our domain $\X$ contains time series of different lengths, and (b) our kernel involves a minimization \eqref{RF} over all possible DTW alignments, and thus is not shift-invariant as required in \citep{rahimi2007random}. To obtain a uniform convergence bound that could potentially handle time series of unbounded length, we introduce the notion of \emph{minimum shape-preserving length}.

\begin{definition} The \emph{Minimum Shape-Preserving Length (MSPL)} $d_{\epsilon}$ of tolerance $\epsilon$ is the smallest $L$ such that $\forall x\in \X, \exists\tx\in\R^{L},$
\begin{equation}\label{mini_length}
\min_{(A,B)\in\A(\tx,x):B=I} \;\|A\tx-Bx\| \leq \epsilon
\end{equation}
where $\A(\tx, x)$ is the set of possible alignments between $\tilde x$ and $x$ considered by DTW, and $I$ is an identity matrix.
\end{definition}

In other words, $d_{\epsilon}$ defines the smallest length one can compress a time series to with approximation error no more than $\epsilon$, measured by DTW in the $\ell_2$ distance. Then the following gives the number of RWS required to guarantee an $\epsilon$ uniform convergence over all possible inputs $x,y\in\mathcal{X}$.

\begin{theorem}\label{theorem_converge}
Assume the ground metric $\tau(A\omega,Bx)$ satisfies $|\tau(.,.)|\leq\gamma$ and is Lipschitz-continuous w.r.t. $x$ with parameter $\beta(\omega)$ where $\text{Var}[\beta(\omega)]\leq\sigma_{\tau}^2$. The RWS approximation with $R$ features satisfies
\begin{multline}
    P \left[\max_{x,y\in \X}|s_R(x,y)-k(x,y)|\geq 3\epsilon \right] \\
\leq  8r^2\left(\frac{4\gamma\sigma_{\tau}}{\epsilon}\right)^2 e^{-\frac{R\epsilon^2}{32\gamma^4(1+d_{\epsilon})}}.
\end{multline}
where $r$ is the radius of time series domain $\X$ in the $\ell_{\infty}$ norm and $d_{\epsilon}$ is the MSPL with precision $\epsilon$.
\end{theorem}

\begin{proof}[Proof Sketch]
Let $f(x,y):=s_R(x,y)-k(x,y)$. We have $E[f(x,y)]=0$ and $|f(x,y)|\leq 2\gamma^2$ by the boundedness of function $\tau(.,.)$. Then by Hoeffding inequality, we have
\begin{equation}\label{hoeffding}
P\left[|f(x,y)|\geq t \right] \leq 2\exp(-Rt^2/8\gamma^4)
\end{equation}
for a given pair $(x,y)\in \X \times \X$. To get a uniform bound that holds for all pairs of series $(x,y)\in\X\times\X$, consider the pair of series  $(\tx,\ty)$ of minimum shape-preserving length $d(\epsilon)$ under precision $\epsilon$. We have an $\epsilon$-net $\E$ with $
|\E|=(\frac{2r}{\epsilon})^{d}$
that covers the $d$-dimensional $\ell_{\infty}$-ball of radius $r$. Then through union bound and \eqref{hoeffding}, we have
\begin{equation}\label{union_bound}
P\left[\max_{\tx_i,\ty_j\in \E}|f(\tx_i,\ty_j)|\geq t \right] \leq 2|\E|^2\exp(-Rt^2/8\gamma^4).
\end{equation}
Let $\B_{\infty}(d)$ be the $d$-dimensional $\ell_{\infty}$-ball. Given any time series $x,y\in\X$ of arbitrary length, we can first find $\tx,\ty\in \B_{\infty}(d)$ with
$\|\tA\tx - x\|\leq \epsilon$, $\|\tA\ty-y\| \leq \epsilon$ and then find
$\tx_i,\ty_j\in\E$ such that $\|\tx-\tx_i\|\leq \epsilon$, $\|\ty-\ty_j\|\leq \epsilon$.
By the result of Lemma \ref{lemma_lipschitz_f} (see appendix \ref{sec:proof_lemma}), the closeness of $(x,y)$ to $(\tx_i,\ty_i)$ implies the closeness of $f(x,y)$ to $f(\tx_i,\ty_i)$, which leads to
\begin{equation}\label{tmp2}
P\left[|f(x,y)-f(\tx_i,\ty_i)| \geq 2t\right]\leq \frac{8\gamma^2\sigma_{\tau}^2\epsilon^2}{t^2}.
\end{equation}
Combining \eqref{union_bound} and \eqref{tmp2}, we have
\begin{equation}\label{bound_prob}
P\left[\max_{x,y\in \X}|f(\tx,\ty)| \geq 3t\right]\leq2\left(\frac{2r}{\epsilon}\right)^{2d}e^{-\frac{Rt^2}{32\gamma^4}}+\frac{16\gamma^2\sigma_{\tau}^2\epsilon^2}{t^2}.
\end{equation}
This is of the from $\kappa_1\epsilon^{-2d}+\kappa_2\epsilon^2$. Choosing $\epsilon=\left(\kappa_1/\kappa_2\right)^{1/(2+2d)}$ to balance the two terms in \eqref{bound_prob}, the RHS becomes $2\kappa_1^{1/(1+d)}\kappa_2^{d/(1+d)}$.
This yields the result
$$
P\left[\max_{x,y\in \X}|f(x,y)| \geq 3t\right] \leq  8r^2\left(\frac{4\gamma\sigma_{\tau}}{t}\right)^2 e^{-\frac{Rt^2}{32\gamma^4(1+d)}}.
$$
\end{proof}
\vspace{-4mm}

The above theorem \ref{theorem_converge} shows that, to guarantee $\sup_{x,y\in\X}|s_R(x,y)-k(x,y)|\leq \epsilon$ with probability $1-\delta$, it suffices to have
$
R = \Omega(\frac{d_{\epsilon}\gamma^4}{\epsilon^2}\log\frac{\gamma r\sigma_{\tau}}{\delta\epsilon}).
$
In practice, the constants $r$, $\gamma$ are not particularly large due to the normalization on series $x,y\in \X$ and dissimilarity function $\tau(.,.)$. The main factor determining the rate of convergence is the shape-preserving length $d_{\epsilon}$. Note that for problems with time series length bounded by $L$, we always have $d_{\epsilon}\leq L$, which means the number of features required would be only of order $R=\Omega(L/\epsilon^2)$.

\section{Experiments}
\label{sec:Experiments}
We conduct experiments to demonstrate the efficiency and effectiveness of the RWS, and compare against 9 baselines on 16 real-world datasets from the widely-used UCR time-series classification archive \citep{UCRArchive} as shown in Table \ref{tb:info of datasets}. We evaluate RWS on the datasets with variable number and length to achieve these goals: 1) competitive or better accuracy for small problems; 2) matches or outperforms other methods in terms of both performance and runtime for middle or large scale tasks.  
We implement our method in Matlab and use C Mex function \footnote{https://www.mathworks.com/matlabcentral /fileexchange/43156-dynamic-time-warping--dtw-} for computationally expensive component of DTW. For other methods we use the same routine to promote a fair runtime comparison, where the window size of DTW is set as $min(L/10, 40)$ similar to \citep{QiYi2016,paparrizos2015k}. More details about datasets and parameter settings are in Appendix \ref{sec:Experimental settings and parameters for RWS}.

\begin{table}[htbp]
\centering
\scriptsize
\caption{Properties of the datasets. The number and the length of time series are sorted increasingly. } 
\label{tb:info of datasets}
\vspace{-3mm}
\begin{center}
    \begin{tabular}{ c c c c c }
    \hline
    Name & $C$:Classes & $N$:Train & $M$:Test & $L$:length \\ \hline 
    Beef    & 5 & 30 & 30 & 470  \\
    DPTW    & 6 & 400 & 139	& 80  \\
    IPD     & 2 & 67 & 1,029 & 24 \\
    PPOAG   & 3 & 400 & 205	& 80 \\ 
    MPOC    & 2 & 600 & 291	& 80  \\  
    POC     & 2 & 1,800 & 858 & 80  \\
    LKA     & 3 & 375 & 375	& 720 \\ 
    IWBS    & 11 & 220 & 1,980 & 256 \\ 
    TWOP    & 4 & 1,000 & 4,000 & 128 \\ 
    ECG5T   & 5 & 500 & 4,500 & 140 \\ 
    CHCO    & 3 & 467 &	3,840 & 166  \\
    Wafer 	& 2 & 1,000 & 6,174 & 152 \\ 
    MALLAT  & 8 & 55 & 2,345 & 1,024    \\ 
    FordB   & 2 & 3636 & 810 & 500 \\ 
    NIFECG  & 42 & 1,800 & 1,965 & 750  \\
    HO      & 2 & 370 & 1,000 & 2,709 \\  \hline
    \end{tabular}
\end{center}
\vspace{-4mm}
\end{table}

\begin{table}[t]
\centering
\caption{Classification performance comparison among RWS, TSEigen, and TSMC with $R = 32$.} 
\label{tb:comp_rf_mc_eigen}
\scriptsize
\newcommand{\Bd}[1]{\textbf{#1}}
\vspace{-3mm}
\begin{center}
    \begin{tabular}{ c cc cc cc }
    \hline
    \multicolumn{1}{c}{Classifier}
    & \multicolumn{2}{c}{RWS} 
    & \multicolumn{2}{c}{TSEigen}
    & \multicolumn{2}{c}{TSMC} \\ \hline 
    \multicolumn{1}{c}{Dataset}
	& Accu & Time & Accu & Time & Accu & Time \\ \hline
	Beef   & \Bd{0.733} & \Bd{0.3} & 0.633 & 2.1 & 0.433 & 0.6 \\
	DPTW   & \Bd{0.79} & \Bd{0.5} & 0.738 & 7.1 & 0.738 & 1.5 \\
	IPD    & \Bd{0.969} & \Bd{0.3} & 0.911 & 8.6 & 0.80 & 1.7 \\
	PPOAG  & \Bd{0.868} & \Bd{0.4} & 0.82 & 8.9 & 0.82 & 1.8 \\
	MPOC   & \Bd{0.711} & \Bd{0.8} & 0.653 & 19.3 & 0.653 & 2.4 \\
	POC    & \Bd{0.711} & \Bd{2.4} & 0.686 & 172.3 & 0.66 & 8.2 \\
	LKA    & \Bd{0.792} & \Bd{7.3} & 0.528 & 401.5 & 0.525 & 39.5 \\
	IWBS   & 0.619 & \Bd{8.9} & \Bd{0.633} & 784.6 & 0.57 & 31.9 \\
	TWOP   & \Bd{0.999} & \Bd{4.4} & 0.976 & 1395 & 0.946 & 32.8 \\
	ECG5T  & \Bd{0.933} & \Bd{10.6} & 0.932 & 1554 & 0.918 & 36.0 \\
	CHCO   & \Bd{0.572} & \Bd{6.3} & 0.529 & 1668 & 0.402 & 45.7 \\
	Wafer  & \Bd{0.993} & \Bd{9.6} & 0.89 & 3475 & 0.89 & 59.3 \\
	MALLAT & \Bd{0.937}  & \Bd{33.9}  & 0.898 & 7982 & 0.888 & 282.6 \\
    FordB  & \Bd{0.727} & \Bd{43.5}  & 0.704 & 10069  & 0.686 & 216.3  \\ 
    NIFECG & \Bd{0.907} & \Bd{19.8} & 0.867 & 10890  & 0.582 & 265  \\ 
    HO & 0.843 & \Bd{43.3} & \Bd{0.845} & 46509 & 0.82 & 979.1 \\ \hline
    \end{tabular}   
\end{center}
\vspace{-4mm}
\end{table}

\subsection{Effects of $\sigma$, $R$ and $D$ on RWS}
\label{Chapter:Effects of different factors on Random Features}
\textbf{Setup.} We first perform experiments to investigate the characteristics of the RWS method by varying the kernel parameter $\sigma$, the rank $R$ and the length $D$ of random series. Due to limited space, we only show typical results and see Appendix \ref{sec:More Results on various effects of random features} for complete ones. 

\textbf{Effects of $\sigma$.} It is well known that the choice of the kernel parameter $\sigma$ determines the quality of various kernels. Figure \ref{fig:exptsA_varyingS} shows that in most cases the training and testing performance curves agree well in the sense that they consistently increase at the beginning, stabilize around $\sigma = 1$ (which corresponds to the standard distribution), and finally decrease in the end. In a few cases like NIFECG, the optimal performance is slightly shifted from $\sigma = 1$. This observation is favorable since it suggests that one may easily tune our approach over a smaller interval around $\sigma = 1$ for good performance. 

\textbf{Effects of $R$.} We evaluate the training and testing performance when varying the rank $R$ from 4 to 512 with fixed $\sigma$ and $D$. Figure \ref{fig:exptsA_varyingR} shows that the training and testing accuracy generally converge almost exponentially when increasing $R$ from very small number ($R = 4$) to a relative large number ($R = 64$), and then slowly saturate to the optimal performance. Empirically, this feature is the most favorable because the performance of RWS is relatively stable even for small $R$. More importantly, this confirms our analysis in Theorem \ref{theorem_converge} that our RWS approximation can guarantee (rapid) convergence to the exact kernel. 

\textbf{Effects of $D$.} We investigate the effect of the length $D$ of the random series on training and testing performance. As hinted at earlier, a key insight behind the proposed time-series kernel depends on the assumption that a random series of short length can effectively segment raw time series in a way that captures its patterns. Figure \ref{fig:exptsA_varyingD} shows that although testing accuracy seems to fluctuate when varying $D_{max}$ from 10 to 100, it is clear that the near-peak performance can be achieved when $D_{max}$ is small in the most of cases. 

\begin{figure*}[!htb]
\centering
    \begin{subfigure}[b]{0.22\textwidth}
      \includegraphics[width=\textwidth]{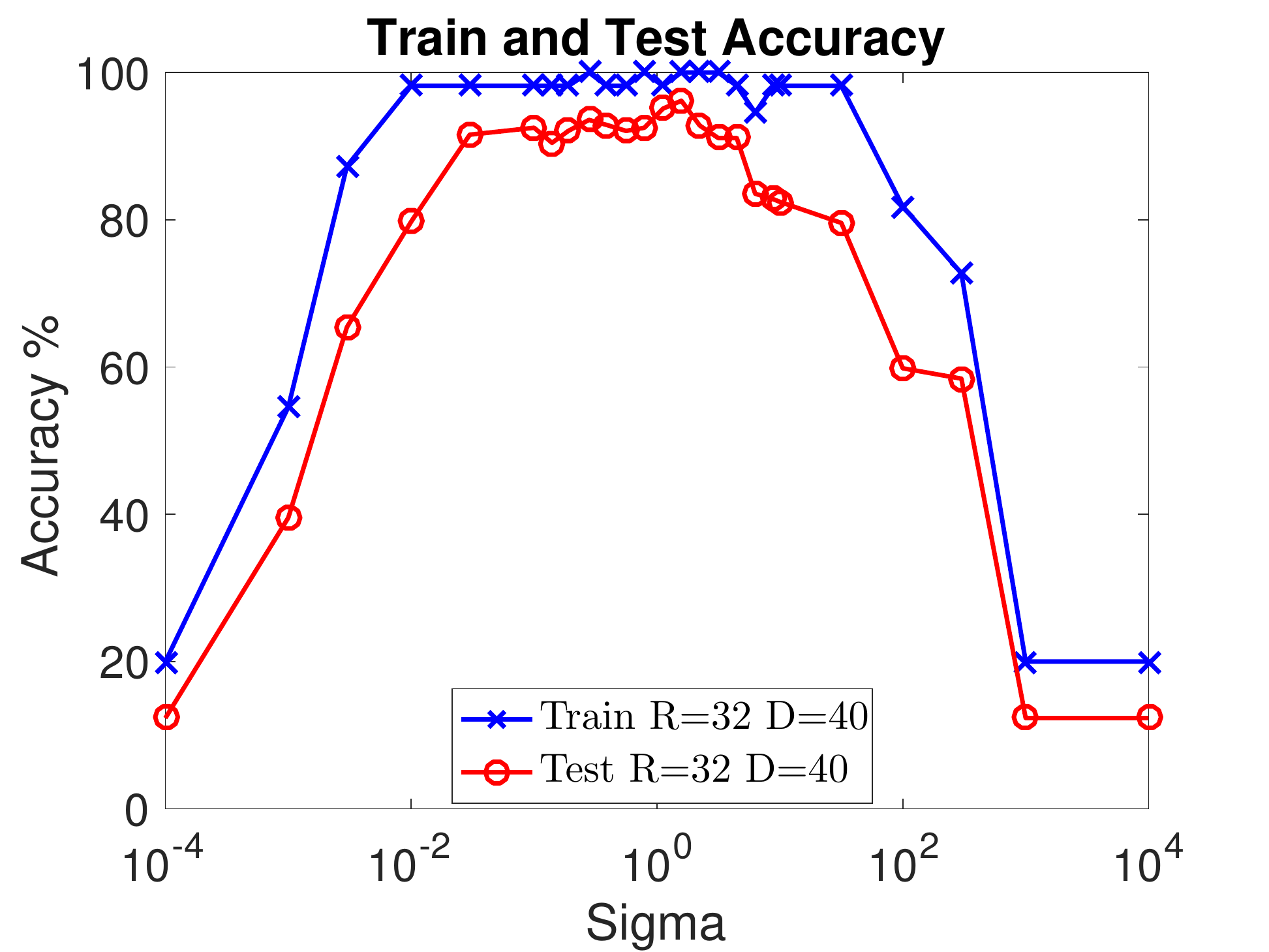}
      \caption{MALLAT}
      \label{fig:exptsA_varyingS_MALLAT}
    \end{subfigure}
   \begin{subfigure}[b]{0.22\textwidth}
      \includegraphics[width=\textwidth]{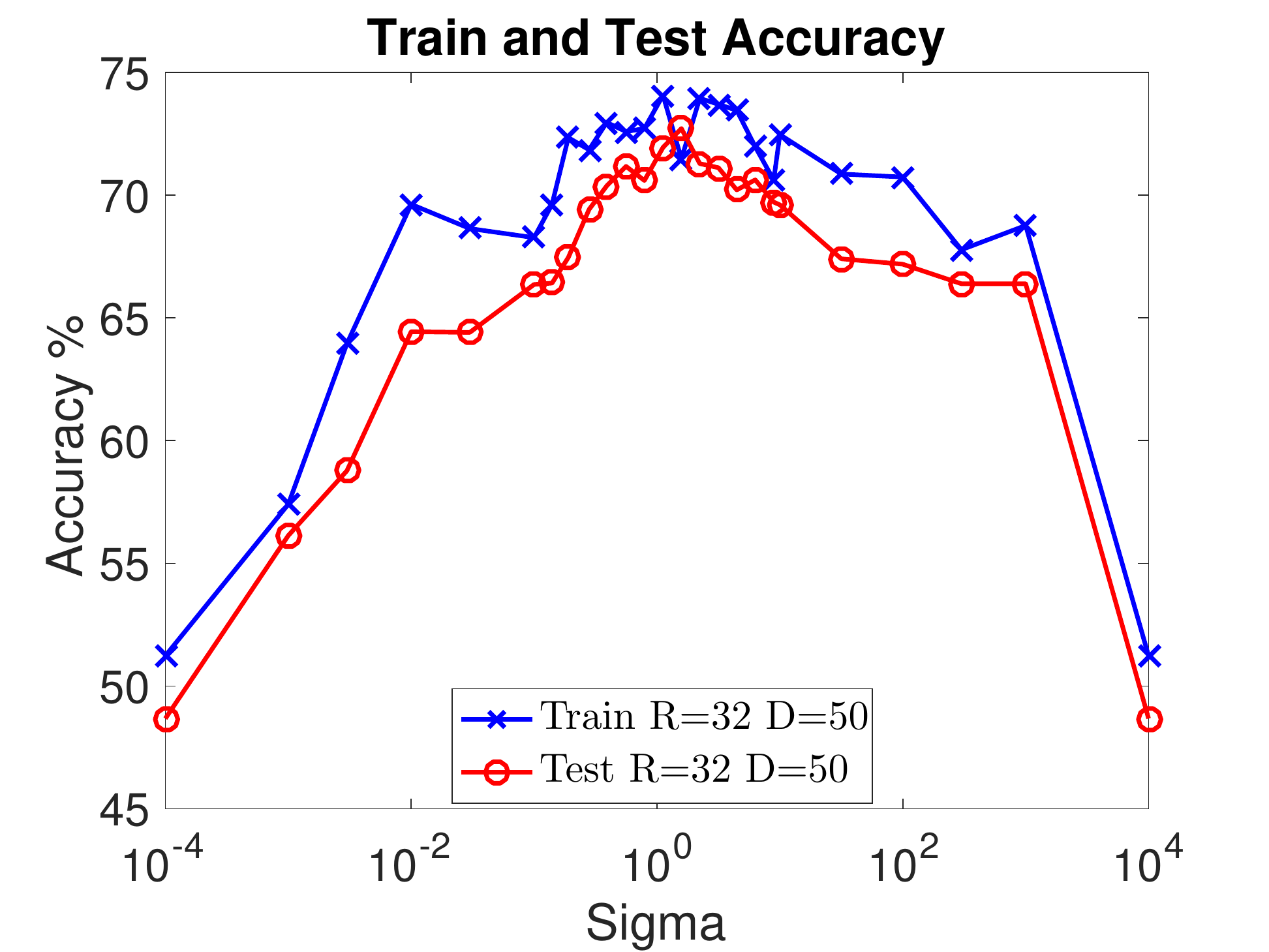}
      \caption{FordB}
      \label{fig:exptsA_varyingS_FordB}
    \end{subfigure}
   \begin{subfigure}[b]{0.22\textwidth}
      \includegraphics[width=\textwidth]{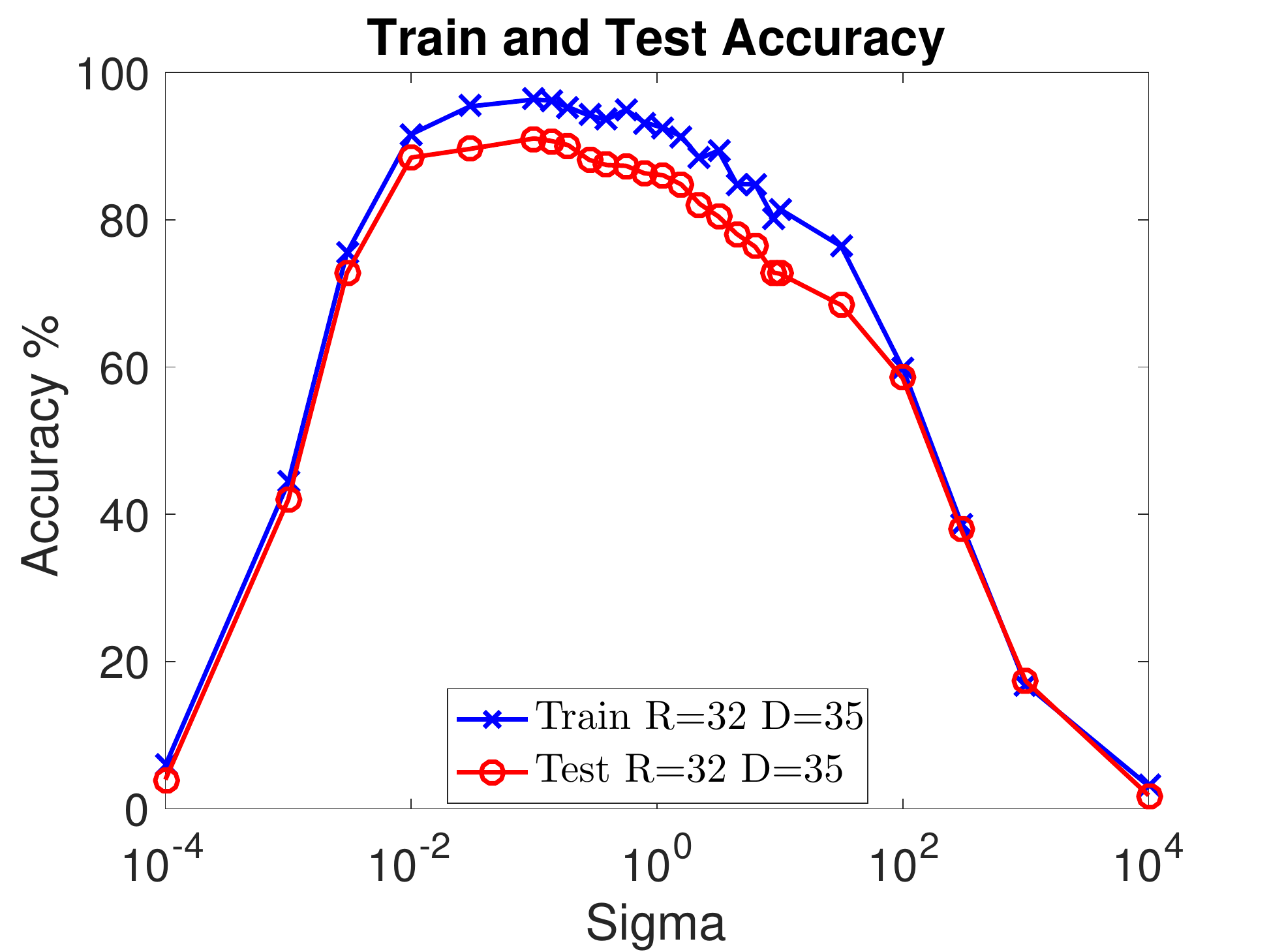}
      \caption{NIFECG}
      \label{fig:exptsA_varyingS_NonInvasiveFatalECG_Thorax2}
    \end{subfigure}
    \begin{subfigure}[b]{0.22\textwidth}
      \includegraphics[width=\textwidth]{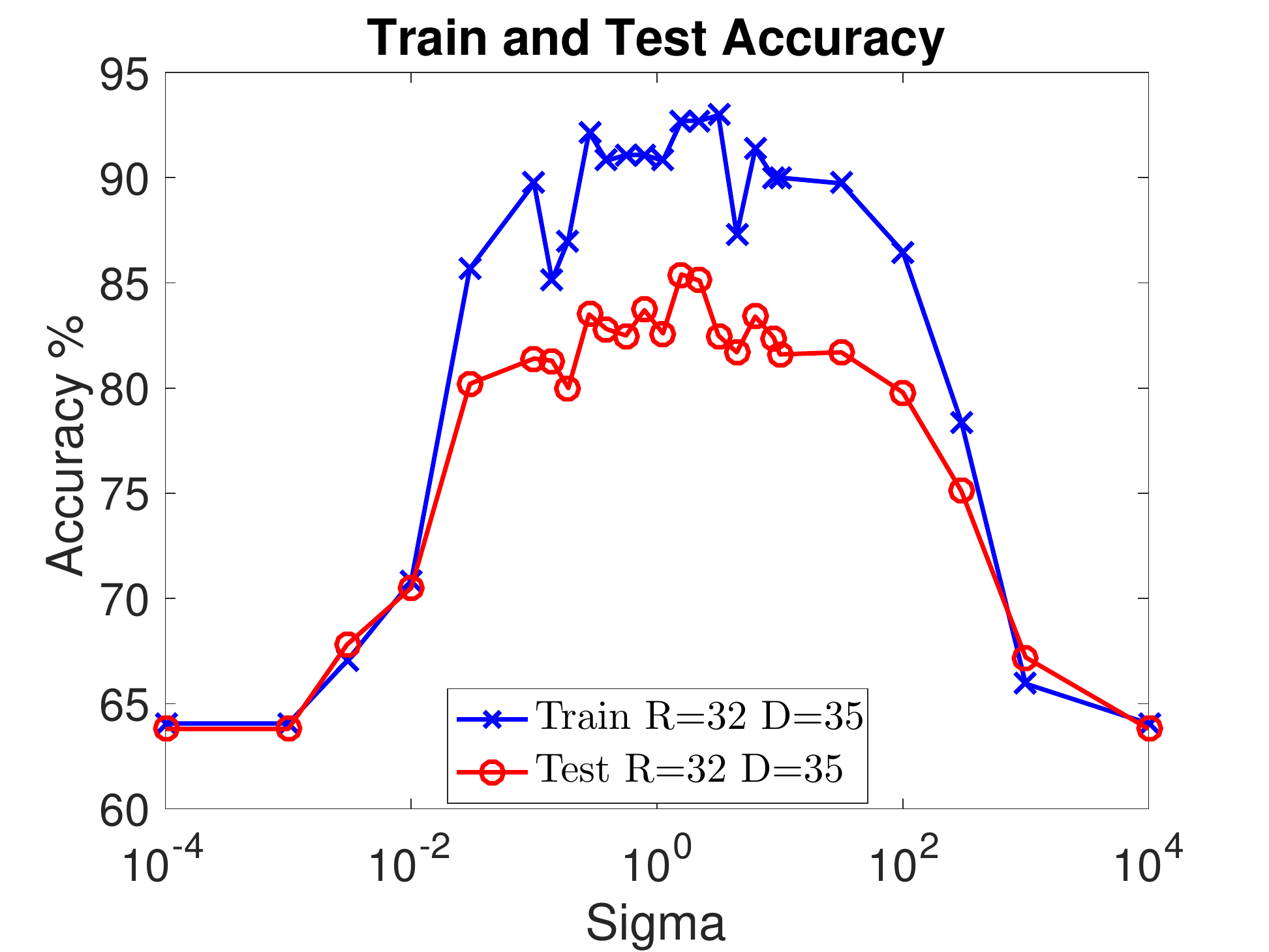}
      \caption{HO}
      \label{fig:exptsA_varyingS_HandOutlines}
    \end{subfigure}
\caption{Train (Blue) and test (Red) accuracy when varying $\sigma$ with fixed $D$ and $R$. }
\label{fig:exptsA_varyingS}
\end{figure*}

\begin{figure*}[!htb]
\centering
    \begin{subfigure}[b]{0.22\textwidth}
      \includegraphics[width=\textwidth]{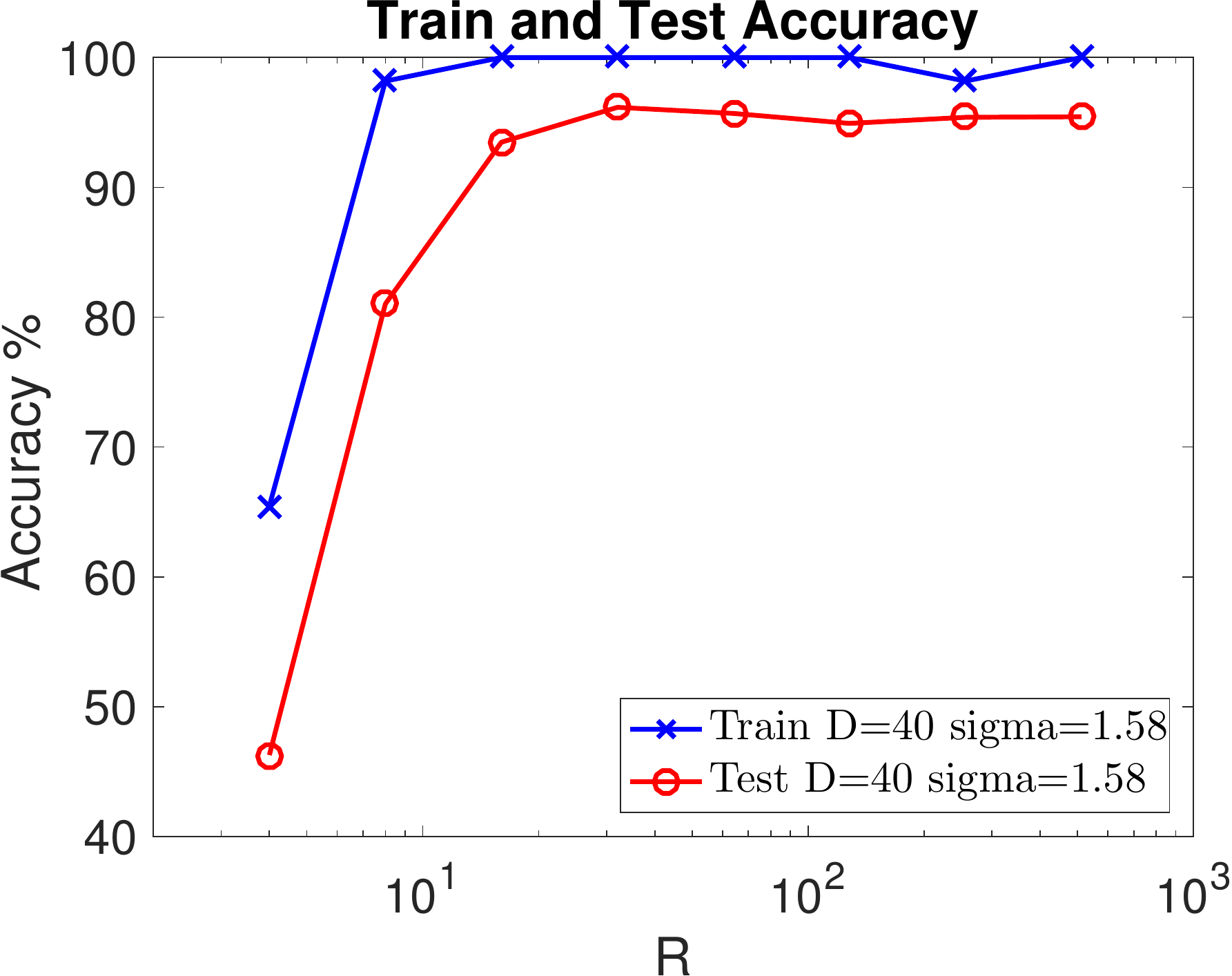}
      \caption{MALLAT}
      \label{fig:exptsA_varyingR_MALLAT}
    \end{subfigure}
   \begin{subfigure}[b]{0.22\textwidth}
      \includegraphics[width=\textwidth]{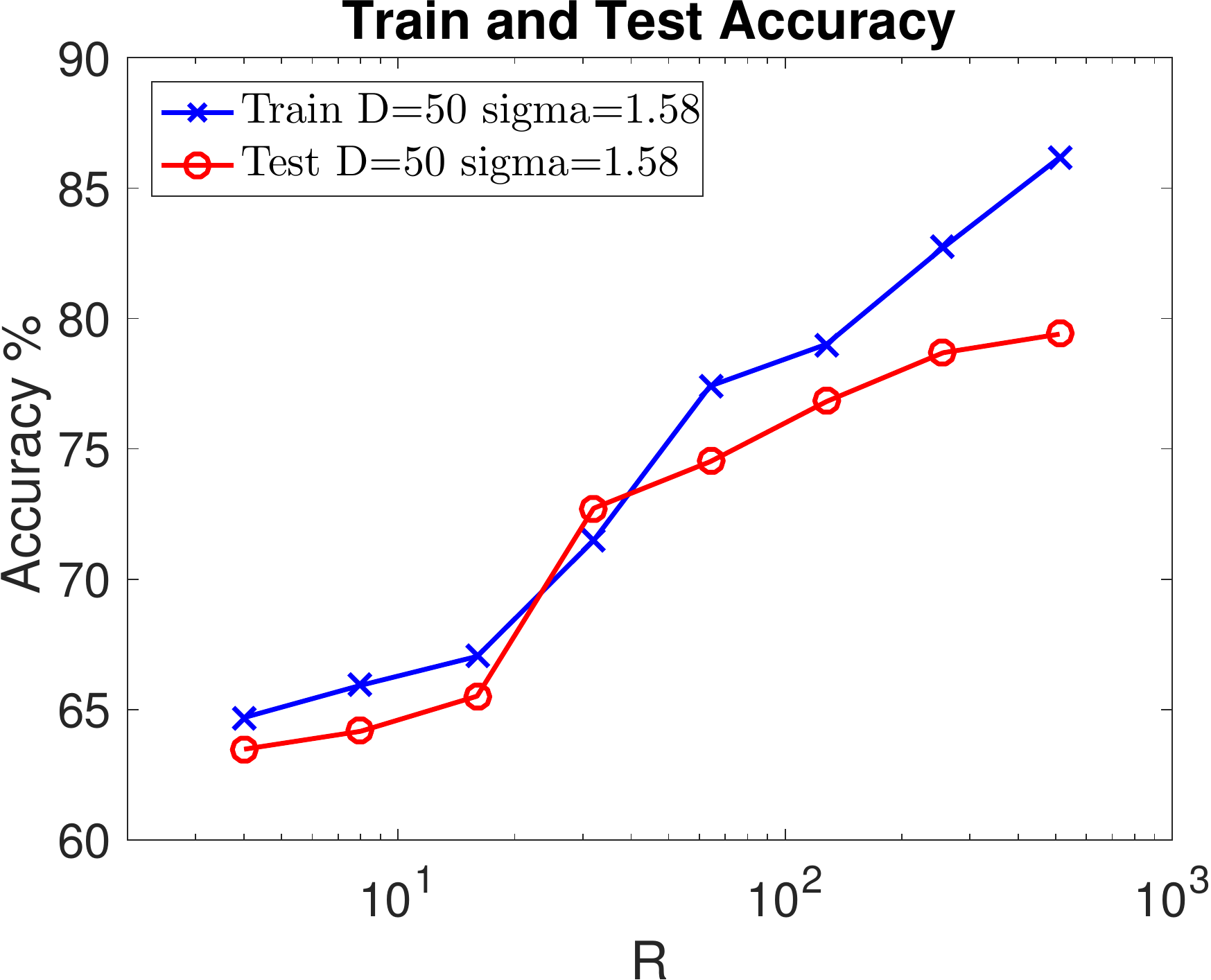}
      \caption{FordB}
      \label{fig:exptsA_varyingR_FordB}
    \end{subfigure}
   \begin{subfigure}[b]{0.22\textwidth}
      \includegraphics[width=\textwidth]{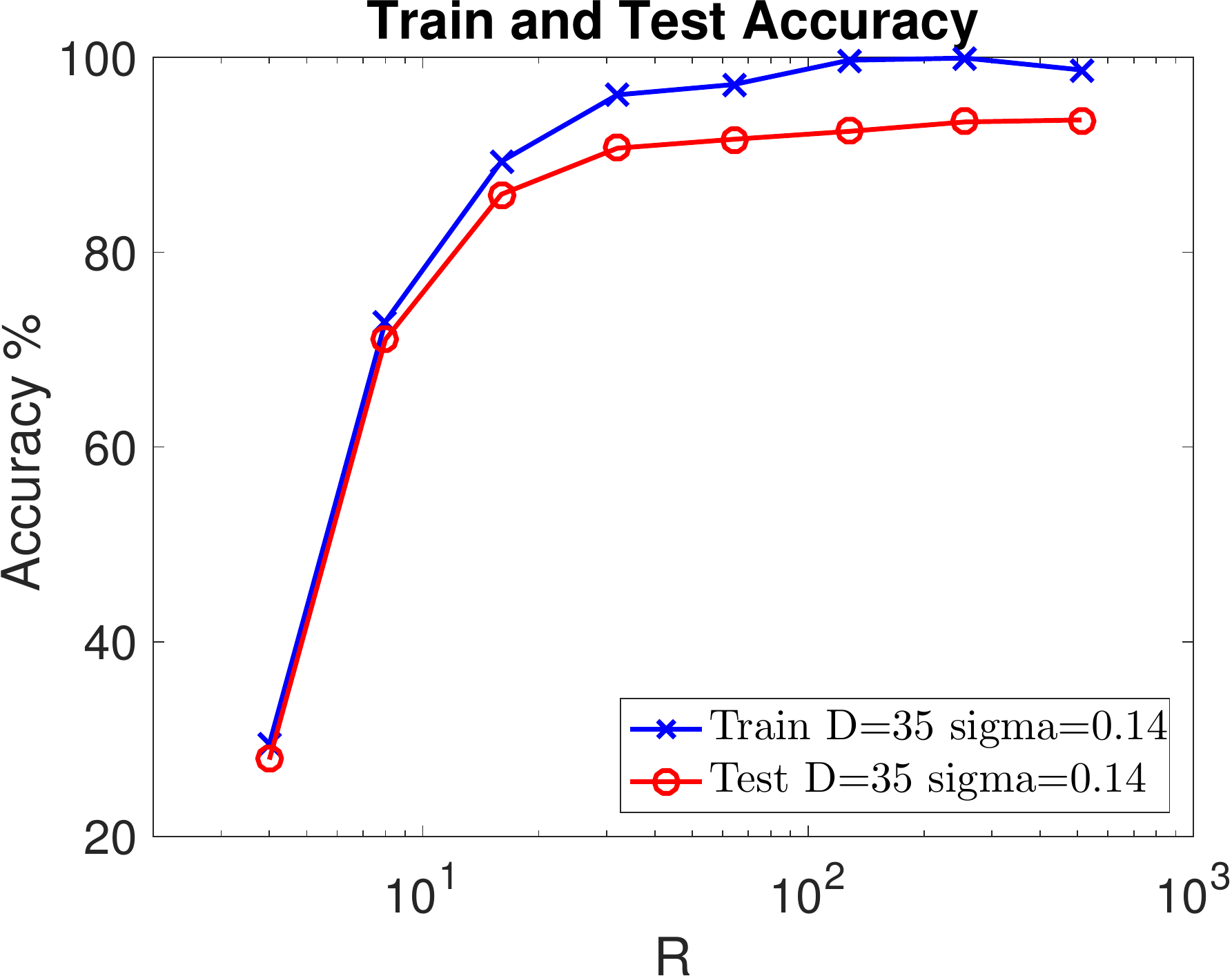}
      \caption{NIFECG}
      \label{fig:exptsA_varyingR_NonInvasiveFatalECG_Thorax2}
    \end{subfigure}
    \begin{subfigure}[b]{0.22\textwidth}
      \includegraphics[width=\textwidth]{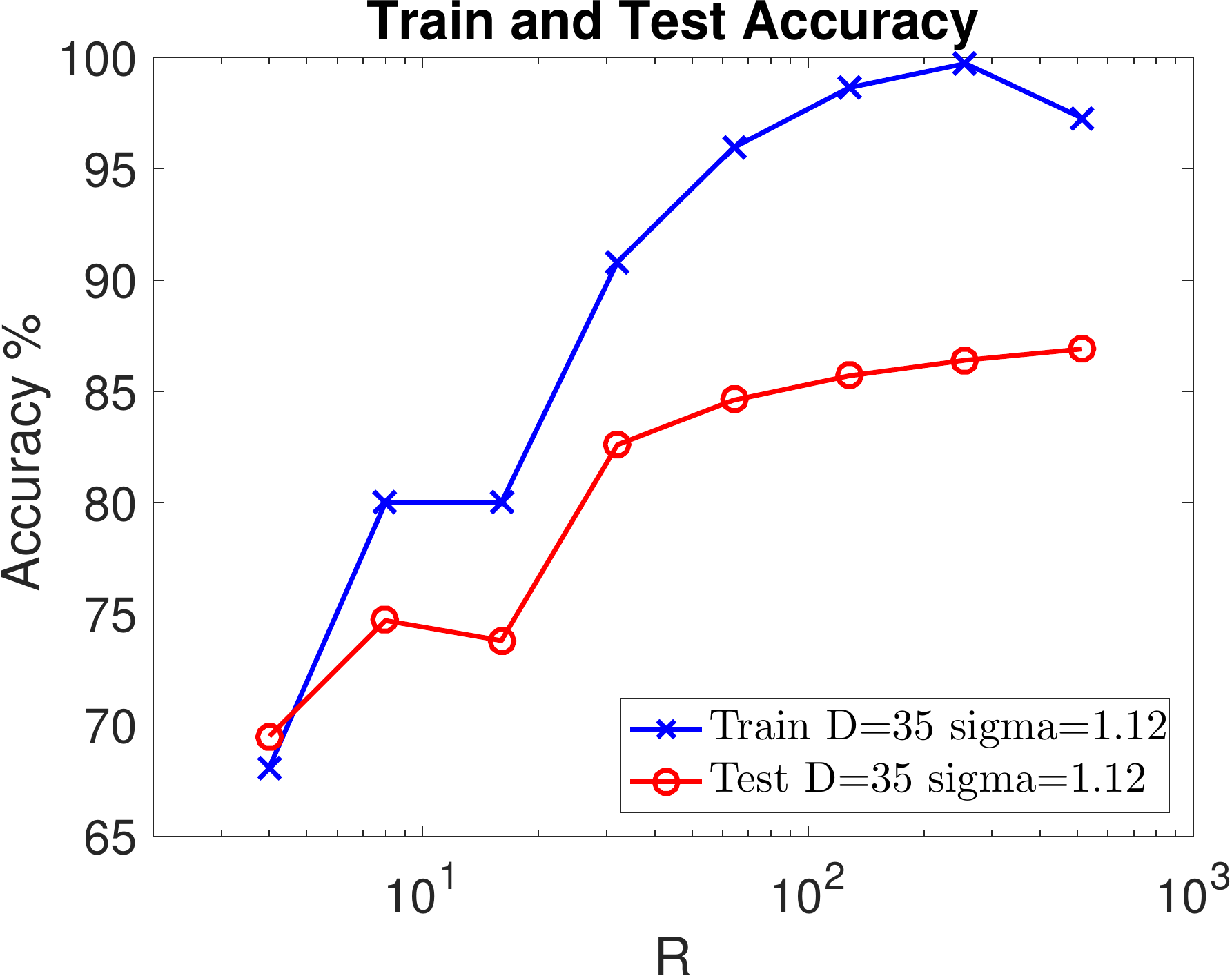}
      \caption{HO}
      \label{fig:exptsA_varyingR_HandOutlines}
    \end{subfigure}
\caption{Train (Blue) and test (Red) accuracy when varying $R$ with fixed $\sigma$ and $D$.}
\label{fig:exptsA_varyingR}
\end{figure*}

\begin{figure*}[!htb]
\centering
   \begin{subfigure}[b]{0.22\textwidth}
      \includegraphics[width=\textwidth]{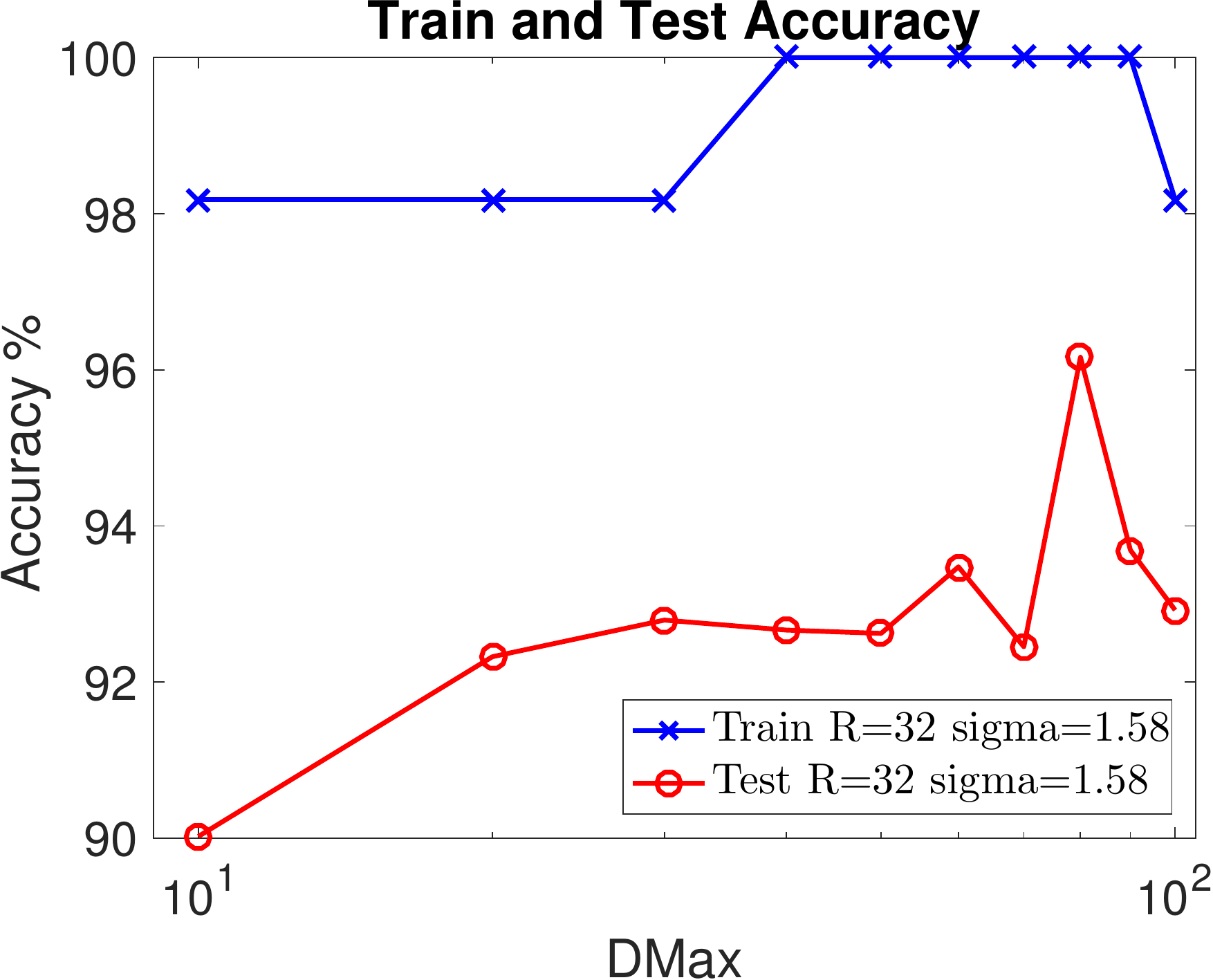}
      \caption{MALLAT}
      \label{fig:exptsA_varyingD_MALLAT}
    \end{subfigure}
   \begin{subfigure}[b]{0.22\textwidth}
      \includegraphics[width=\textwidth]{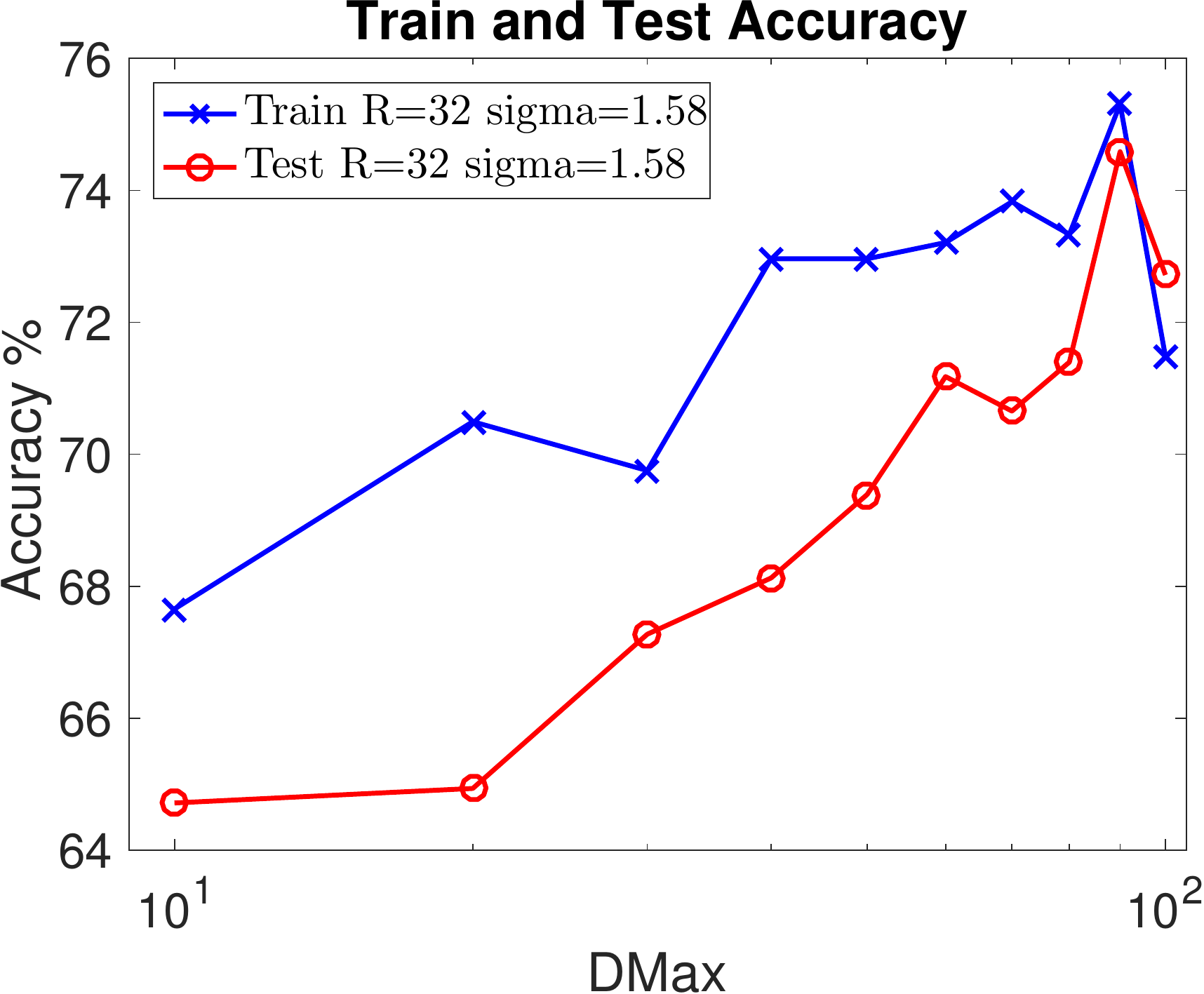}
      \caption{FordB}
      \label{fig:exptsA_varyingD_FordB}
    \end{subfigure}
   \begin{subfigure}[b]{0.22\textwidth}
      \includegraphics[width=\textwidth]{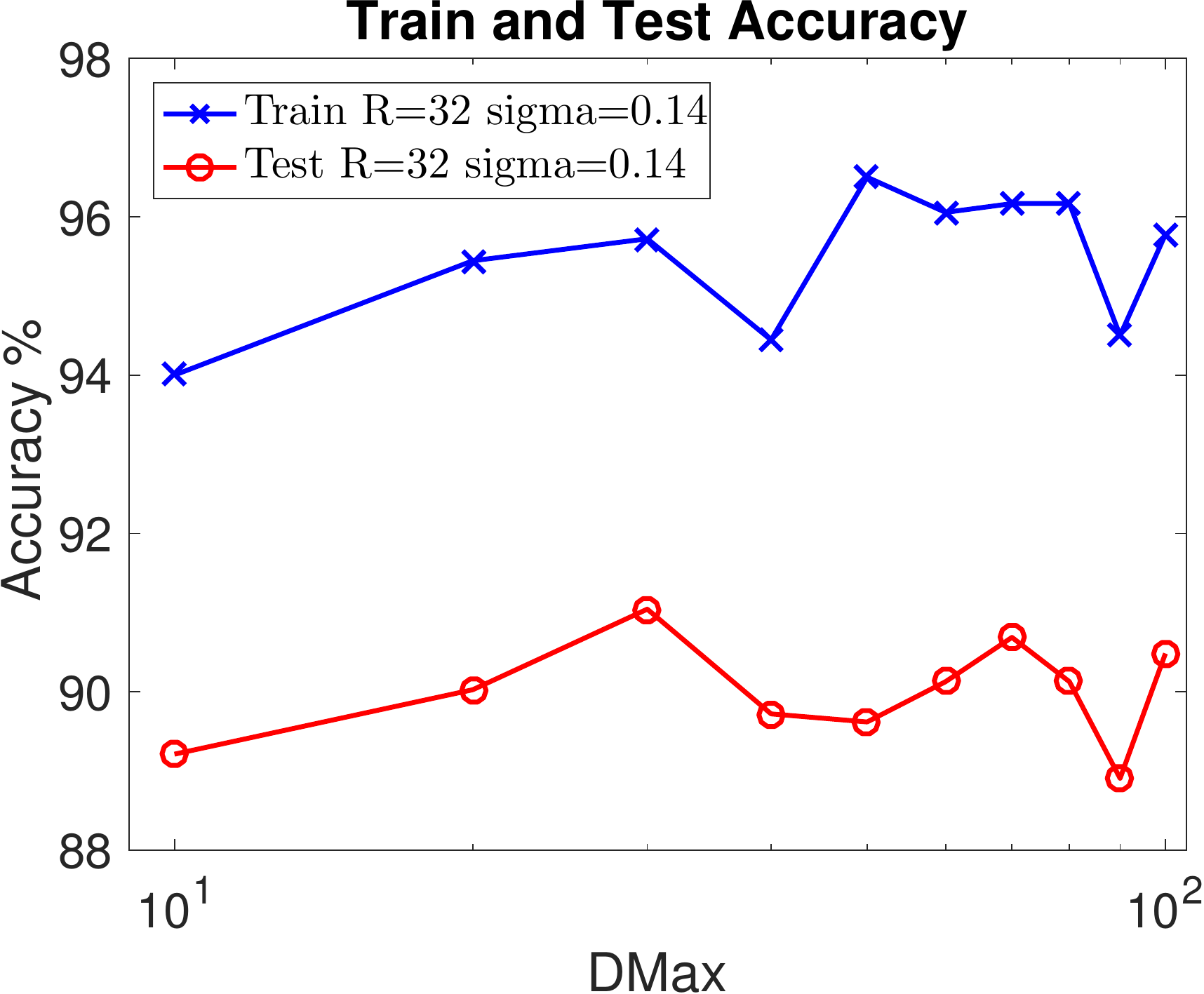}
      \caption{NIFECG}
      \label{fig:exptsA_varyingD_NonInvasiveFatalECG_Thorax2}
    \end{subfigure}
    \begin{subfigure}[b]{0.22\textwidth}
      \includegraphics[width=\textwidth]{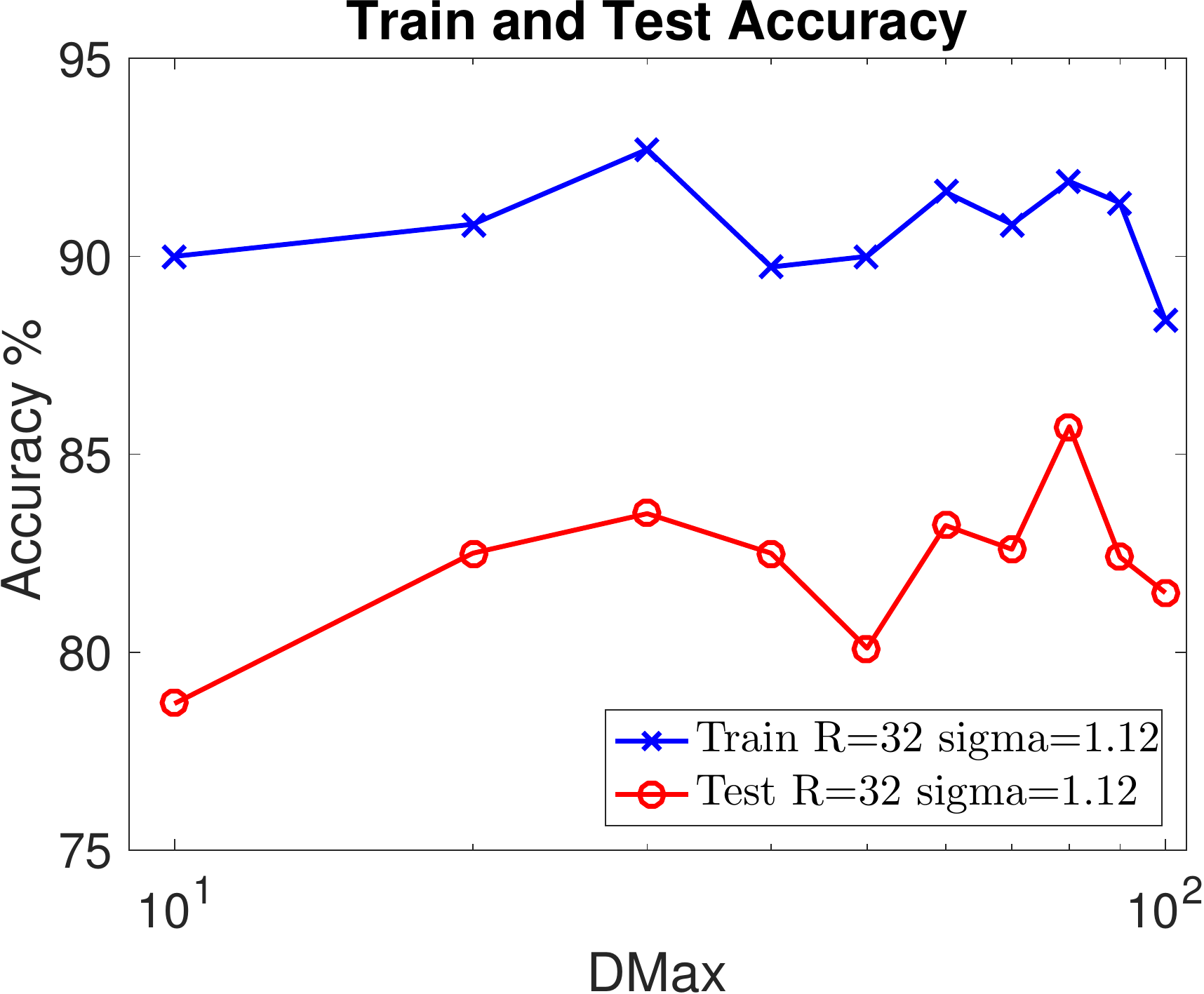}
      \caption{HO}
      \label{fig:exptsA_varyingD_HandOutlines}
    \end{subfigure}
\caption{Train (Blue) and test (Red) accuracy when varying $D$ with fixed $\sigma$ and $R$.}
\label{fig:exptsA_varyingD}
\vspace{-2mm}
\end{figure*}

\begin{table*}[t]
\centering
\caption{Classification performance comparison among methods using DTW or DTW-like kernels.} 
\label{tb:comp_allmethods_classification}
\scriptsize
\newcommand{\Bd}[1]{\textbf{#1}}
\newcommand{\Em}[1]{\emph{#1}}
\vspace{-3mm}
\begin{center}
    \begin{tabular}{ c cc cc cc cc cc cc } \hline
    \multicolumn{1}{c}{Classifier}
    & \multicolumn{2}{c}{RWS(LR)}
    & \multicolumn{2}{c}{RWS(SR)}
    & \multicolumn{2}{c}{1NN-DTW} 
    & \multicolumn{2}{c}{1NN-DTW\textsuperscript{opt}} 
    & \multicolumn{2}{c}{TGAK}
    & \multicolumn{2}{c}{DTWF} \\ \hline 
    \multicolumn{1}{c}{Dataset} 
	& Accu & Time & Accu & Time & Accu & Time & Accu & Time
	& Accu & Time & Accu & Time  \\ \hline
	Beef & \Bd{0.767} & 0.8 & 0.733 & \Bd{0.3} & 0.567 & 1.1 & 0.633 & \Bd{0.3} & 0.633 & 24.7 & 0.60 & 3.7 \\
	DPTW & \Bd{0.865} & 4.2 & 0.80 & \Bd{0.2} & 0.73 & 1.4  & 0.718 & 0.8 & 0.738 & 27.9 & 0.77 & 3.0 \\
	IPD  & \Bd{0.965} & 1.0  & 0.962 & \Bd{0.4} & 0.947 & 55.3 & 0.962 & 56.0 & 0.739 & 3.7  & 0.953 & 0.5 \\
	PPOAG & \Bd{0.868} & 0.3 & 0.859 & \Bd{0.2} & 0.776 & 2.0 & 0.785 & 1.2 & 0.854 & 118.2 & 0.829 & 9.7 \\
	MPOC & \Bd{0.773} & 6.8 & 0.708 & \Bd{0.8}  & 0.635 & 4.4 & 0.663 & 2.7  & 0.627 & 117.3 & 0.653 & 10.2 \\
	POC & \Bd{0.815} & 38.2 & 0.746 & \Bd{4.7} & 0.721 & 36.9 & 0.751 & 20.1 & 0.613 & 2373 & 0.79 & 202.7 \\
	LKA & \Bd{0.84} & 54.9 & 0.816 & \Bd{13.6} & 0.712 & 97.7 & 0.837 & 573.6 & 0.645 & 13484 & 0.80 & 1220 \\
	IWBS  & \Bd{0.641} & 132.4 & 0.619 & \Bd{8.8} & 0.504 & 70.9 & 0.589 & 36.1 & 0.126 & 2413 & 0.609 & 260.3 \\
	TWOP & \Bd{1} & 16.1  & 0.999 & \Bd{4.4} & 1 & 222.2 & 1 & 157.5 & 0.269 & 5690 & 1 & 481.7 \\
	ECG5T & \Bd{0.94} & 9.2 & 0.934 & \Bd{4.9} & 0.928 & 137.8 & 0.928 & 70.1 & 0.927 & 2822 & 0.933 & 278.3 \\
	CHCO & \Bd{0.777} & 189.1 & 0.683 & \Bd{48.1} & 0.627 & 160.8 & 0.627 & 57.0 & 0.545 & 3122 & 0.666 & 333.6 \\
	Wafer & \Bd{0.995} & 143.6 & 0.993 & \Bd{9.6} & 0.986 & 412.3 & 0.996 & 210.1 & 0.896 & 11172 & 0.994 & 980.5  \\ 
	MALLAT & \Bd{0.952} & 72.8 & 0.937 & \Bd{33.8} & 0.937 & 150.3 & 0.925 & 65.5 & 0.257 & 11882 & 0.915 & 988.4 \\
	FordB & 0.793 & 543.8 & 0.62 & \Bd{5.6} & 0.589 & 1476 & 0.581 & 577.6 & N/A & N/A & \Bd{0.83} & 8402 \\
	NIFECG & \Bd{0.936} & 140.2 & 0.903 & \Bd{20.0} & 0.845 & 2699 & 0.857 & 1432 & N/A & N/A & 0.906 & 32493 \\
	HO & 0.871 & 336.9 & 0.834 & \Bd{41.9} & 0.816 & 4883 & 0.807 & 5837 & N/A & N/A & \Bd{0.898} & 40407  \\ \hline
    \end{tabular}   
\end{center}
\vspace{-2mm}
\end{table*}

\begin{table*}[t]
\centering
\caption{Clustering performance comparison among different methods.} 
\label{tb:comp_allmethods_clustering}
\scriptsize
\newcommand{\Bd}[1]{\textbf{#1}}
\newcommand{\Em}[1]{{#1}}
\vspace{-3mm}
\begin{center}
    \begin{tabular}{ c cc cc cc cc cc }
    \hline
    \multicolumn{1}{c}{Clustering}
    & \multicolumn{2}{c}{RWS(LR)}
    & \multicolumn{2}{c}{RWS(SR)}
    & \multicolumn{2}{c}{KMeans-DTW}
    & \multicolumn{2}{c}{CLDS} 
    & \multicolumn{2}{c}{K-Shape} \\ \hline 
    \multicolumn{1}{c}{Dataset}
	& NMI & Time & NMI & Time & NMI & Time & NMI & Time & NMI & Time \\ \hline
    Beef & \Bd{0.29} & \Em{1.1} & \Em{0.27} & \Bd{1.0} & 0.25 & 377 & 0.24 & 61.3 & 0.22 & 1.8 \\
    DPTW & 0.52 & \Em{0.6} & \Bd{0.56} & \Bd{0.5} & \Em{0.55} & 182 & \Em{0.55} & 176.8 & 0.45 & 14.9 \\
    PPOAG & \Bd{0.56} & \Em{0.5} & \Em{0.54} & \Bd{0.2} & 0.44 & 105.4 & 0.55 & 191.1 & 0.27 & 40.2 \\
    IWBS & \Bd{0.43} & \Em{43.9} & 0.36 & \Bd{6.3} & 0.37 & 5676 & 0.38 & 1109 & \Bd{0.43} & 377.6 \\
    TWOP & 0.23 & \Em{11.2} & \Em{0.3} & \Bd{4.7} & 0.12 & 1960 & 0.02 & 1312 & \Bd{0.4} & 292.1 \\
    ECG5T & \Em{0.46} & \Em{25.7} & 0.4 & \Bd{7.0} & \Bd{0.48} & 2539 & 0.37 & 1308 & 0.35 & 360.7 \\
    MALLAT & \Bd{0.92} & \Em{48.2} & \Em{0.91} & \Bd{25.4} & 0.72 & 95218 & \Bd{0.92} & 2448 & 0.75 & 900.4 \\
    NIFECG & \Em{0.71} & \Em{346.1} & 0.68 & \Bd{43.7} & 0.63 & 101473 & 0.67 & 3442 & \Bd{0.73} & 5387  \\ \hline 
    \end{tabular}
\end{center}
\vspace{-4mm}
\end{table*}

\subsection{Comparing Feature Representations}
\label{Chapter:comparisons of feature representations}
\textbf{Baselines and Setup.} We compare our approach with two recently developed methods: 1) TSEigen \citep{hayashi2005embedding}: learn a low-rank feature representation for a similarity matrix computed using DTW distance through Singular Value Decomposition \citep{wu2015preconditioned, wu2017primme_svds}; 2) TSMC \citep{QiYi2016}: a recently proposed similarity preserving representation for DTW-based similarity matrix using matrix completion approach. We set $R = 32$ for all methods. We employ a linear SVM implemented in LIBLINEAR \citep{fan2008liblinear} since it can separate the effectiveness of the feature representation from the power of the nonlinear learning solvers. 

\textbf{Results.} Table \ref{tb:comp_rf_mc_eigen} clearly demonstrates the significant advantages of our approach compared to other representations in terms of both classification accuracy and computational time. Indeed, TSMC improves the computational efficiency compared to TSEigen without compromising large loss of the accuracy as claimed in \citep{QiYi2016}. However, RWS is corroborated to achieve both higher accuracy and faster train and testing time compared to TSMC and TSEigen. The improved accuracy of RWS suggests that a truly p.d. time series kernel admits better feature representations than those obtained from a similarity or kernel (not p.d.) matrix. In addition, improved computational time illustrates the effectiveness of using random series to approximate the exact kernel.

\subsection{Comparing Time-Series Classification}
\textbf{Baselines.} We now compare our method with other state-of-the-art time series classification methods that also take advantage of DTW distance or employ DTW-like kernels: 1) 1NN-DTW: use window size $min(L/10, 40)$; 2) 1NN-DTW\textsuperscript{opt}: use optimal window size using leave-one-out cross validation from test data in \citep{UCRArchive} 
3) DTWF \citep{kate2016using}: a recently proposed method that combines DTW without and with constraints and SAX \citep{lin2007experiencing} as features; 4) TGAK \citep{cuturi2011fast}: a fast triangular global alignment kernel for time-series; 5) RWS(LR): RWS with large rank that achieves the best accuracy with more computational time; 6) RWS(SR): small rank that obtains comparable accuracy in less time. We conduct grid search for important parameters in each method suggested in \citep{kate2016using, cuturi2011fast}. 

\textbf{Results.} Table \ref{tb:comp_allmethods_classification} corroborates that RWS consistently outperforms or matches other state-of-the-art methods in terms of testing accuracy while requiring significantly less computational time. First, RWS(SR) can achieve  better or similar performance compared to 1NN-DTW and 1NN-DTW \textsuperscript{opt} for all datasets. This is a strong sign that our learned feature representation is very effective, since, using it, even a linear SVM can beat the well-recognized benchmark. Meanwhile, the clear computational advantages of RWS over 1NN-DTW can be observed when the number or the length of time series samples become large. This is not surprising since RWS reduces both number and length of time series from quadratic complexity to linear complexity. Second, RWS is much better than another family of time series kernels represented by TGAK, which probably indicates that considering the soft-minimum of all alignment distances does not capture well hidden patterns of time series. Third, DTWF shows significant performance difference compared to 1NN-DTW, which is consistent with the reported results in \citep{kate2016using}. However, compared to DTWF, RWS(LR) can still show clear advantages in accuracy among 11 cases out of the total 16 datasets while achieving one or two orders of magnitude speedup. More importantly, RWS can support a trade-off between the accuracy and run-time. This feature is highly desirable in real applications that may have a variety of priorities and constraints.

\subsection{Comparing Time-Series Clustering}
\textbf{Baselines.} We compare our method against several time-series clustering baselines: 1) KMeans-DTW \citep{petitjean2011global,paparrizos2015k}: accelerate computation with lower bounding approach $LB_{Keogh}$ \citep{keogh2002exact}; 2) CLDS \citep{li2011time}: learns a feature representation with $R$ hidden variables through complex-valued linear dynamical systems; 3) K-Shape \citep{paparrizos2015k}: recently proposed clustering method demonstrated to outperform state-of-the-art clustering approaches in accuracy and computational time; 4) RWS(LR); 5) RWS(SR). We combine our learned feature representation with the classic KMeans algorithm \citep{hartigan1979algorithm}. We employ a commonly used clustering metric, the normalized mutual information (NMI scaling between 0 and 1) to measure the performance, where higher value indicates better accuracy. 

\textbf{Results.} Table \ref{tb:comp_allmethods_clustering} shows that RWS provides similar or better performance and typically is substantially faster than KMeans-DTW when the number or the length of time-series become large. In addition, RWS can consistently outperform CLDS in terms of both accuracy and runtime. Interestingly, even compared to the state-of-the-art method K-Shape, RWS can still yield a clear advantage in terms of accuracy; RWS yields 5 wins, 1 even, and 2 loses over K-Shape for 8 datasets. Besides its accuracy, the better computational efficiency of RWS over K-Shape is also corroborated.

 \section{Conclusions and Future Work}
 In this work, we have studied an effective and scalable time-series (p.d.) kernel for large-scale time series problems based on RWS approximation, and the feature embedding generated by the technique is generally applicable to most of learning problems. There are several interesting directions of future work, including: i) studying the effects of different random time-series distribution $p(\omega)$ and ii) exploring more elastic dissimilarity measure between time series such as CID and DTDC.

\clearpage
\newpage
\bibliographystyle{plainnat}
\bibliography{RWS}

\clearpage
\section{Appendix}

\subsection{Proof of Lemma \ref{lemma_lipschitz_f}}
\label{sec:proof_lemma}

\begin{lemma}[Lipschitz parameter]\label{lemma_lipschitz_f}
Let $\tau(A\omega,Bx)$ be Lipschitz-continuous w.r.t. $x$ with parameter $\beta(\omega)$. Then given $\|x_1-x_2\|\leq \epsilon$ and $\|y_1-y_2\|\leq\epsilon$, we have
$$
f(x,y)=s_{R}(x,y)-k(x,y)
$$
satisfies
$$
|f(x,y)-f(x',y')|\leq 2\gamma\sigma_{R}\epsilon t.
$$
with probability at least $1-1/t^2$, where $\sigma_{\tau}^2=Var[\beta(\omega)]/R$ is the variance of the Lipschitz parameter averaged over $R$ samples.
\end{lemma}

\begin{proof}
Consider an arbitrary pair of series $(x_1,x_2)$ $\in\X$ of the same length. From Lipschitz-continuity of $\tau(.)$, we have
$$
\tau(A\omega,Bx_2)=\tau(A\omega,Bx_1)+\beta(\omega)\Delta
$$
for some $|\Delta| \leq \|x_2-x_1\|$. Then let $(A_1,B_1)$ and $(A_2,B_2)$ be the minimizers of $\tau(A\omega,Bx_1)$ and $\tau(A\omega,Bx_2)$ respectively, we have
$$
 \tau(A_2\omega,B_2x_2) \leq \tau(A_1\omega,B_1x_2) \leq \tau(A_1\omega, B_1x_1) + \beta(\omega)|\Delta|.
$$
and
$$
 \tau(A_1\omega,B_1x_1) \leq \tau(A_2\omega,B_2x_1) \leq \tau(A_2\omega,B_2x_2) + \beta(\omega)|\Delta|.
$$
Therefore, 
$
\phi_{\omega}(x_2)=\phi_{\omega}(x_1) + \beta(\omega)\Delta
$
and
\begin{align*}
s_{R}(x_2,y_2)&=\frac{1}{R}\sum_{i=1}^R \phi_{\omega_i}(x_2)\phi_{\omega_i}(y_2)\\
&\leq\frac{1}{R}\sum_{i=1}^R \phi_{\omega_i}(x_1)\phi_{\omega_i}(y_2) + r\tbeta\epsilon \leq s_R(x_1,y_1) + 2r\tbeta \epsilon.
\end{align*}
where $\tbeta=\frac{1}{R}\sum_{i=1}^R\beta(\omega_i)$. With the similar argument we have
$
|s_R(x_2,y_2)-s_R(x_1,y_1)|\leq 2\gamma\tbeta\epsilon
$
and
$
|k(x_2,y_2)-k(x_1,y_1)|\leq 2\gamma\bbeta\epsilon.
$
Then since $E[\tbeta]=\bbeta$ and Let $Var[\tbeta]=\sigma_{\tau}^2$. By Chebyshev inequality, we have 
$$
P\left[ |f(x_1,y_1)-f(x_2,y_2)|  \geq 2\gamma\sigma_{\tau}\epsilon*t \right] \leq \frac{1}{t^2}
$$
\end{proof}

\subsection{Experimental settings and parameters for RWS}
\label{sec:Experimental settings and parameters for RWS}
As shown in Table \ref{tb:info of datasets}, we choose 16 datasets that come from various applications, including ECG, sensor, image, spectro, simulated and device, and have various numbers of classes, varying numbers of time series, and a wide range of lengths of time series, as shown in Table \ref{tb:info of datasets}.  
For all experiments, we generate random document from uniform distribution with mean centered in Word2Vec embedding space since we observe the best performance with this setting. We perform 10-fold cross-validation to search for best parameters for $\sigma$, and $DMax$ as well as parameter $C$ for LIBLINEAR on training set for each dataset. We simply fix the $DMin=1$, and vary $DMax$ in the range of [10 20 30 40 50 60 70 80 90 100], $\sigma$ in the range of [1e-4 1e-3 3e-3 1e-2 3e-2 0.10 0.14 0.19 0.28 0.39 0.56 0.79 1.12 1.58 2.23 3.16 4.46 6.30 8.91 10 31.62 1e2 3e2 1e3 1e4], and $C$ in the range of [1e-5 1e-4 1e-3 1e-2 1e-1 1 1e1 1e2 1e3 1e4 1e5] respectively in all experiments. 
All computations were carried out on a DELL dual socket system with Intel Xeon processors 272 at 2.93GHz for a total of 16 cores and 250 GB of memory, running the SUSE Linux operating system.

\begin{table}[htbp]
\centering
\small
\caption{Properties of the datasets: Beef, ChlorineConcentration (CHCO), DistalPhalanxTW (DPTW), ECG5000 (ECG5T), FordB, HandOutlines (HO), InsectWingbeatSound (IWBS), ItalyPowerDemand (IPD), LargeKitchenAppliances (LKA), MALLAT, MiddlePhalanxOutlineCorrect (MPOC), NonInvasiveFatalECG\_Thorax2 (NIFECG), PhalangesOutlinesCorrect (POC), ProximalPhalanxOutlineAgeGroup (PPOAG), Two\_Patterns (TWOP), and Wafer. We define $C$:Classes, $N$:Train, $M$:Test, and $L$:length.} 
\vspace{0mm}
\label{tb:info of datasets_sup}
\begin{center}
    \begin{tabular}{ cccccc}
    \hline
    Name    & $C$ & $N$ & $M$ & $L$ & App\\ \hline 
    Beef    & 5 & 30 & 30 & 470 & Spectro    \\ 
    DPTW    & 6 & 400 & 139	& 80 & Image \\ 
    IPD     & 2 & 67 & 1,029 & 24 & Sensor \\ 
    PPOAG   & 3 & 400 & 205	& 80 & Image \\ 
    MPOC    & 2 & 600 & 291	& 80 & Image \\ 
    POC     & 2 & 1,800 & 858 & 80  & Image \\ 
    LKA     & 3 & 375 & 375	& 720 & Device \\ 
    IWBS    & 11 & 220 & 1,980 & 256 & Sensor \\ 
    TWOP    & 4 & 1,000 & 4,000 & 128 & Simulated \\ 
    ECG5T   & 5 & 500 & 4,500 & 140 & ECG \\ 
    CHCO    & 3 & 467 &	3,840 & 166 & Simulated \\ 
    Wafer 	& 2 & 1,000 & 6,174 & 152 & Sensor \\ 
    MALLAT  & 8 & 55 & 2,345 & 1,024 & Simulated \\ 
    FordB   & 2 & 3636 & 810 & 500 & Sensor \\ 
    NIFECG  & 42 & 1,800 & 1,965 & 750 & ECG \\ 
    HO      & 2 & 370 & 1,000 & 2,709 & Image \\ \hline
    \end{tabular}
\end{center}
\vspace{-4mm}
\end{table}

\subsection{More Results on Effects of $\sigma$, $R$ and $D$ on Random Features}
\label{sec:More Results on various effects of random features}
To fully investigate the behavior of the WME method, we study the effect of the kernel parameter $\sigma$, the $R$ number of random documents and the $D$ length of random documents on training and testing accuracy for  all 16 datasets. Clearly, the training and testing accuracy can converge rapidly to the exact kernels when varying R from 4 to 512, which confirms our analysis in Theory 1. When varying D from 10 to 100, we can see that in the majority of cases $DMax = [10 \ 40]$ generally yields a near-peak performance except FordB.

\begin{figure*}[!htb]
\centering
  \begin{subfigure}[b]{0.24\textwidth}
      \includegraphics[width=\textwidth]{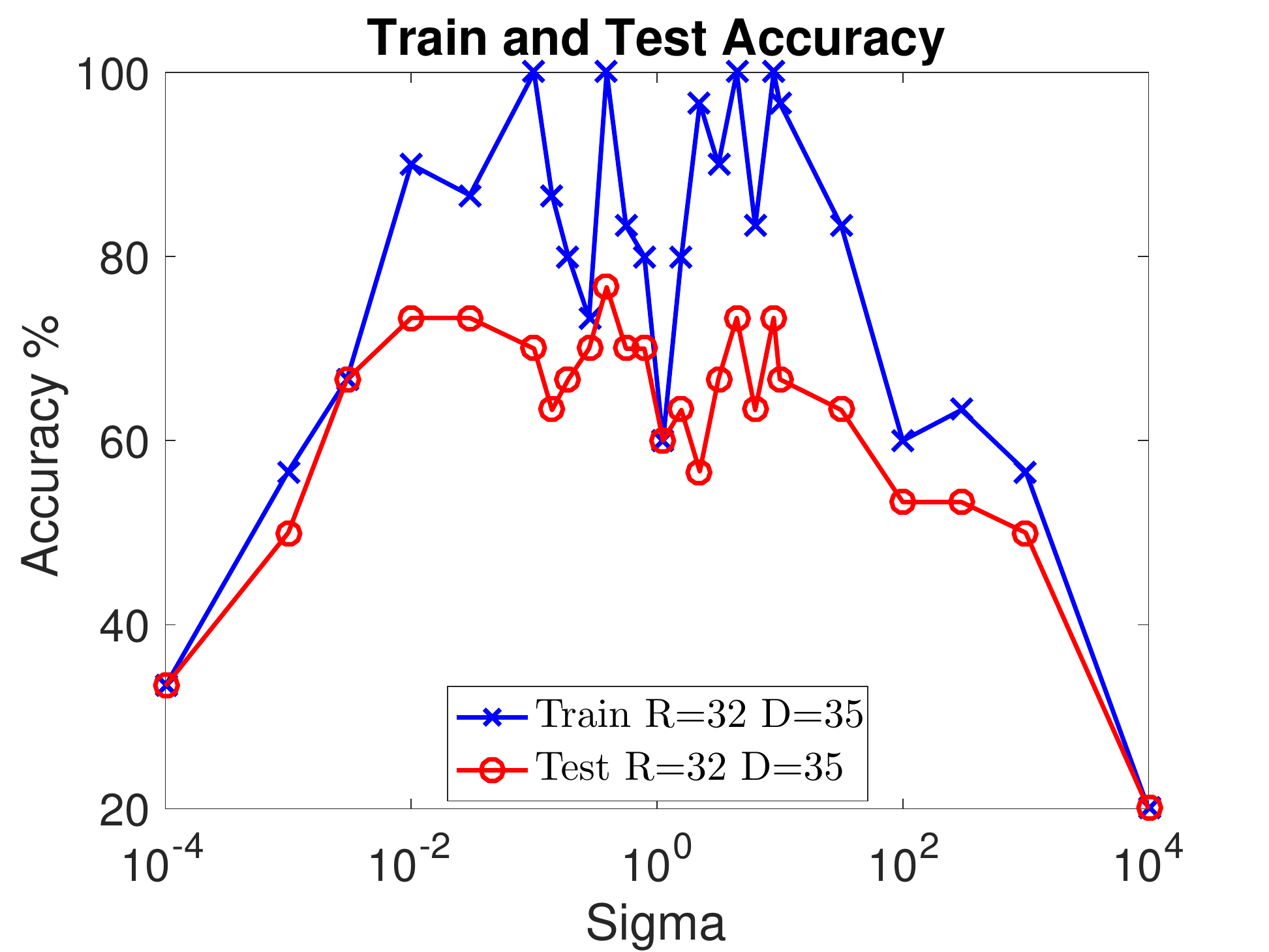}
      \caption{Beef}
      \label{fig:exptsA_varyingS_Beef}
    \end{subfigure}
  \begin{subfigure}[b]{0.24\textwidth}
      \includegraphics[width=\textwidth]{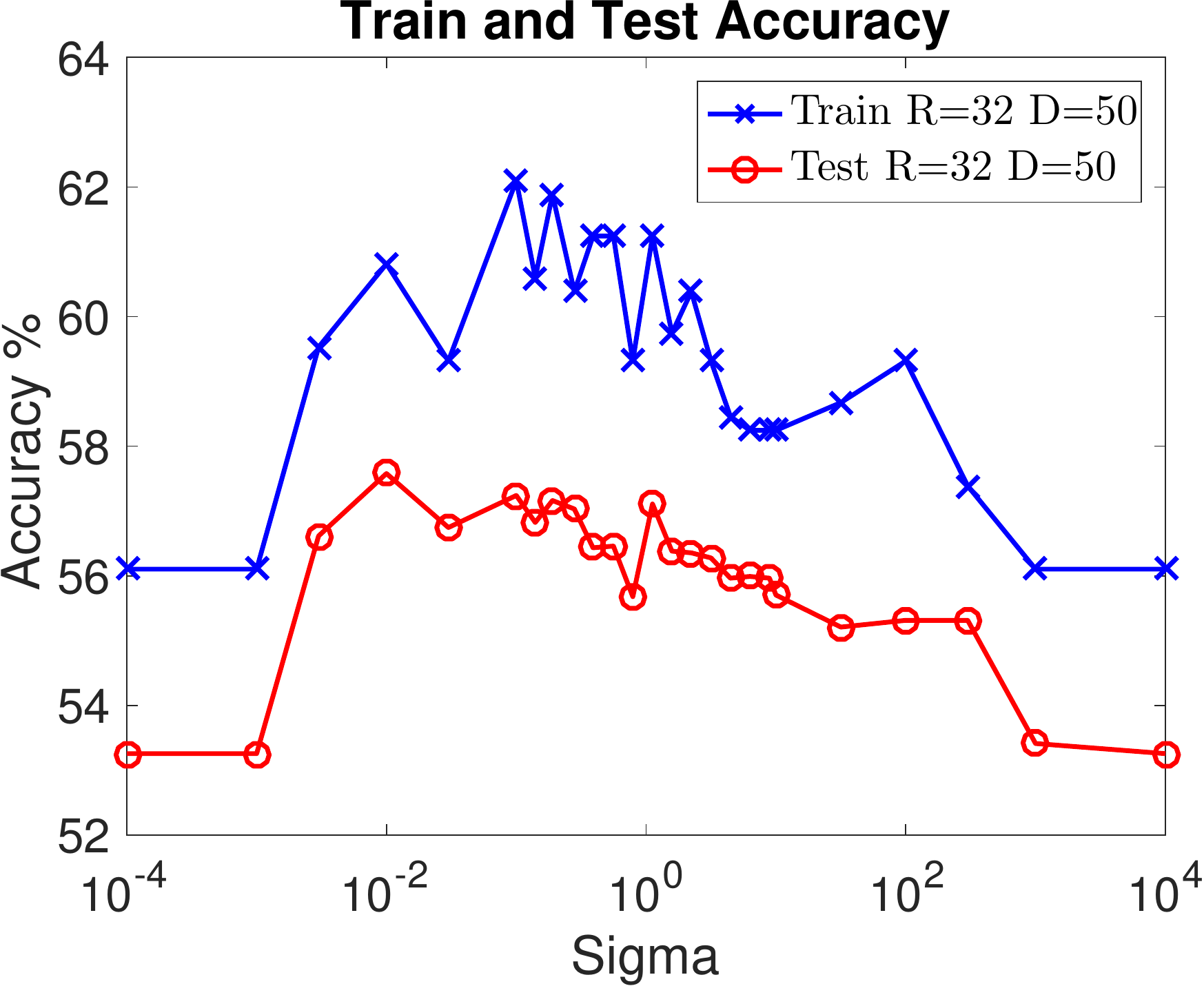}
      \caption{CHCO}
      \label{fig:exptsA_varyingS_ChlorineConcentration}
    \end{subfigure}
  \begin{subfigure}[b]{0.24\textwidth}
      \includegraphics[width=\textwidth]{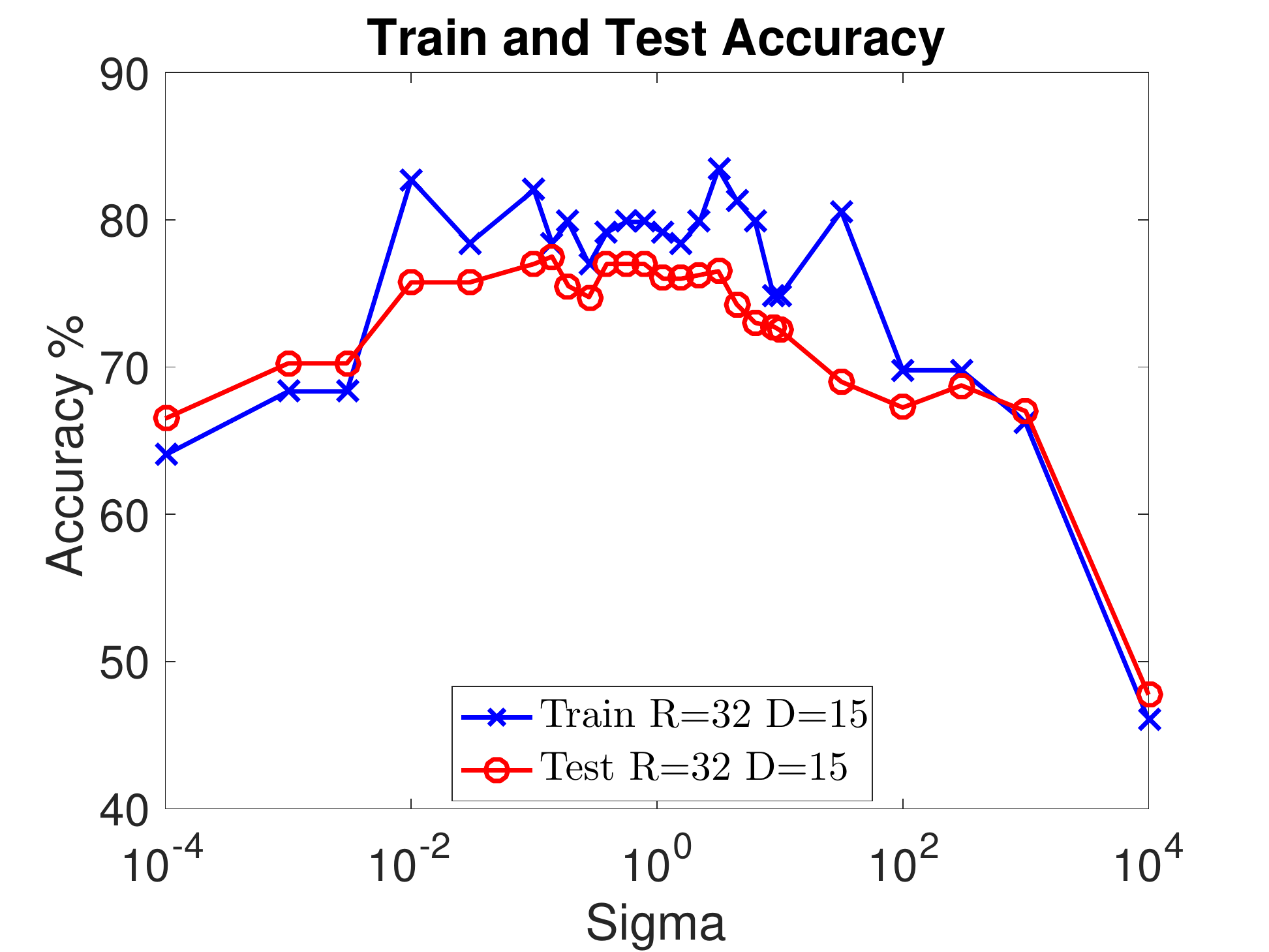}
      \caption{DPTW}
      \label{fig:exptsA_varyingS_DistalPhalanxTW}
    \end{subfigure}
  \begin{subfigure}[b]{0.24\textwidth}
      \includegraphics[width=\textwidth]{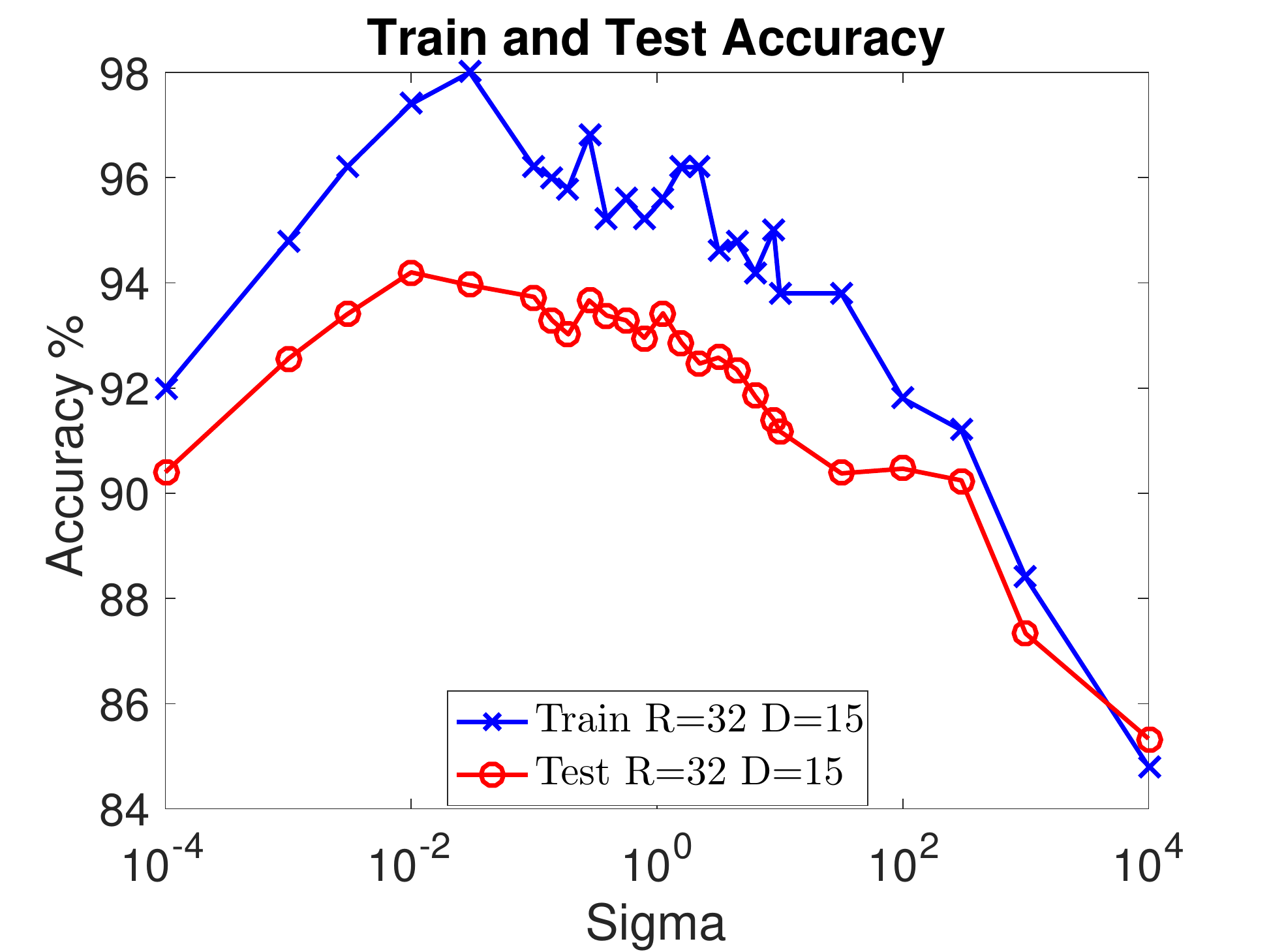}
      \caption{ECG5T}
      \label{fig:exptsA_varyingS_ECG5000}
    \end{subfigure}
  \begin{subfigure}[b]{0.24\textwidth}
      \includegraphics[width=\textwidth]{Graphs/ExptsA_varyingS/FordB_Accu_VaryingS-eps-converted-to.pdf}
      \caption{FordB}
      \label{fig:exptsA_varyingS_FordB}
    \end{subfigure}
  \begin{subfigure}[b]{0.24\textwidth}
      \includegraphics[width=\textwidth]{Graphs/ExptsA_varyingS/HandOutlines_Accu_VaryingS-eps-converted-to.pdf}
      \caption{HO}
      \label{fig:exptsA_varyingS_HandOutlines}
    \end{subfigure}
    \begin{subfigure}[b]{0.24\textwidth}
      \includegraphics[width=\textwidth]{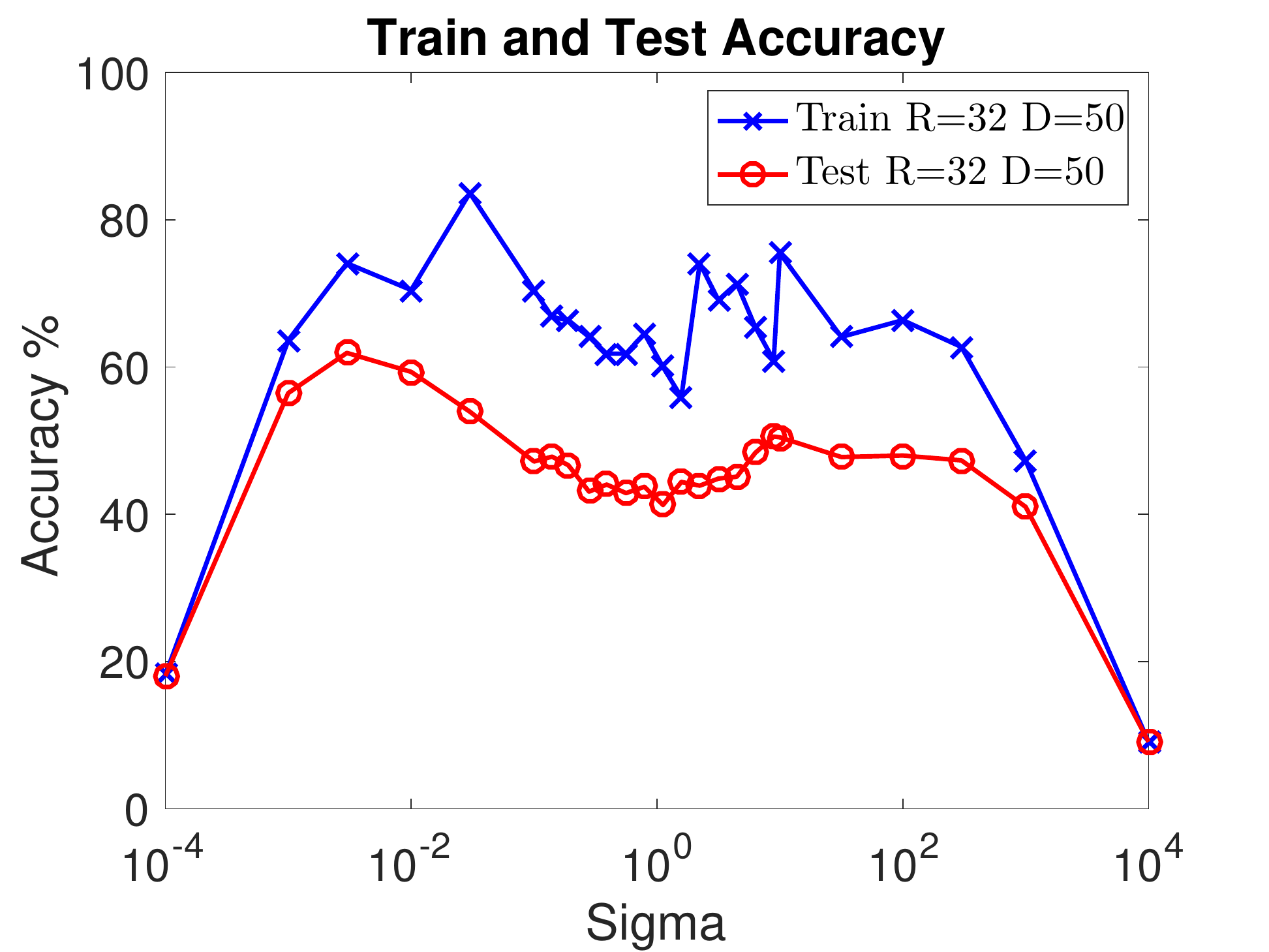}
      \caption{IWBS}
      \label{fig:exptsA_varyingS_InsectWingbeatSound}
    \end{subfigure}
  \begin{subfigure}[b]{0.24\textwidth}
      \includegraphics[width=\textwidth]{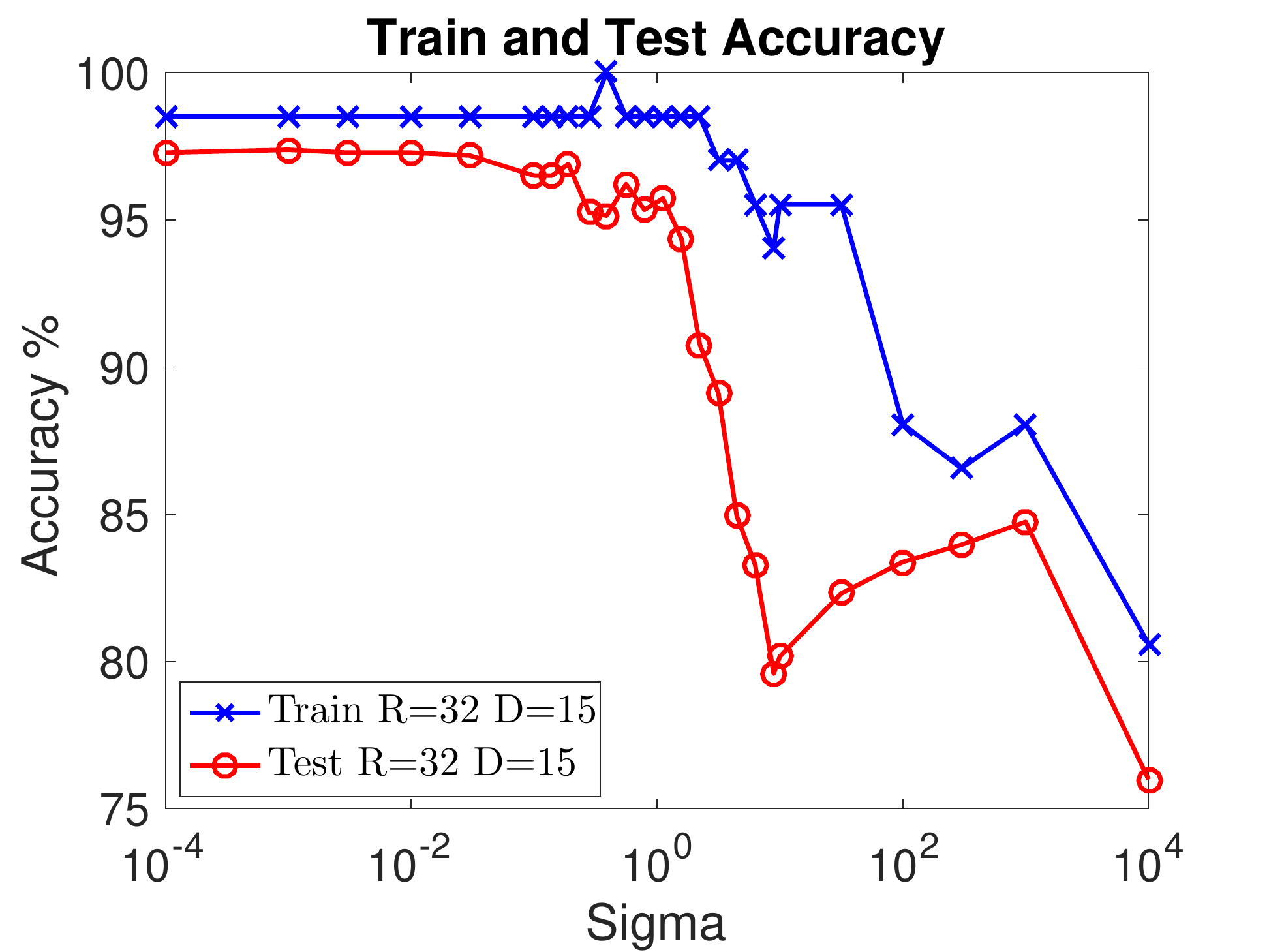}
      \caption{IPD}
      \label{fig:exptsA_varyingS_ItalyPowerDemand}
    \end{subfigure}
  \begin{subfigure}[b]{0.24\textwidth}
      \includegraphics[width=\textwidth]{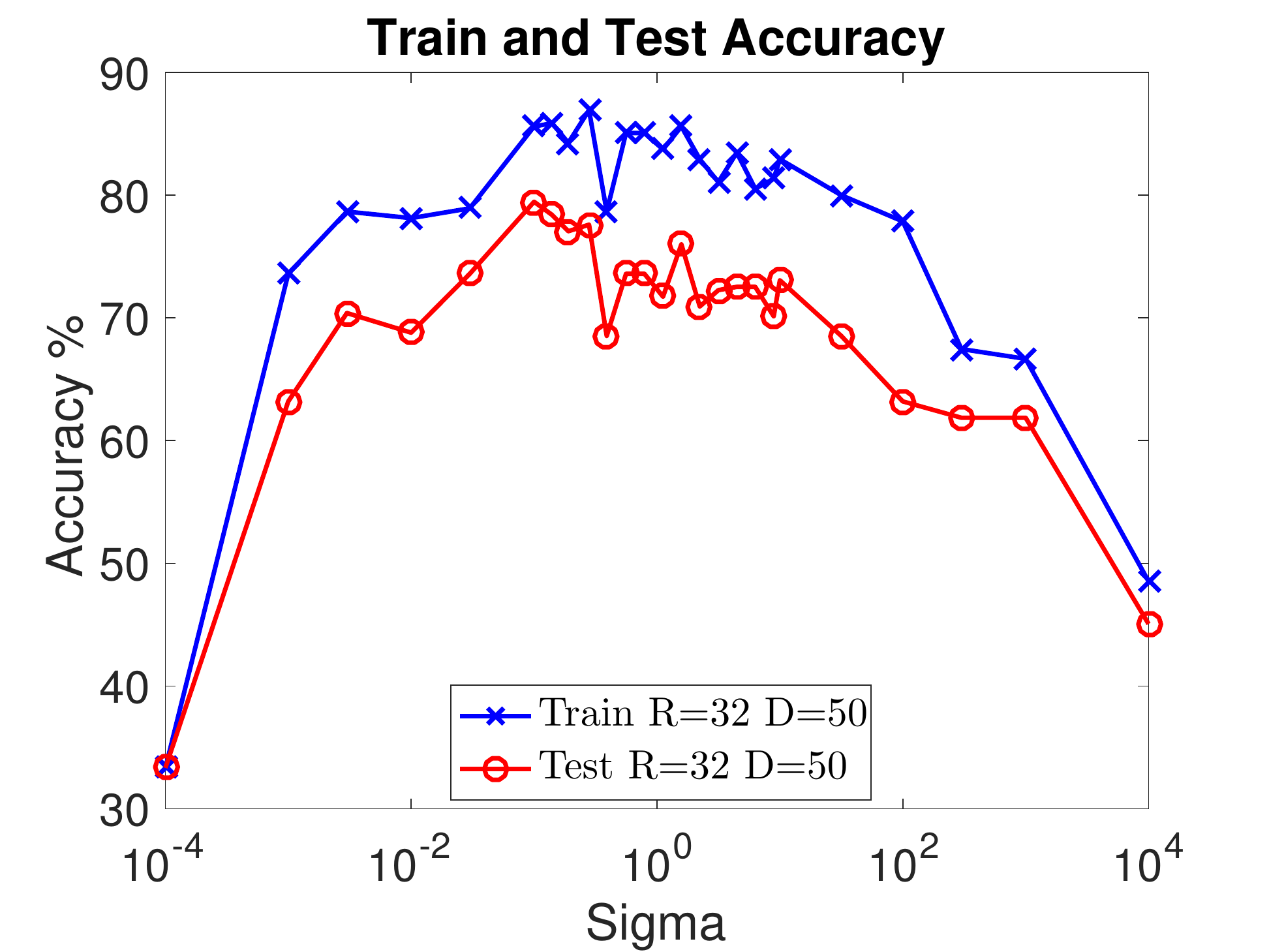}
      \caption{LKA}
      \label{fig:exptsA_varyingS_LargeKitchenAppliances}
    \end{subfigure}
  \begin{subfigure}[b]{0.24\textwidth}
      \includegraphics[width=\textwidth]{Graphs/ExptsA_varyingS/MALLAT_Accu_VaryingS-eps-converted-to.pdf}
      \caption{MALLAT}
      \label{fig:exptsA_varyingS_MALLAT}
    \end{subfigure}
  \begin{subfigure}[b]{0.24\textwidth}
      \includegraphics[width=\textwidth]{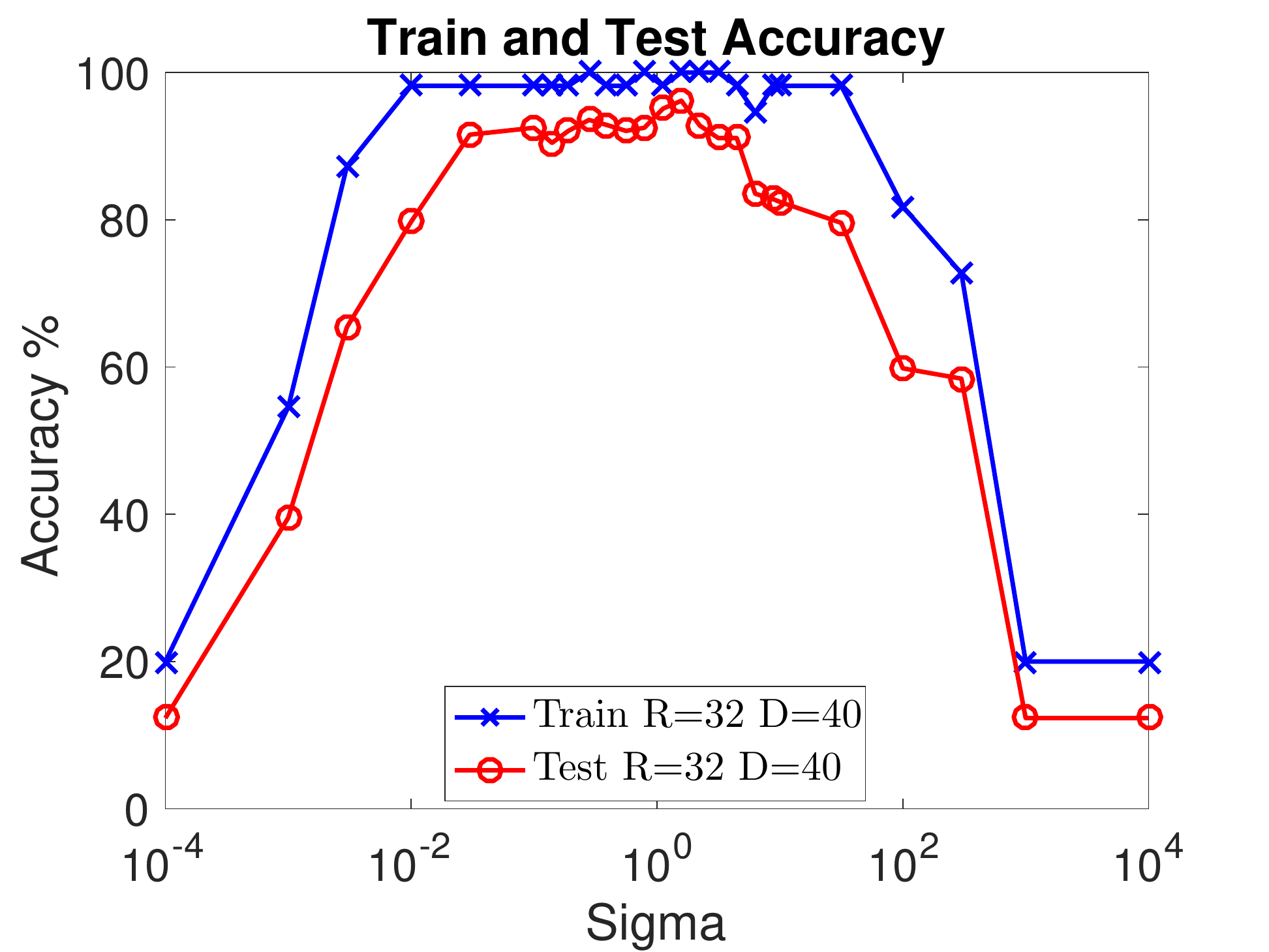}
      \caption{MPOC}
      \label{fig:exptsA_varyingS_MiddlePhalanxOutlineCorrect}
    \end{subfigure}
  \begin{subfigure}[b]{0.24\textwidth}
      \includegraphics[width=\textwidth]{Graphs/ExptsA_varyingS/NonInvasiveFatalECG_Thorax2_Accu_VaryingS-eps-converted-to.pdf}
      \caption{NIFECG}
      \label{fig:exptsA_varyingS_NonInvasiveFatalECG_Thorax2}
    \end{subfigure}
  \begin{subfigure}[b]{0.24\textwidth}
      \includegraphics[width=\textwidth]{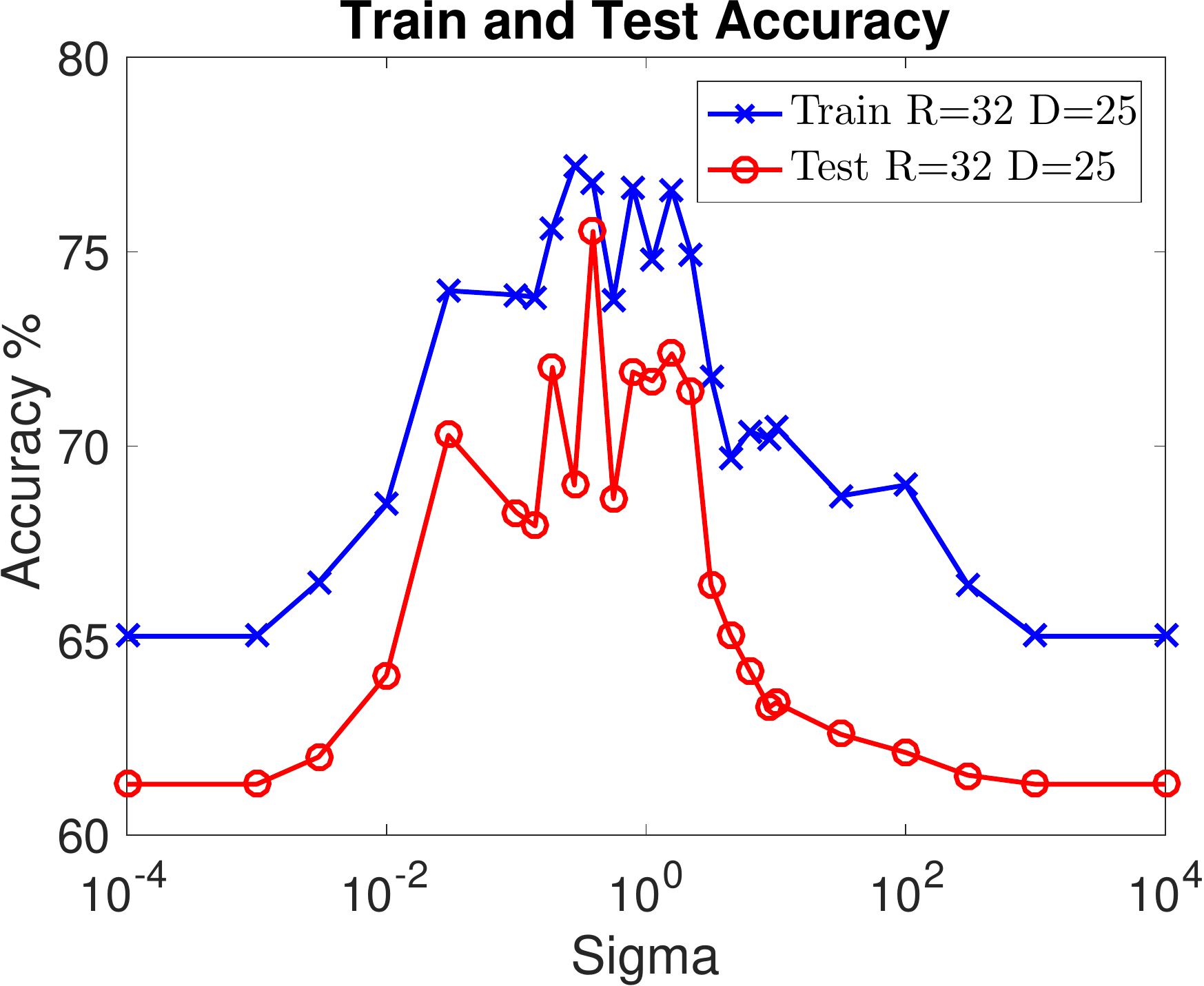}
      \caption{POC}
      \label{fig:exptsA_varyingS_PhalangesOutlinesCorrect}
    \end{subfigure}
  \begin{subfigure}[b]{0.24\textwidth}
      \includegraphics[width=\textwidth]{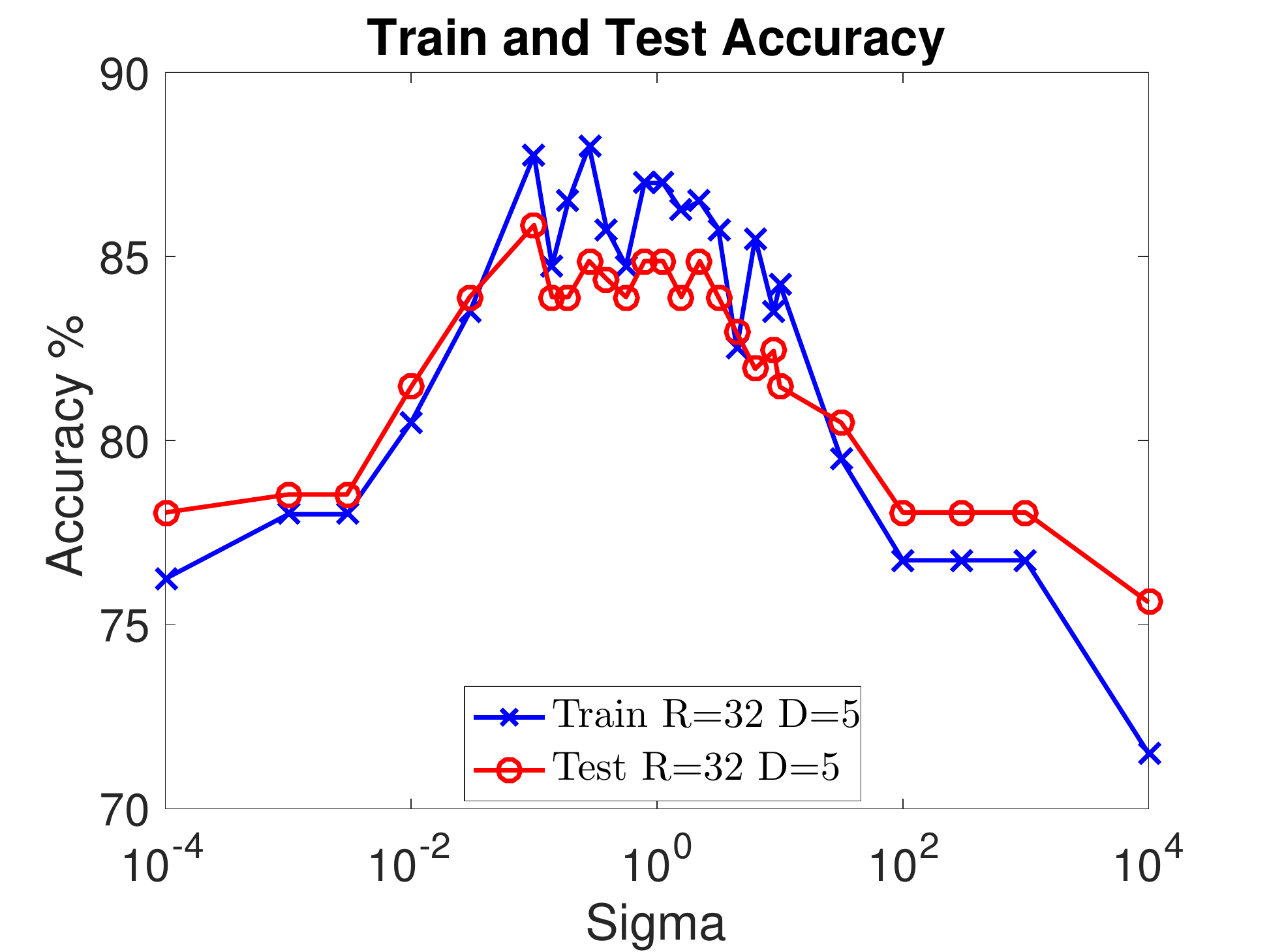}
      \caption{PPOAG}
      \label{fig:exptsA_varyingS_ProximalPhalanxOutlineAgeGroup}
    \end{subfigure}
  \begin{subfigure}[b]{0.24\textwidth}
      \includegraphics[width=\textwidth]{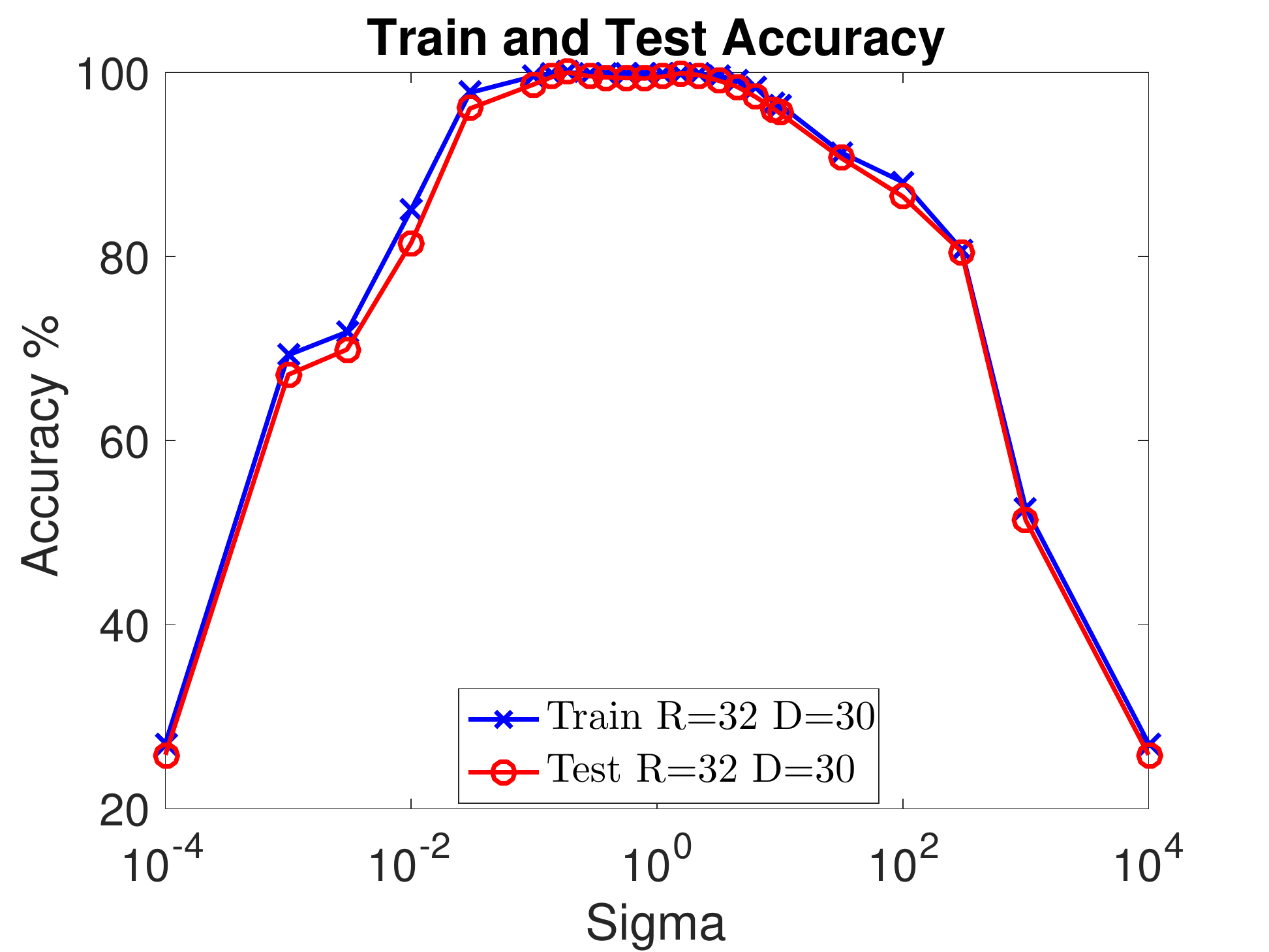}
      \caption{TWOP}
      \label{fig:exptsA_varyingS_Two_Patterns}
    \end{subfigure}
  \begin{subfigure}[b]{0.24\textwidth}
      \includegraphics[width=\textwidth]{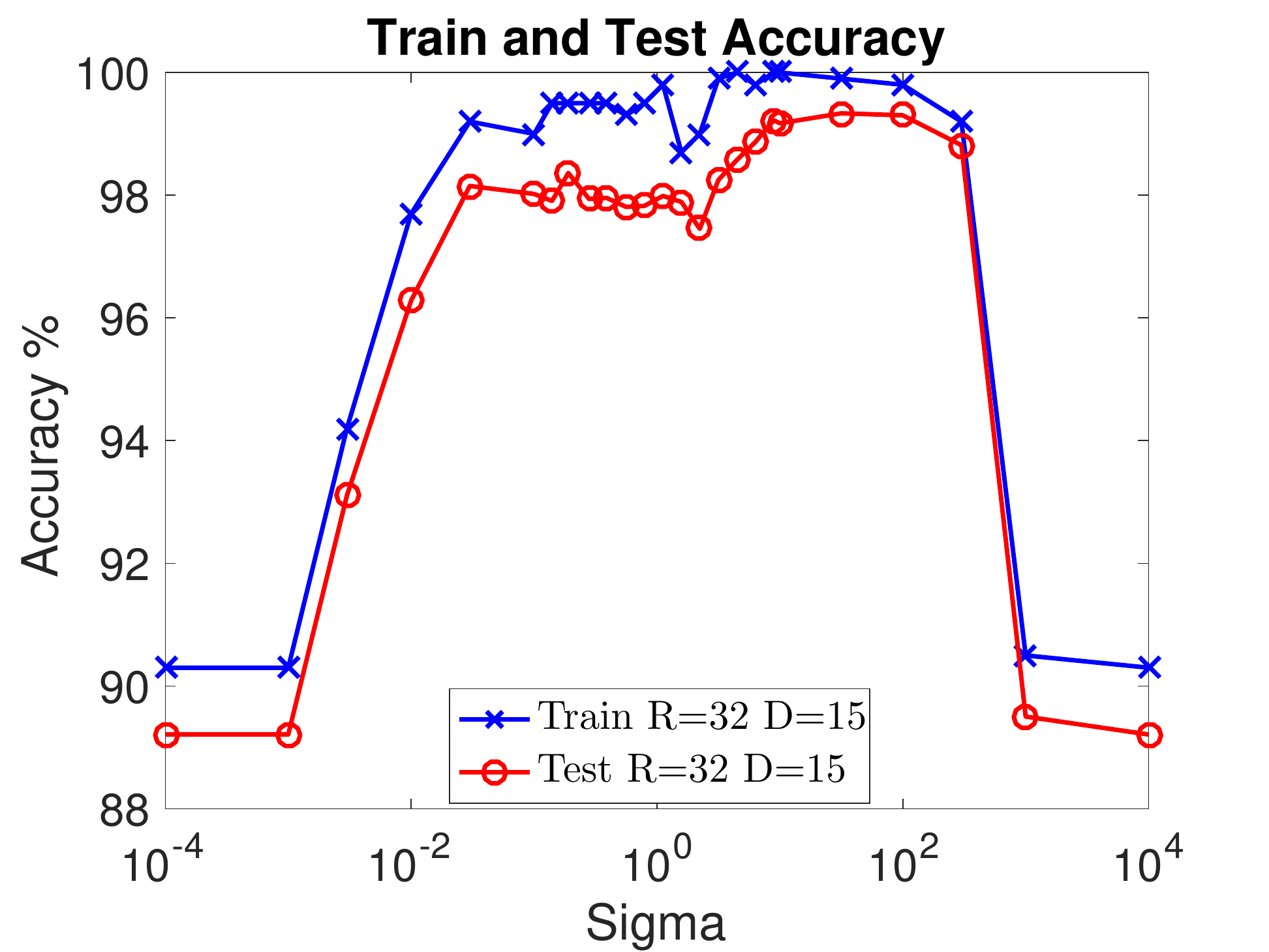}
      \caption{Wafer}
      \label{fig:exptsA_varyingS_wafer}
    \end{subfigure}
\vspace{-0mm}
\caption{Train (Blue) and test (Red) accuracy when varying $\sigma$ with fixed $D$ and $R$. We denote $D = DMax/2$. }
\vspace{-2mm}
\label{fig:exptsA_varyingS_sup}
\end{figure*}

\begin{figure*}[!htb]
\centering
  \begin{subfigure}[b]{0.24\textwidth}
      \includegraphics[width=\textwidth]{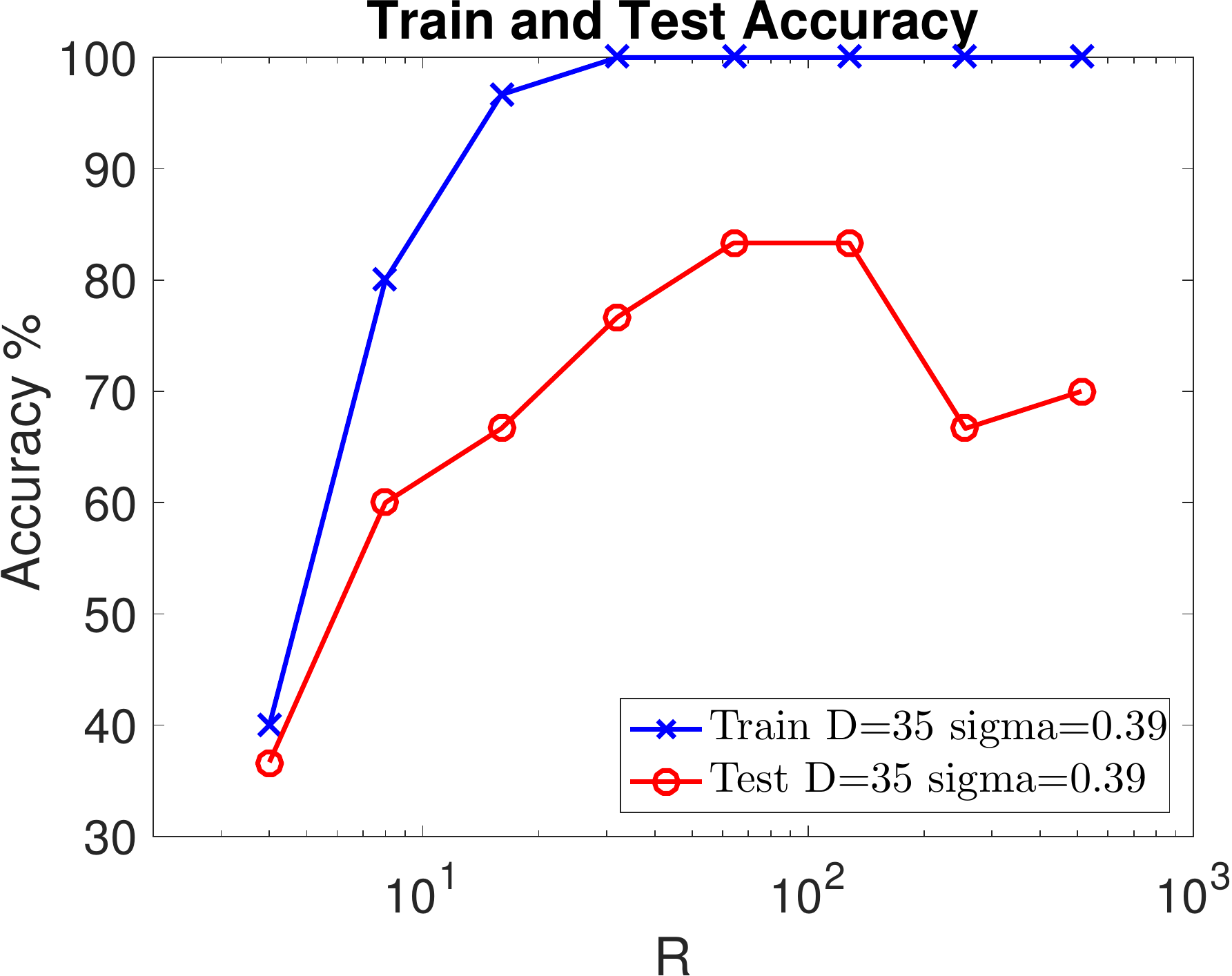}
      \caption{Beef}
      \label{fig:exptsA_varyingR_Beef}
    \end{subfigure}
  \begin{subfigure}[b]{0.24\textwidth}
      \includegraphics[width=\textwidth]{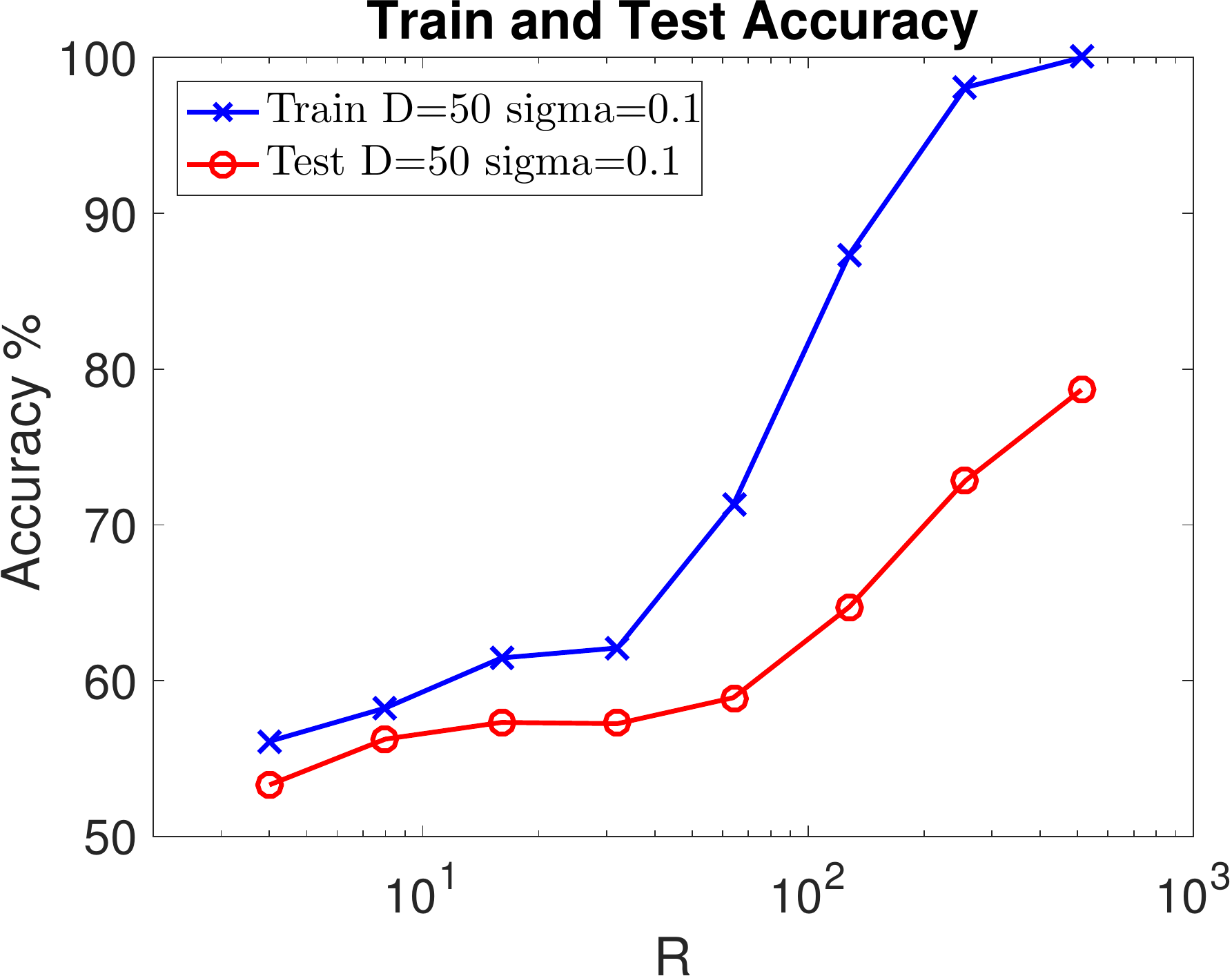}
      \caption{CHCO}
      \label{fig:exptsA_varyingR_ChlorineConcentration}
    \end{subfigure}
  \begin{subfigure}[b]{0.24\textwidth}
      \includegraphics[width=\textwidth]{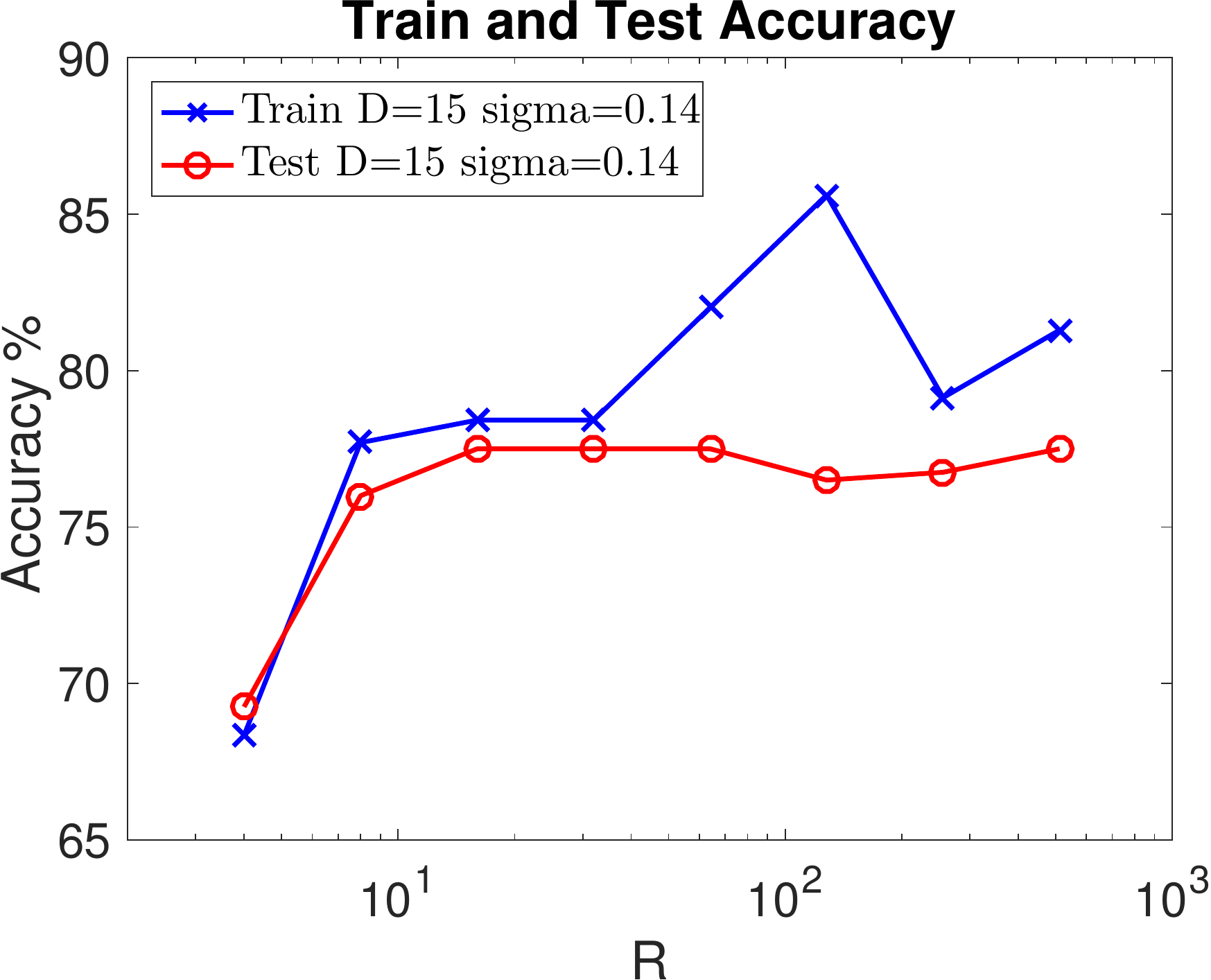}
      \caption{DPTW}
      \label{fig:exptsA_varyingR_DistalPhalanxTW}
    \end{subfigure}
  \begin{subfigure}[b]{0.24\textwidth}
      \includegraphics[width=\textwidth]{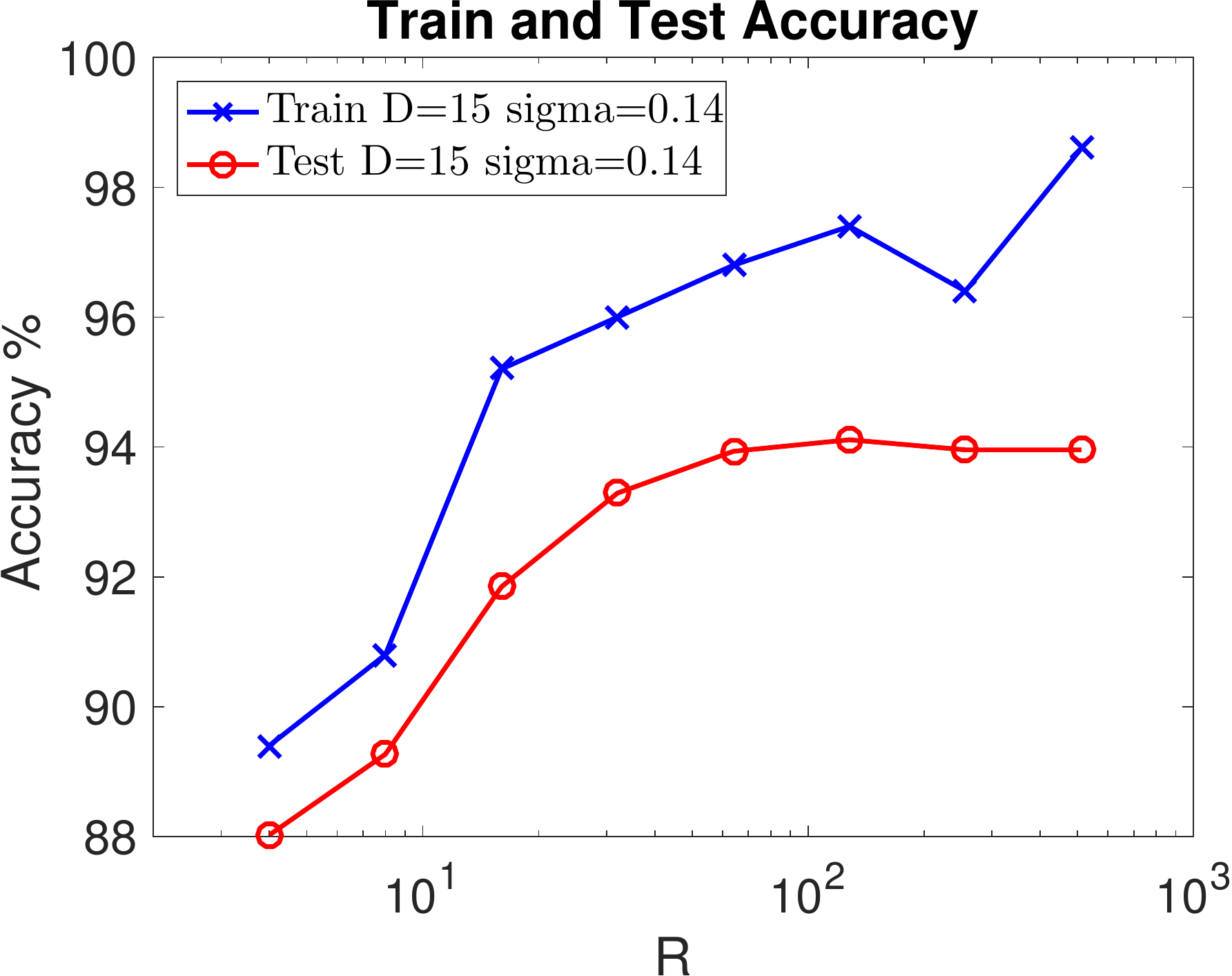}
      \caption{ECG5T}
      \label{fig:exptsA_varyingR_ECG5000}
    \end{subfigure}
  \begin{subfigure}[b]{0.24\textwidth}
      \includegraphics[width=\textwidth]{Graphs/ExptsA_varyingR/FordB_Accu_VaryingR-eps-converted-to.pdf}
      \caption{FordB}
      \label{fig:exptsA_varyingR_FordB}
    \end{subfigure}
  \begin{subfigure}[b]{0.24\textwidth}
      \includegraphics[width=\textwidth]{Graphs/ExptsA_varyingR/HandOutlines_Accu_VaryingR-eps-converted-to.pdf}
      \caption{HO}
      \label{fig:exptsA_varyingR_HandOutlines}
    \end{subfigure}
    \begin{subfigure}[b]{0.24\textwidth}
      \includegraphics[width=\textwidth]{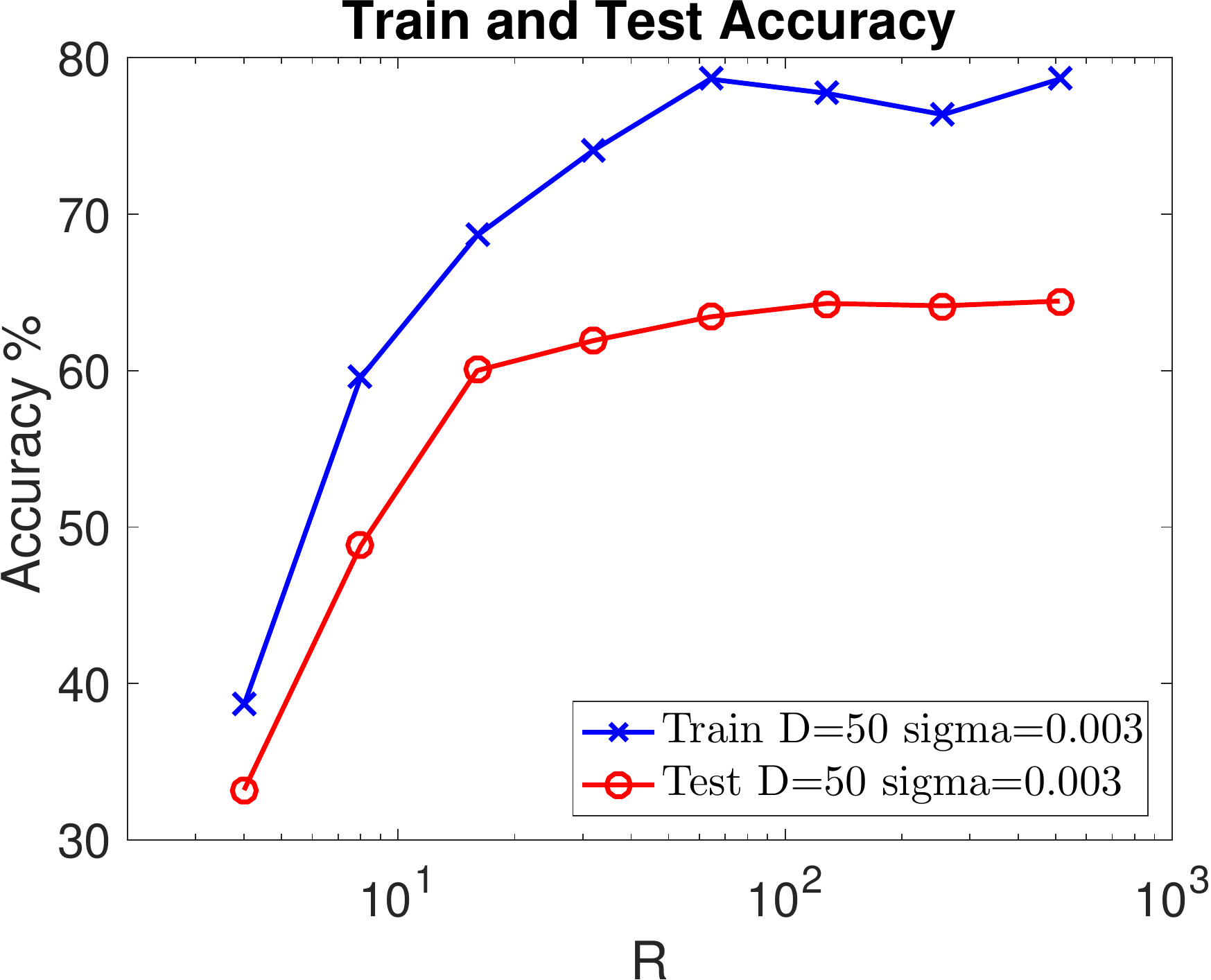}
      \caption{IWBS}
      \label{fig:exptsA_varyingR_InsectWingbeatSound}
    \end{subfigure}
  \begin{subfigure}[b]{0.24\textwidth}
      \includegraphics[width=\textwidth]{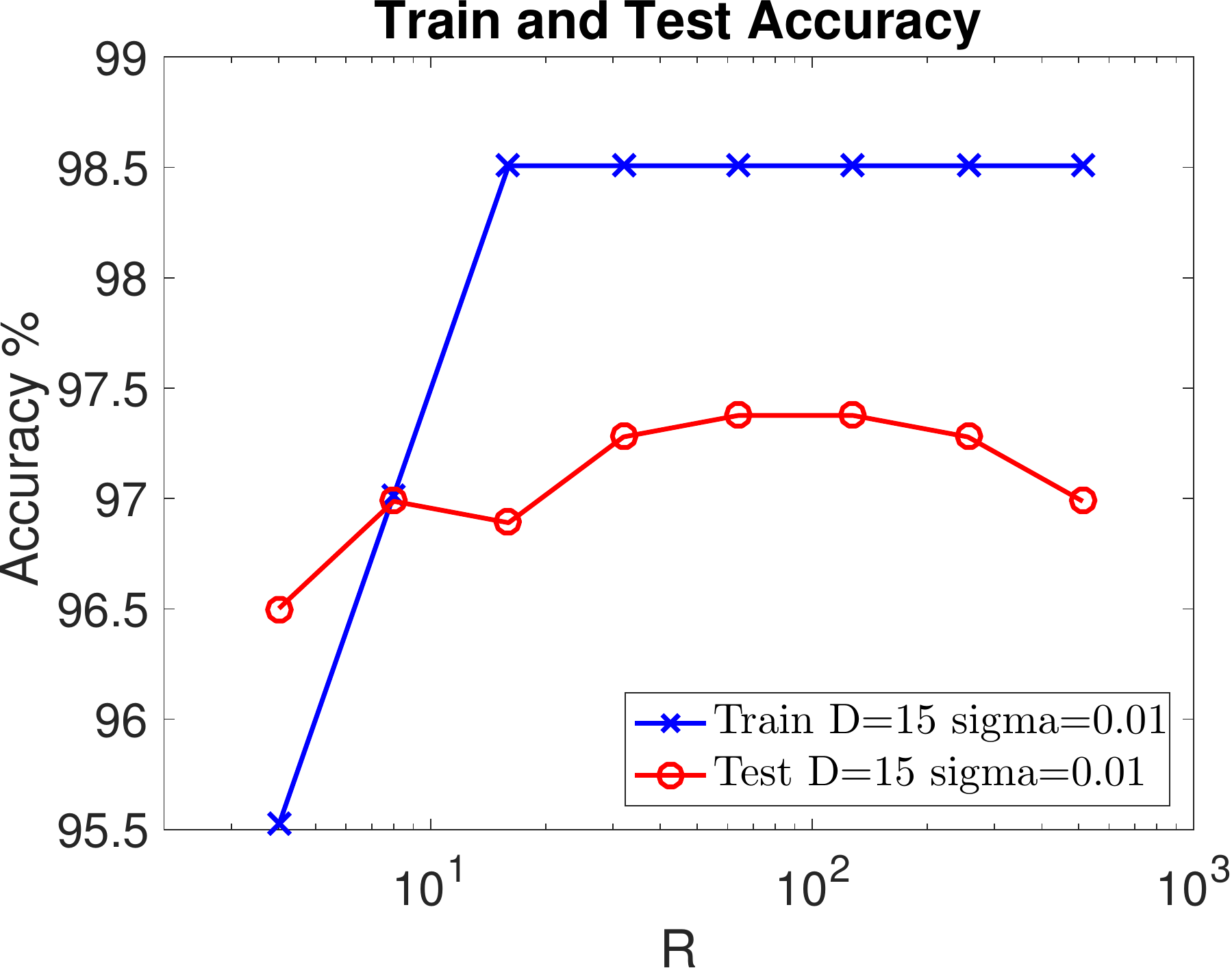}
      \caption{IPD}
      \label{fig:exptsA_varyingR_ItalyPowerDemand}
    \end{subfigure}
  \begin{subfigure}[b]{0.24\textwidth}
      \includegraphics[width=\textwidth]{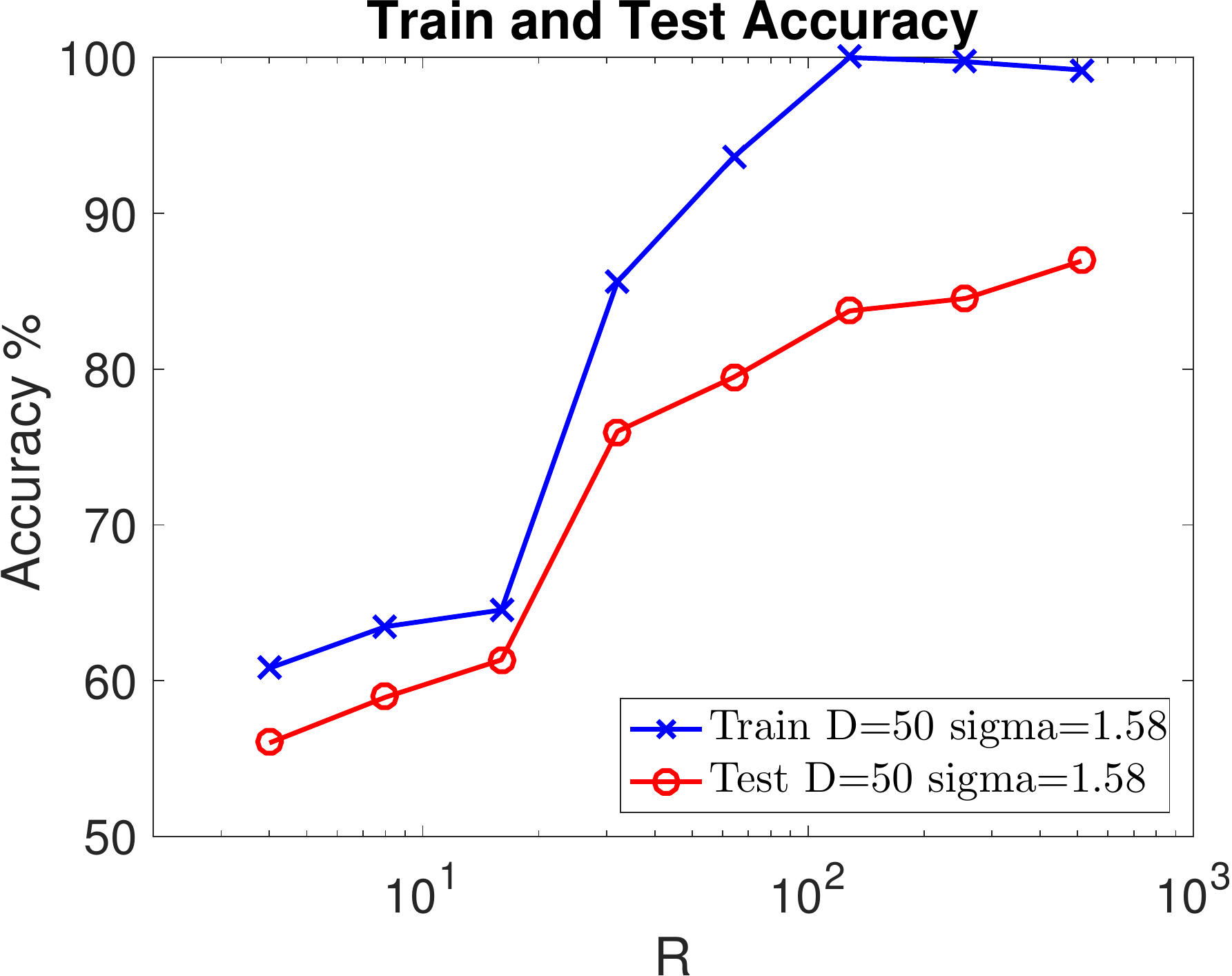}
      \caption{LKA}
      \label{fig:exptsA_varyingR_LargeKitchenAppliances}
    \end{subfigure}
  \begin{subfigure}[b]{0.24\textwidth}
      \includegraphics[width=\textwidth]{Graphs/ExptsA_varyingR/MALLAT_Accu_VaryingR-eps-converted-to.pdf}
      \caption{MALLAT}
      \label{fig:exptsA_varyingR_MALLAT}
    \end{subfigure}
  \begin{subfigure}[b]{0.24\textwidth}
      \includegraphics[width=\textwidth]{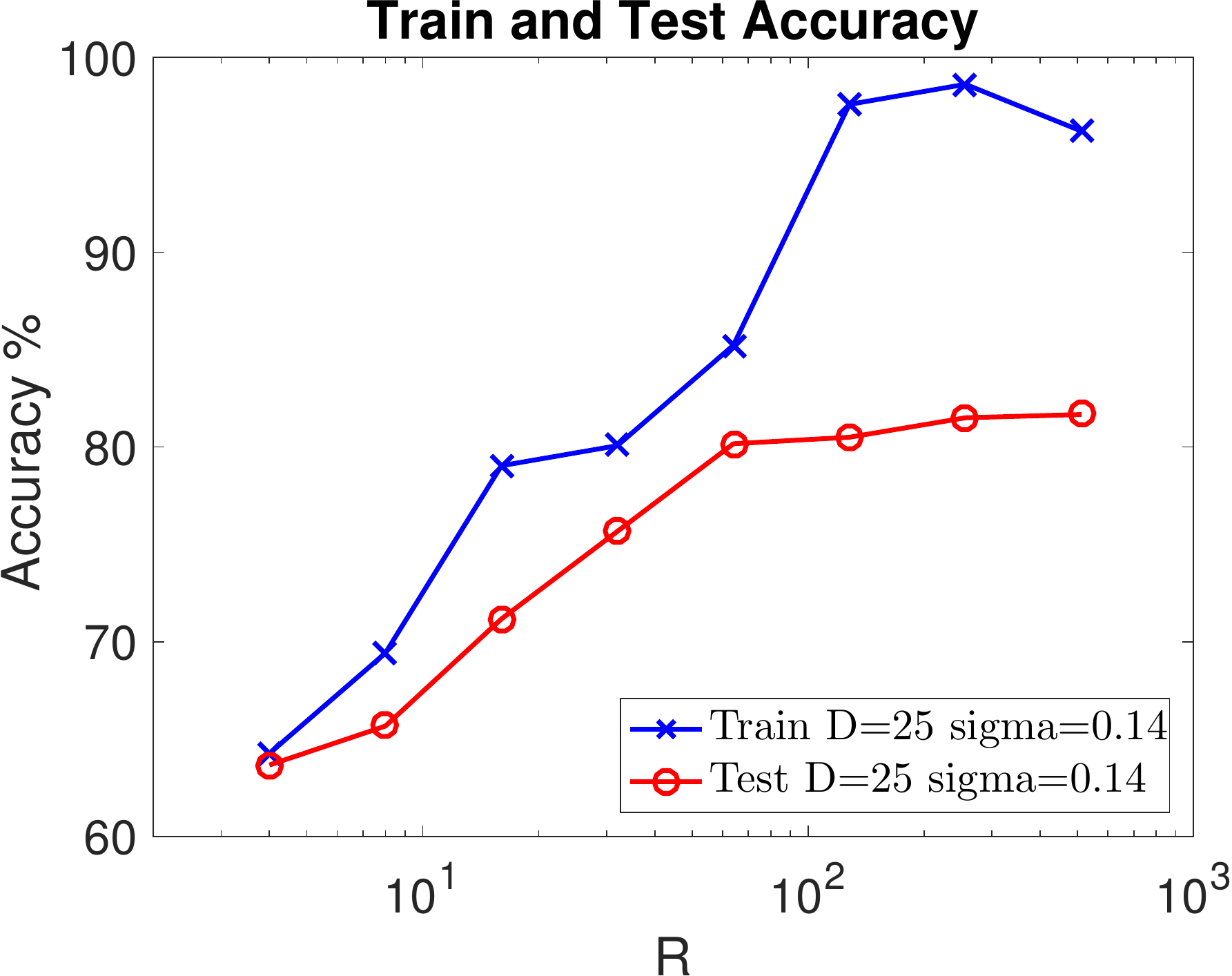}
      \caption{MPOC}
      \label{fig:exptsA_varyingR_MiddlePhalanxOutlineCorrect}
    \end{subfigure}
  \begin{subfigure}[b]{0.24\textwidth}
      \includegraphics[width=\textwidth]{Graphs/ExptsA_varyingR/NonInvasiveFatalECG_Thorax2_Accu_VaryingR-eps-converted-to.pdf}
      \caption{NIFECG}
      \label{fig:exptsA_varyingR_NonInvasiveFatalECG_Thorax2}
    \end{subfigure}
  \begin{subfigure}[b]{0.24\textwidth}
      \includegraphics[width=\textwidth]{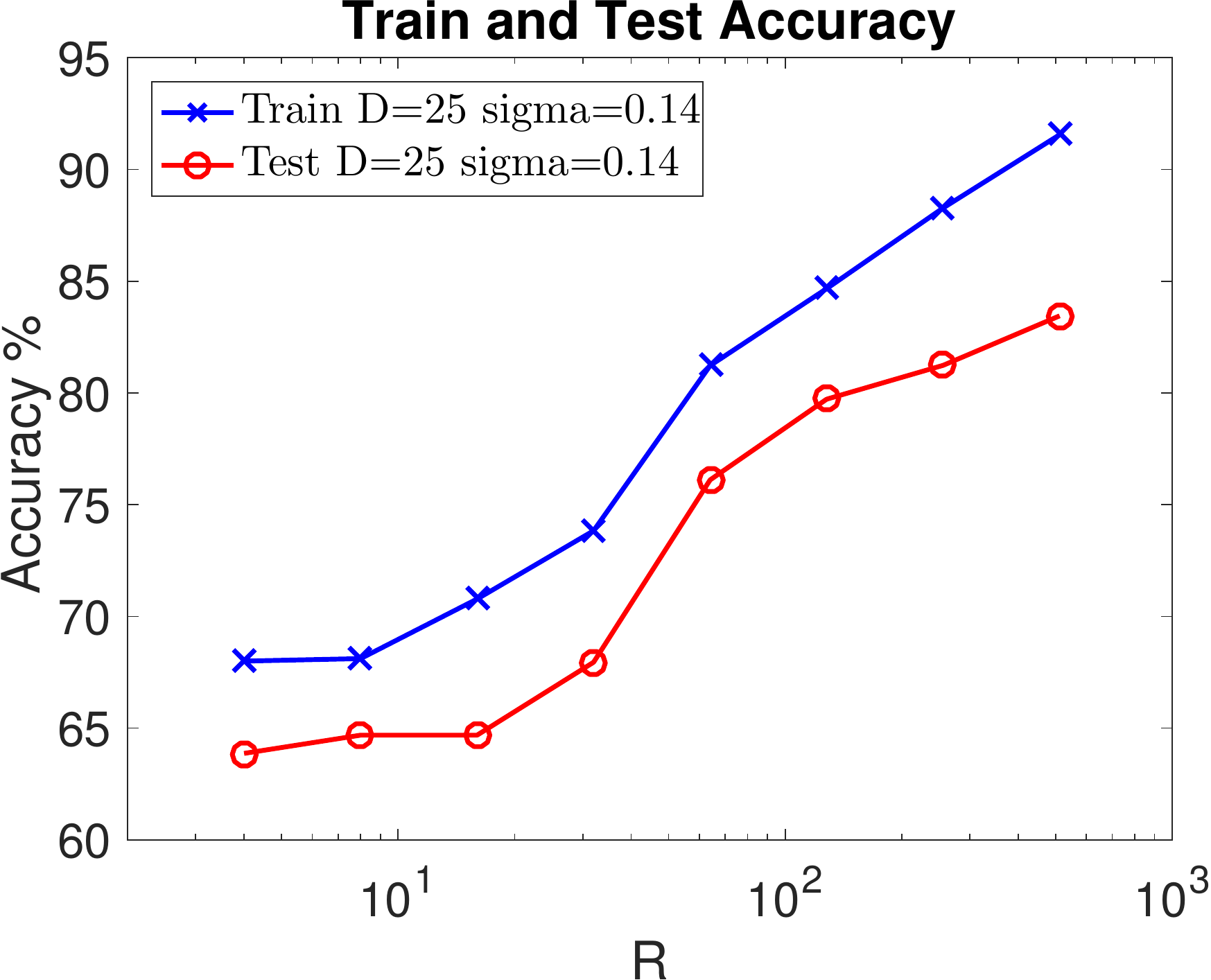}
      \caption{POC}
      \label{fig:exptsA_varyingR_PhalangesOutlinesCorrect}
    \end{subfigure}
  \begin{subfigure}[b]{0.24\textwidth}
      \includegraphics[width=\textwidth]{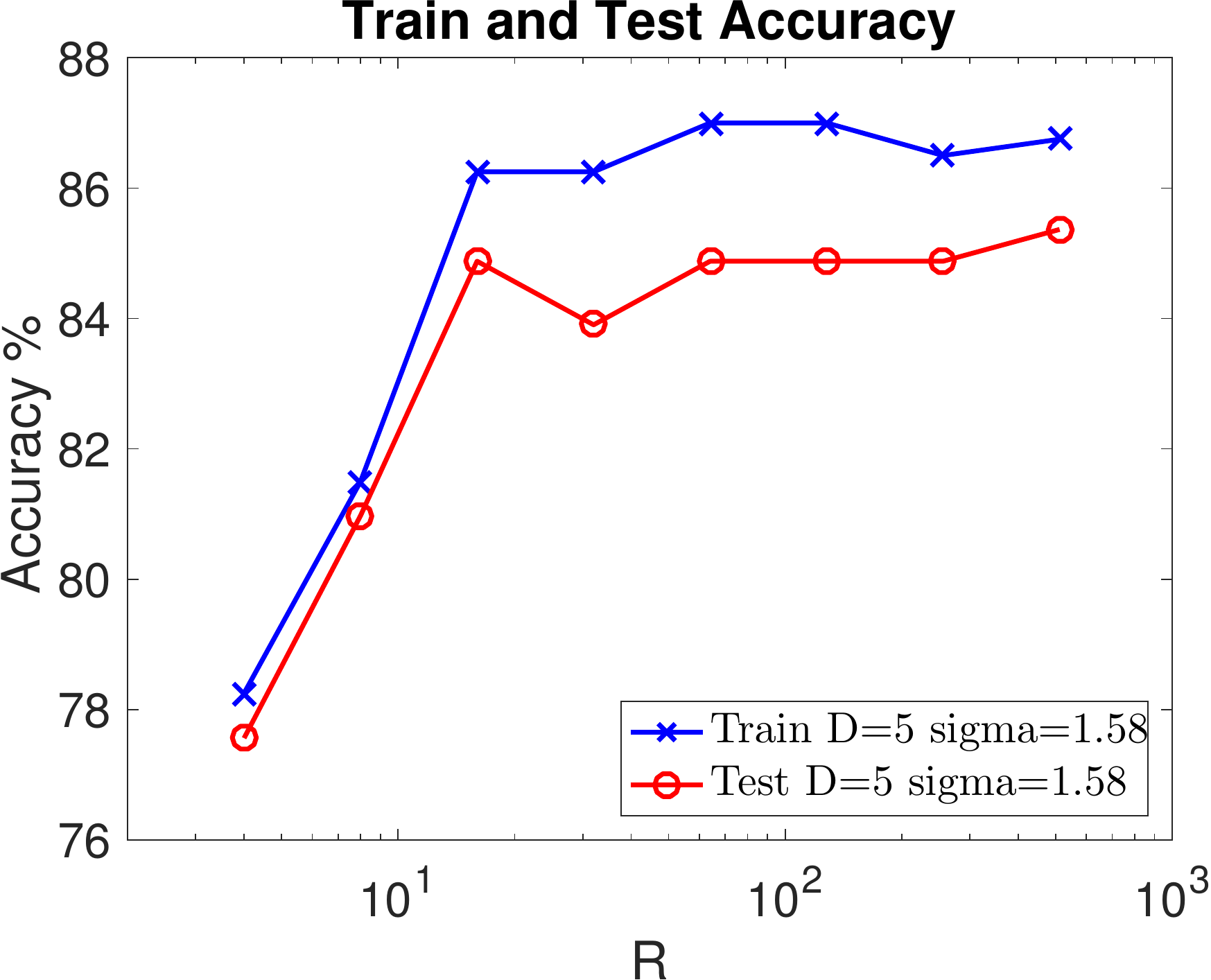}
      \caption{PPOAG}
      \label{fig:exptsA_varyingR_ProximalPhalanxOutlineAgeGroup}
    \end{subfigure}
  \begin{subfigure}[b]{0.24\textwidth}
      \includegraphics[width=\textwidth]{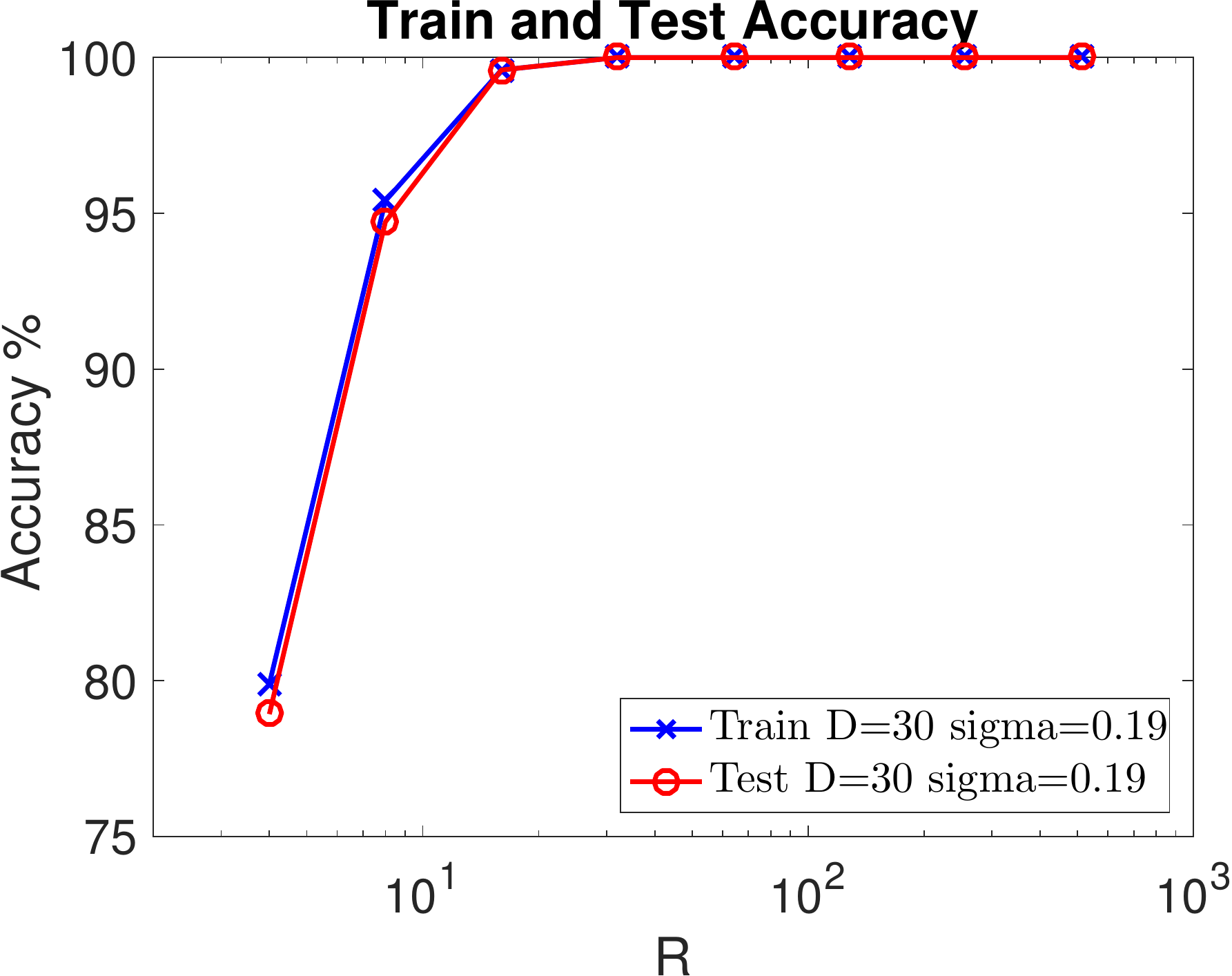}
      \caption{TWOP}
      \label{fig:exptsA_varyingR_Two_Patterns}
    \end{subfigure}
  \begin{subfigure}[b]{0.24\textwidth}
      \includegraphics[width=\textwidth]{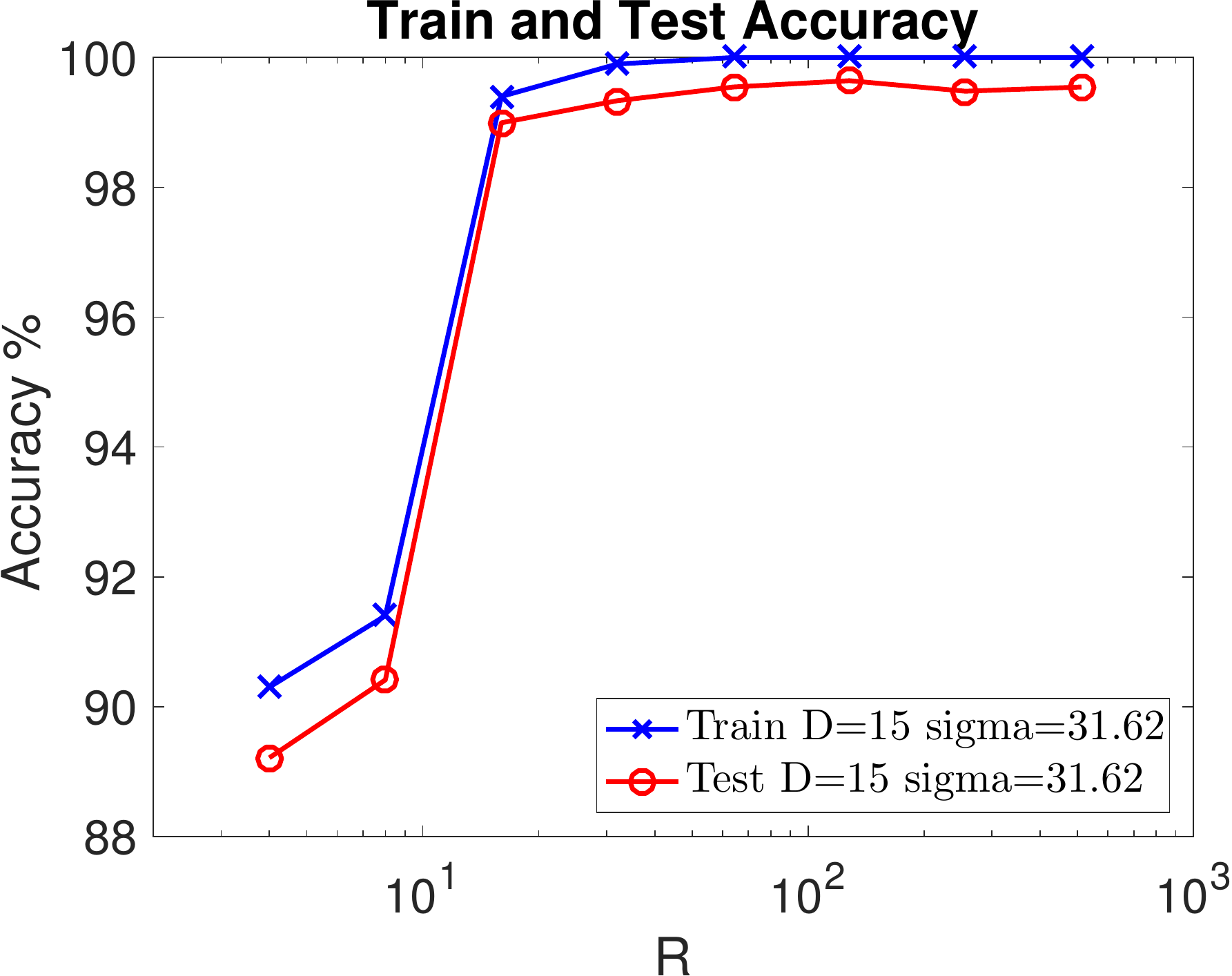}
      \caption{Wafer}
      \label{fig:exptsA_varyingR_wafer}
    \end{subfigure}
\vspace{-0mm}
\caption{Train (Blue) and test (Red) accuracy when varying $R$ with fixed $\sigma$ and $D$. We denote $D = DMax/2$.}
\vspace{-2mm}
\label{fig:exptsA_varyingR_sup}
\end{figure*}

\begin{figure*}[!htb]
\centering
  \begin{subfigure}[b]{0.24\textwidth}
      \includegraphics[width=\textwidth]{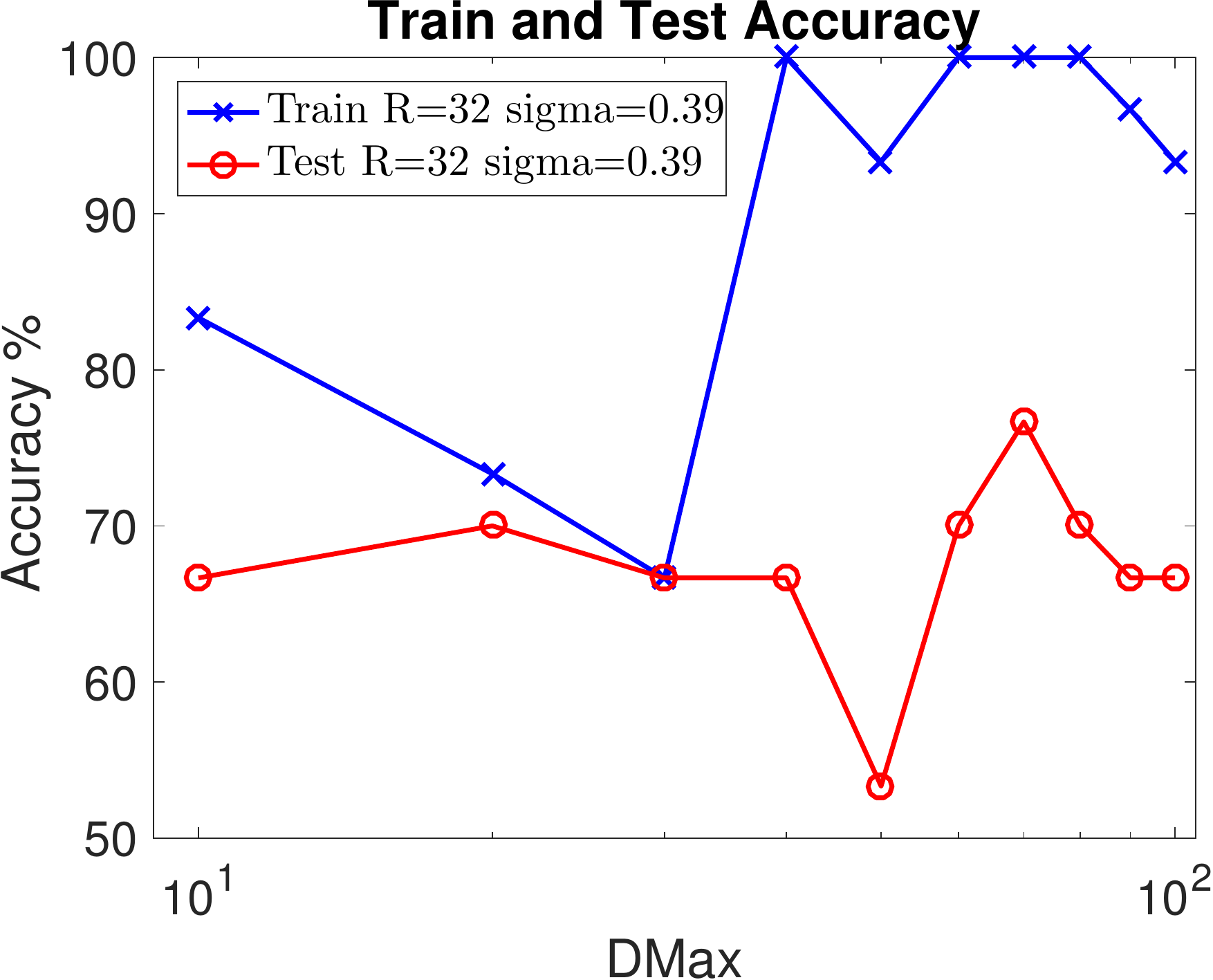}
      \caption{Beef}
      \label{fig:exptsA_varyingD_Beef}
    \end{subfigure}
  \begin{subfigure}[b]{0.24\textwidth}
      \includegraphics[width=\textwidth]{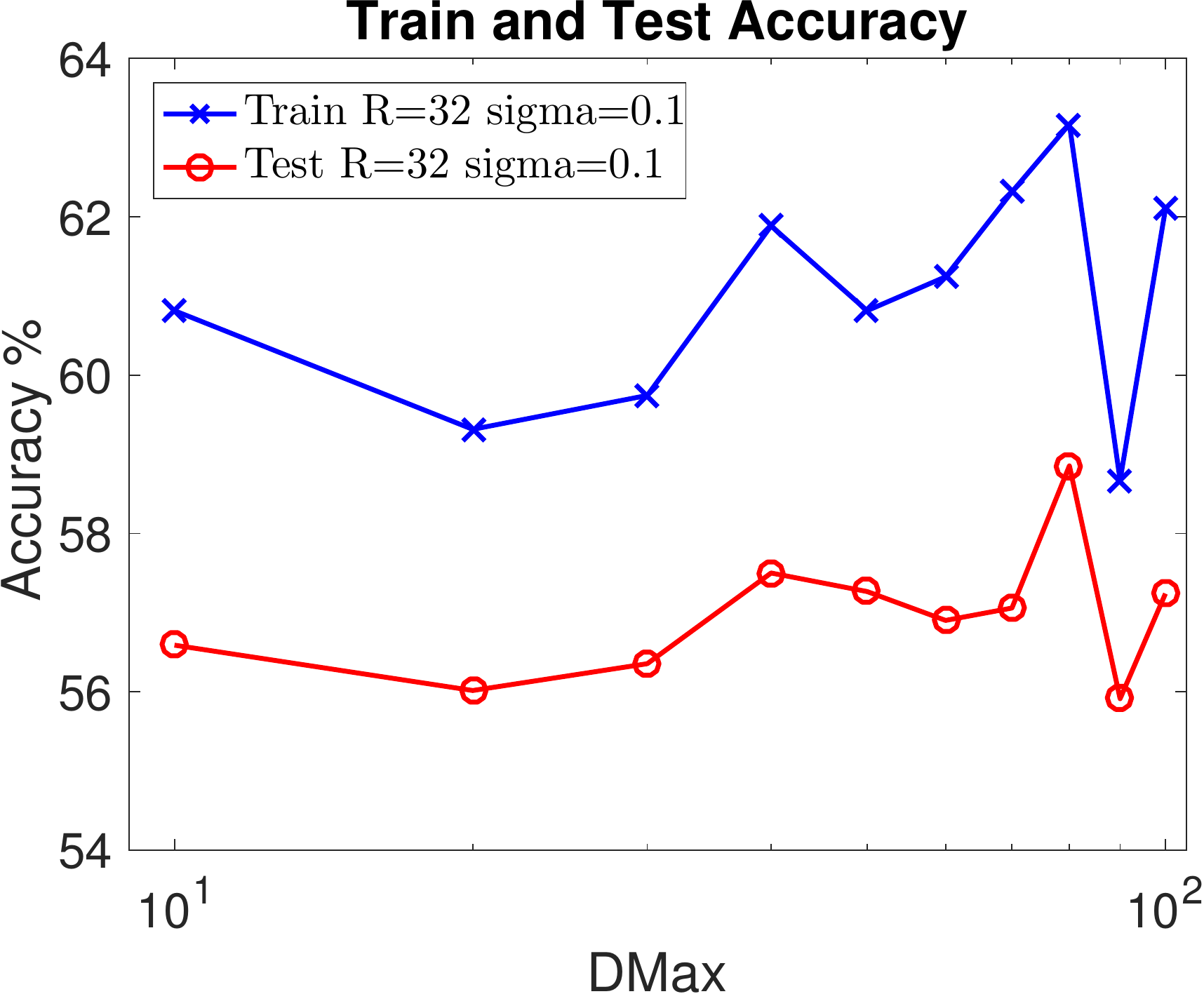}
      \caption{CHCO}
      \label{fig:exptsA_varyingD_ChlorineConcentration}
    \end{subfigure}
  \begin{subfigure}[b]{0.24\textwidth}
      \includegraphics[width=\textwidth]{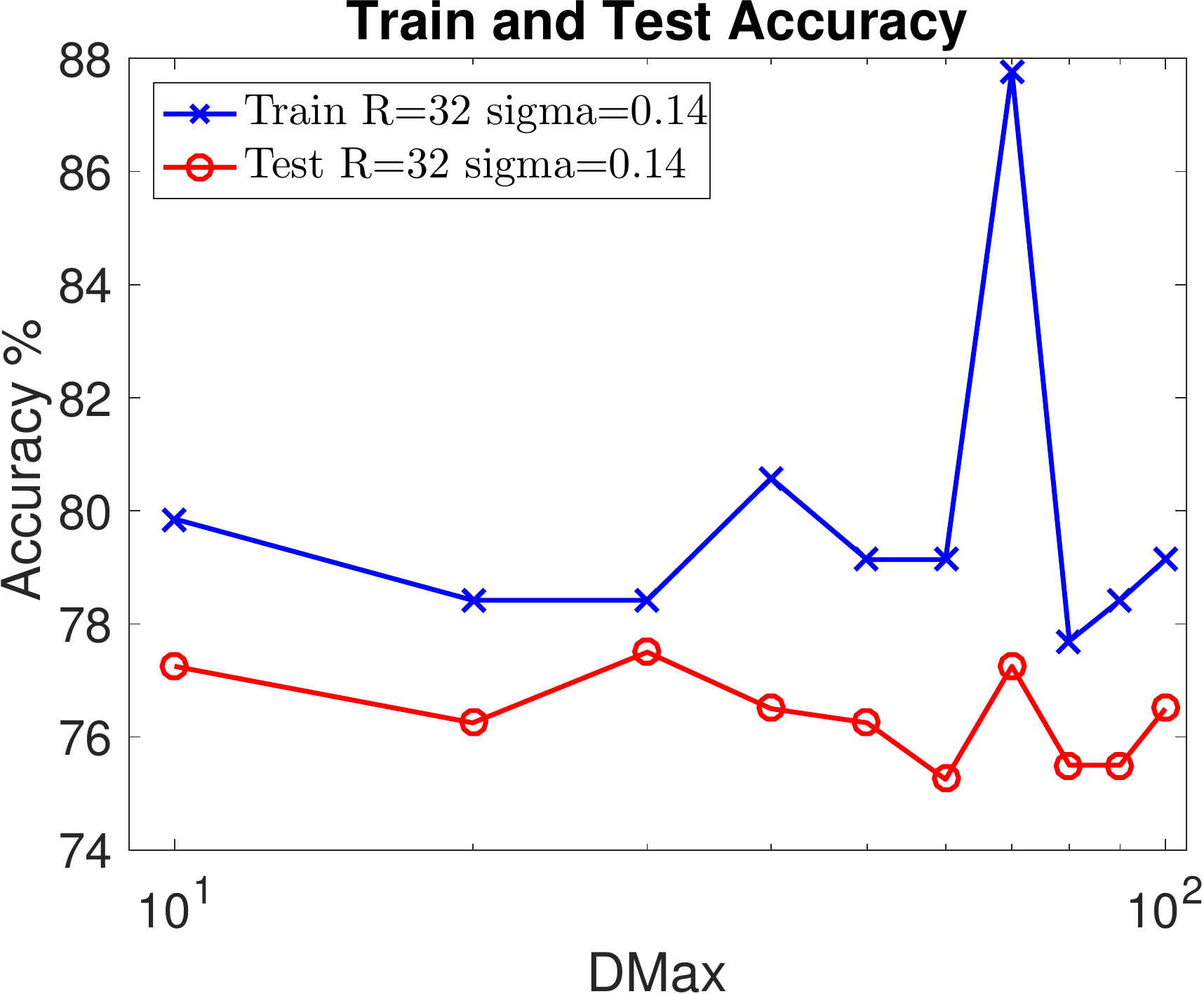}
      \caption{DPTW}
      \label{fig:exptsA_varyingD_DistalPhalanxTW}
    \end{subfigure}
  \begin{subfigure}[b]{0.24\textwidth}
      \includegraphics[width=\textwidth]{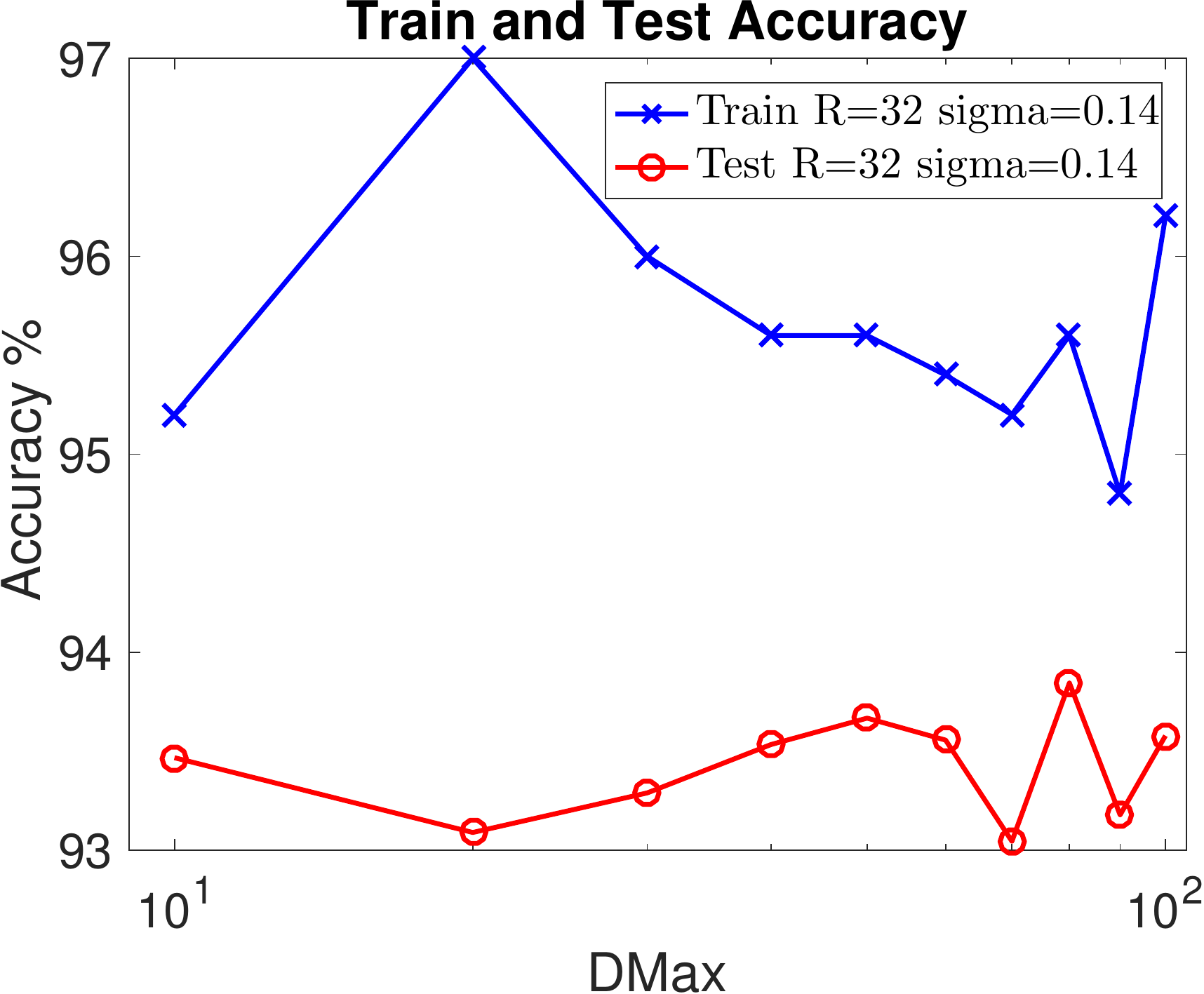}
      \caption{ECG5T}
      \label{fig:exptsA_varyingD_ECG5000}
    \end{subfigure}
  \begin{subfigure}[b]{0.24\textwidth}
      \includegraphics[width=\textwidth]{Graphs/ExptsA_varyingD/FordB_Accu_VaryingD-eps-converted-to.pdf}
      \caption{FordB}
      \label{fig:exptsA_varyingD_FordB}
    \end{subfigure}
  \begin{subfigure}[b]{0.24\textwidth}
      \includegraphics[width=\textwidth]{Graphs/ExptsA_varyingD/HandOutlines_Accu_VaryingD-eps-converted-to.pdf}
      \caption{HO}
      \label{fig:exptsA_varyingD_HandOutlines}
    \end{subfigure}
    \begin{subfigure}[b]{0.24\textwidth}
      \includegraphics[width=\textwidth]{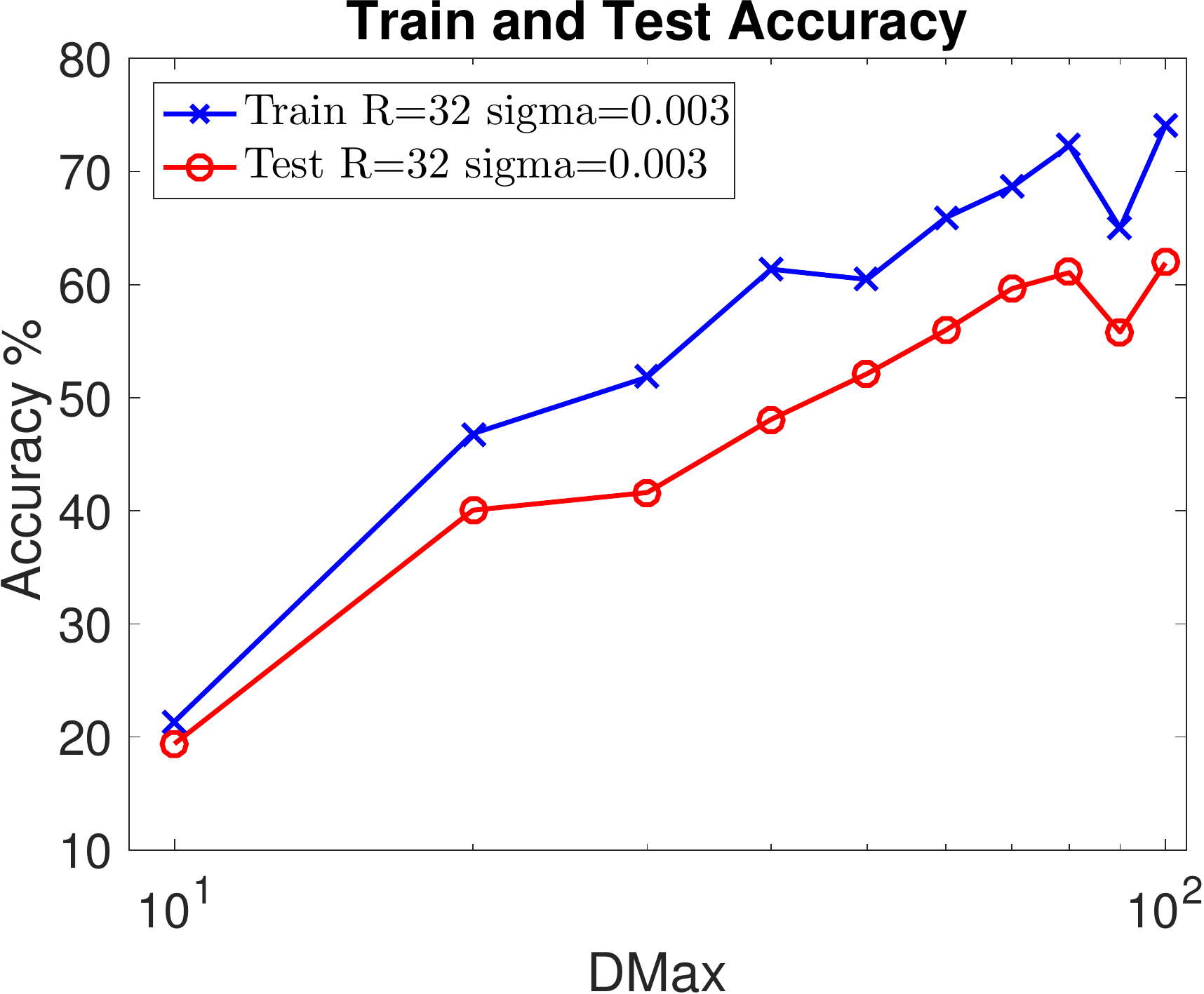}
      \caption{IWBS}
      \label{fig:exptsA_varyingD_InsectWingbeatSound}
    \end{subfigure}
  \begin{subfigure}[b]{0.24\textwidth}
      \includegraphics[width=\textwidth]{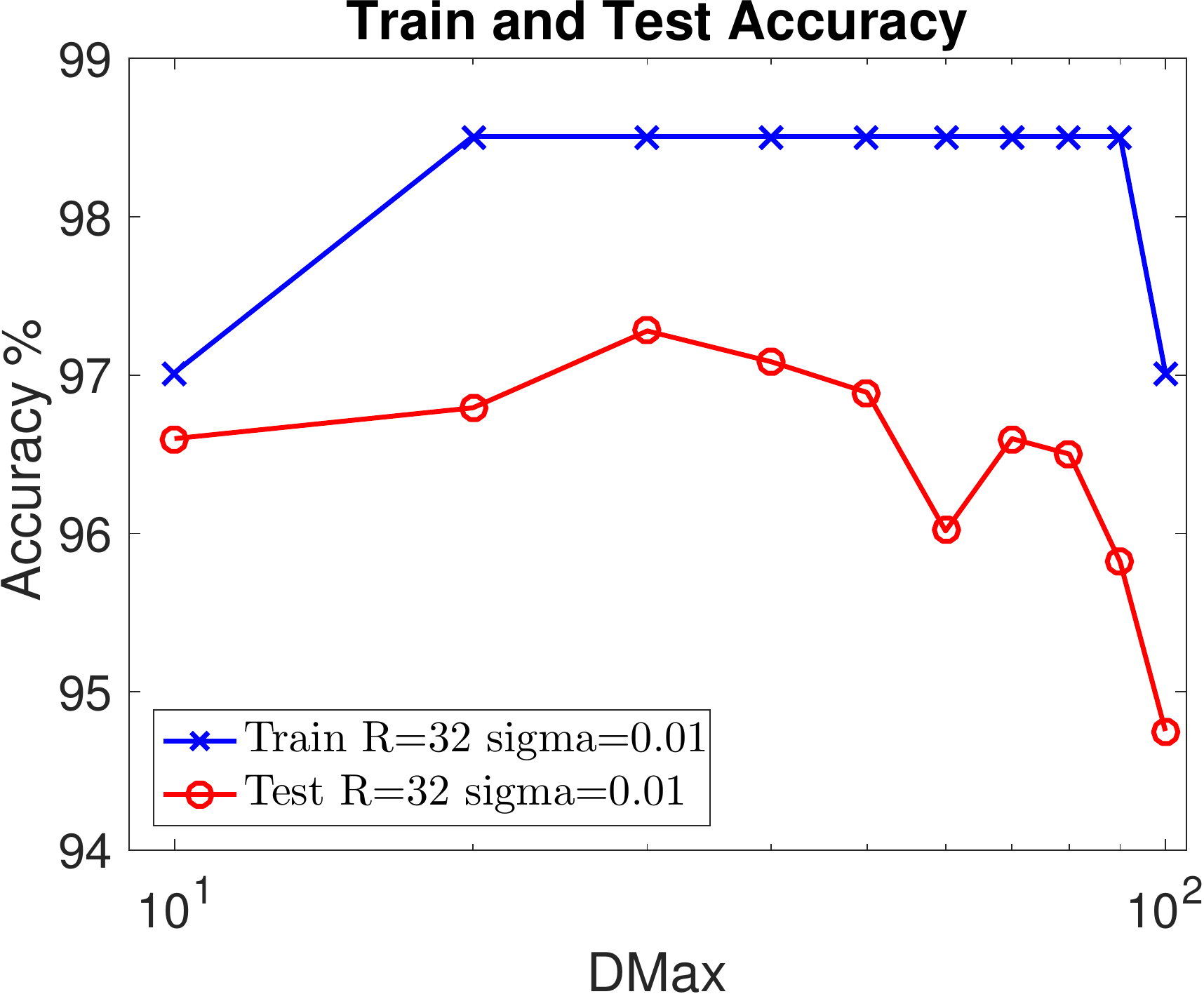}
      \caption{IPD}
      \label{fig:exptsA_varyingD_ItalyPowerDemand}
    \end{subfigure}
  \begin{subfigure}[b]{0.24\textwidth}
      \includegraphics[width=\textwidth]{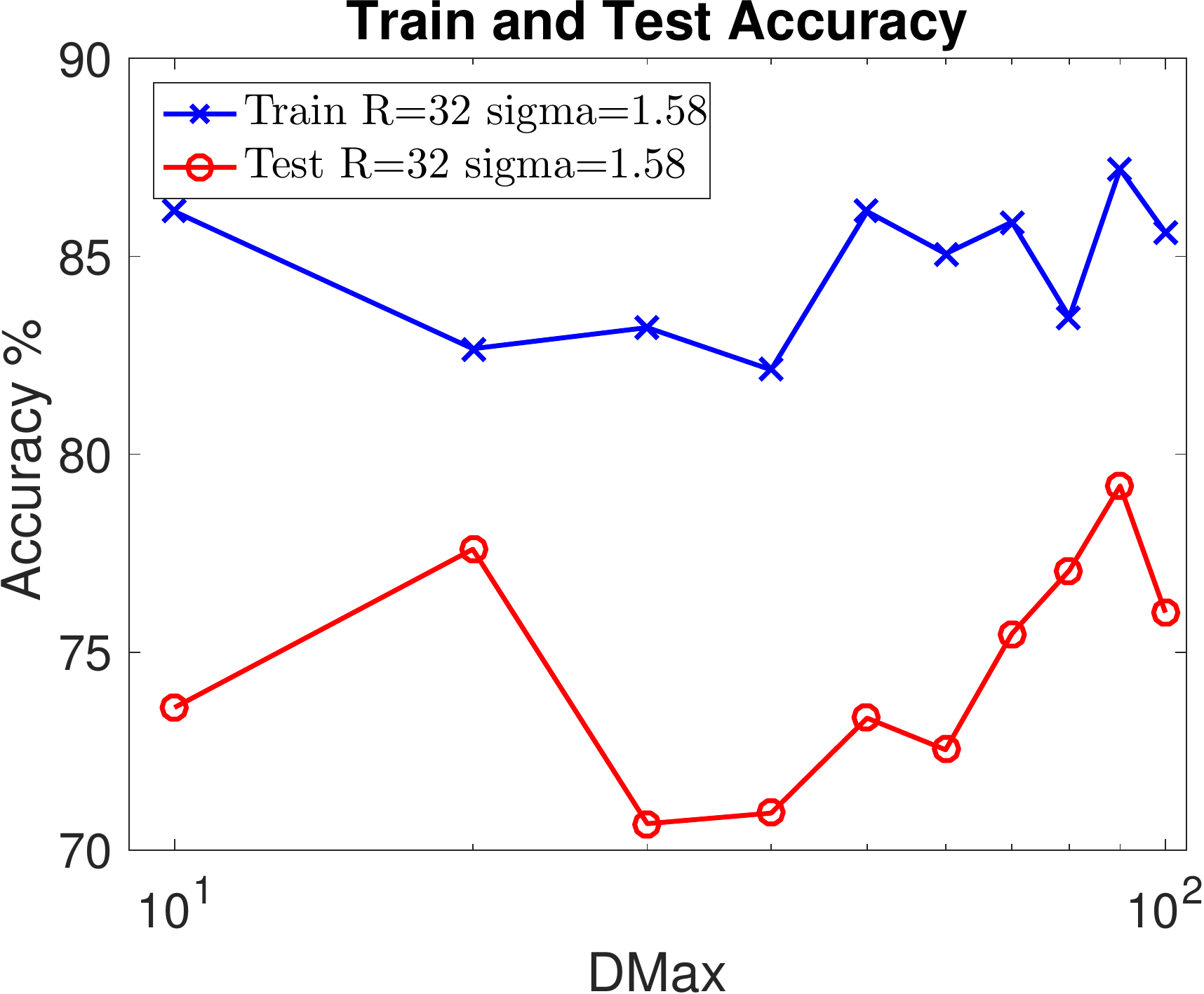}
      \caption{LKA}
      \label{fig:exptsA_varyingD_LargeKitchenAppliances}
    \end{subfigure}
  \begin{subfigure}[b]{0.24\textwidth}
      \includegraphics[width=\textwidth]{Graphs/ExptsA_varyingD/MALLAT_Accu_VaryingD-eps-converted-to.pdf}
      \caption{MALLAT}
      \label{fig:exptsA_varyingD_MALLAT}
    \end{subfigure}
  \begin{subfigure}[b]{0.24\textwidth}
      \includegraphics[width=\textwidth]{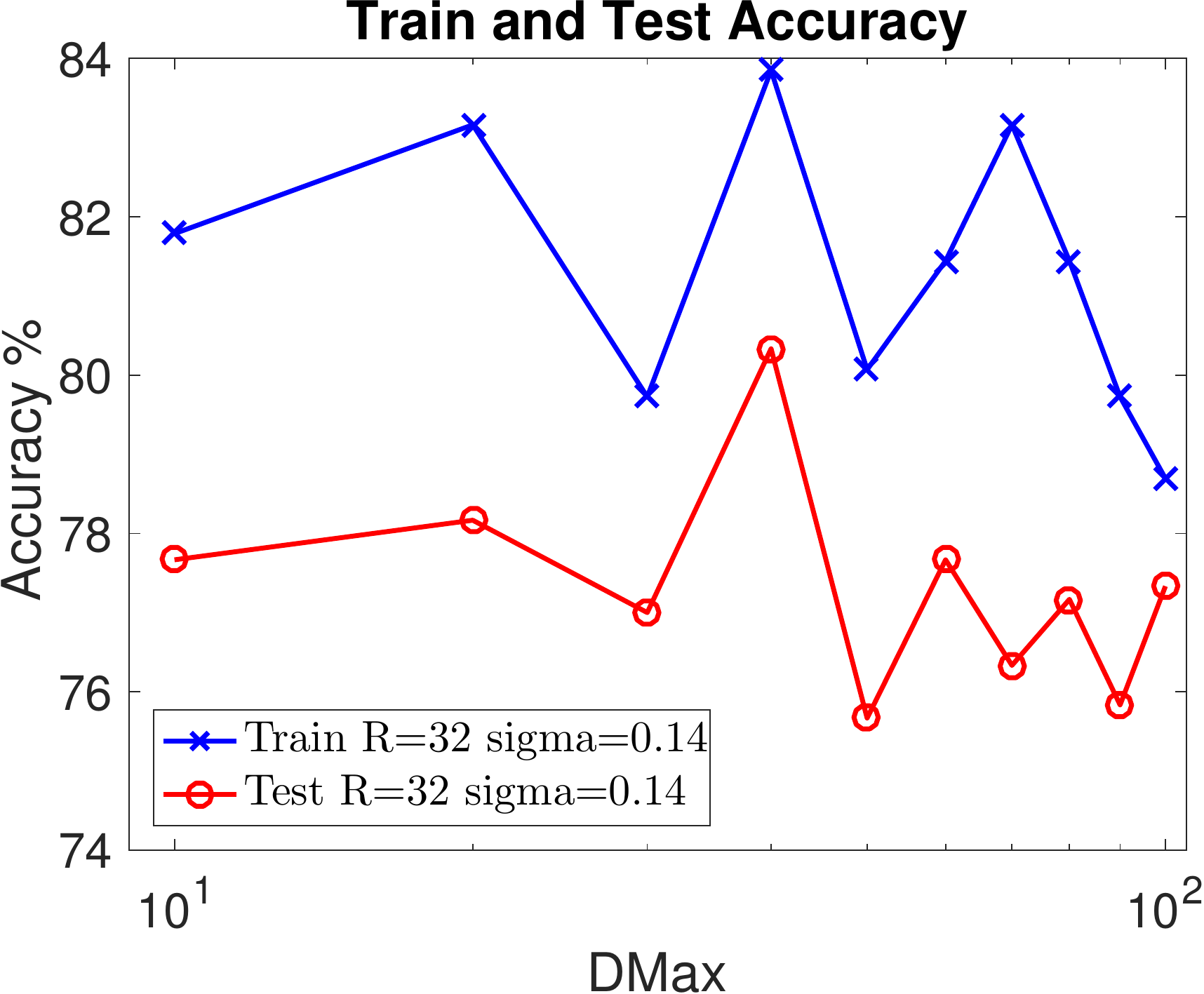}
      \caption{MPOC}
      \label{fig:exptsA_varyingD_MiddlePhalanxOutlineCorrect}
    \end{subfigure}
  \begin{subfigure}[b]{0.24\textwidth}
      \includegraphics[width=\textwidth]{Graphs/ExptsA_varyingD/NonInvasiveFatalECG_Thorax2_Accu_VaryingD-eps-converted-to.pdf}
      \caption{NIFECG}
      \label{fig:exptsA_varyingD_NonInvasiveFatalECG_Thorax2}
    \end{subfigure}
  \begin{subfigure}[b]{0.24\textwidth}
      \includegraphics[width=\textwidth]{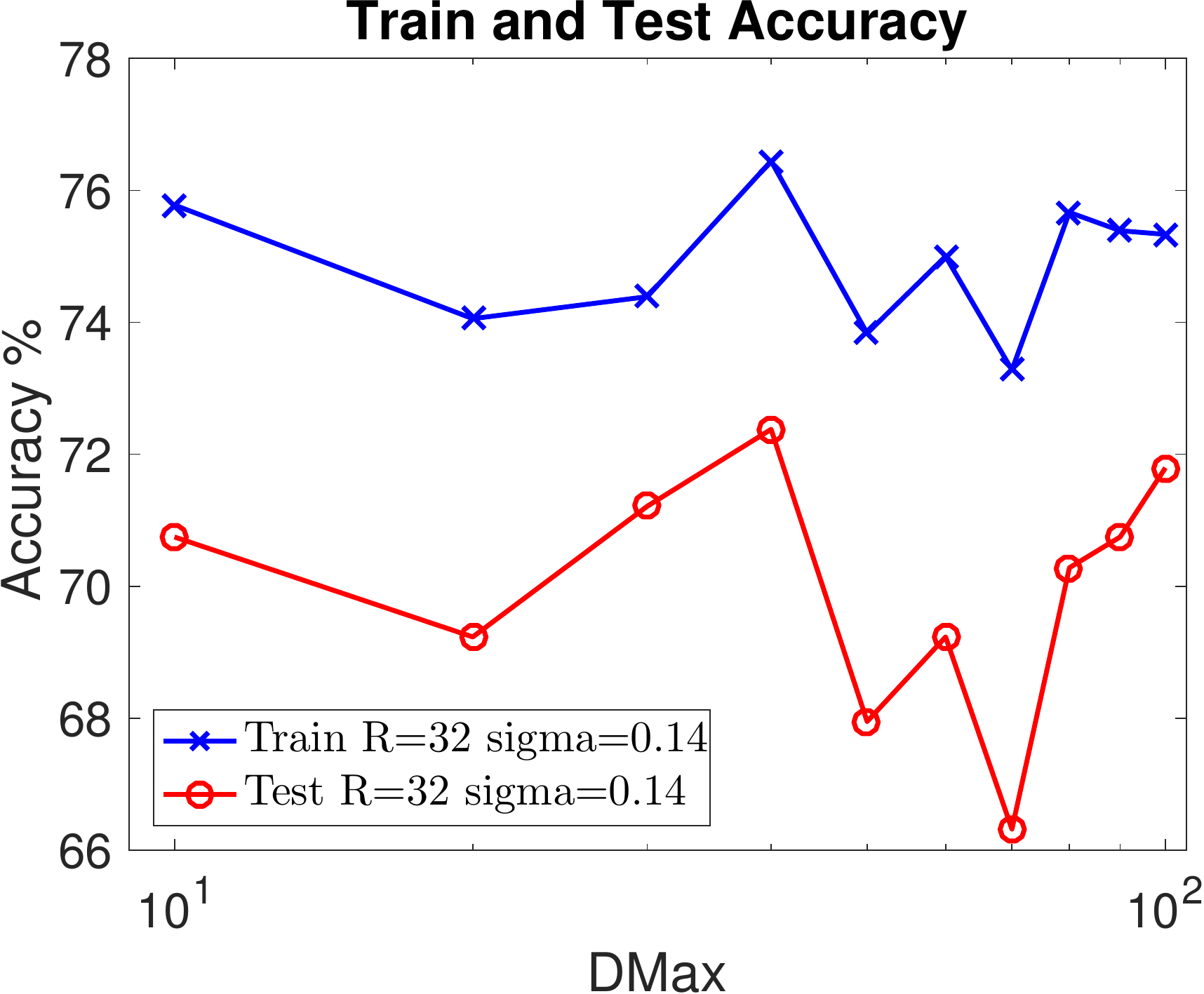}
      \caption{POC}
      \label{fig:exptsA_varyingD_PhalangesOutlinesCorrect}
    \end{subfigure}
  \begin{subfigure}[b]{0.24\textwidth}
      \includegraphics[width=\textwidth]{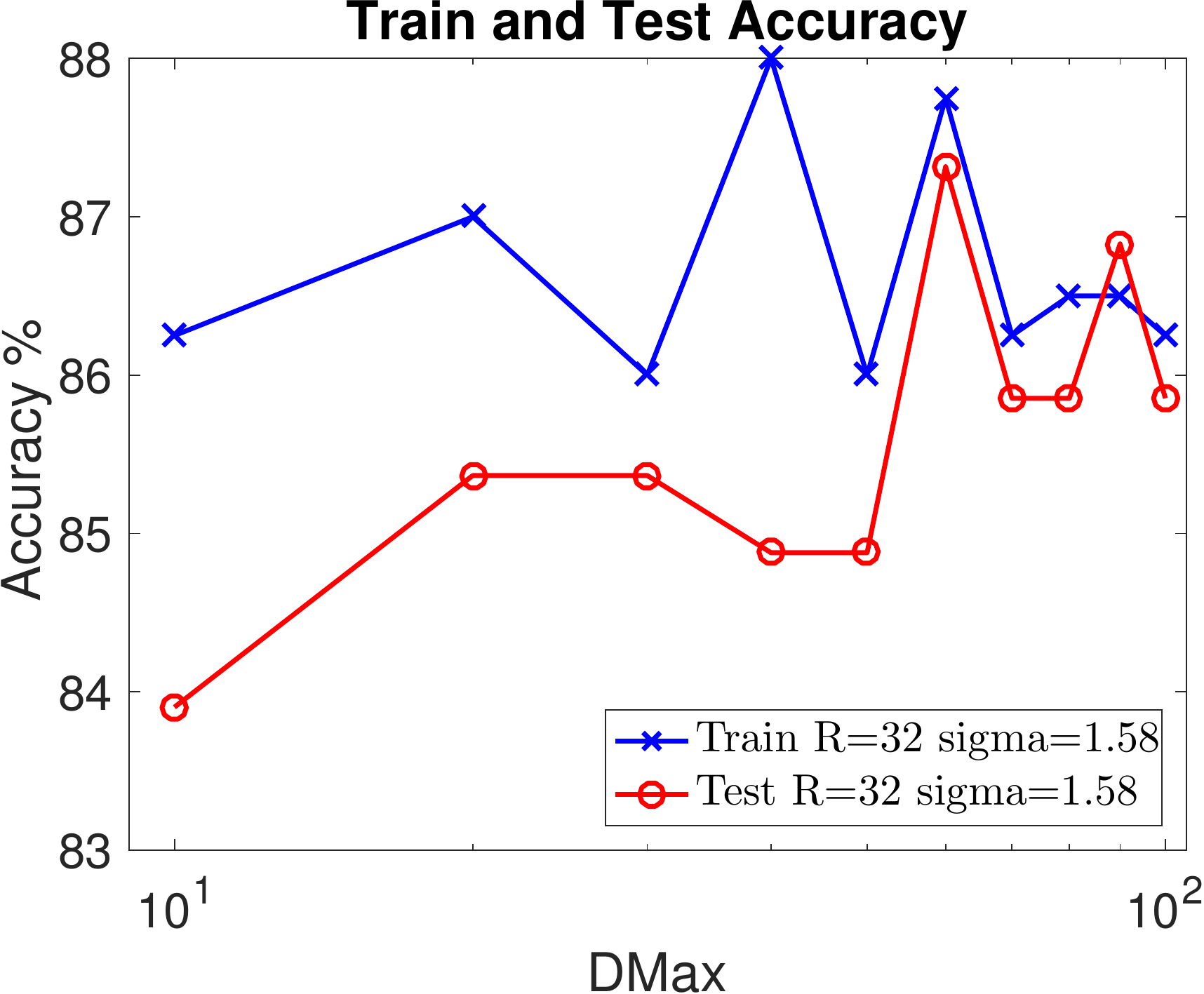}
      \caption{PPOAG}
      \label{fig:exptsA_varyingD_ProximalPhalanxOutlineAgeGroup}
    \end{subfigure}
  \begin{subfigure}[b]{0.24\textwidth}
      \includegraphics[width=\textwidth]{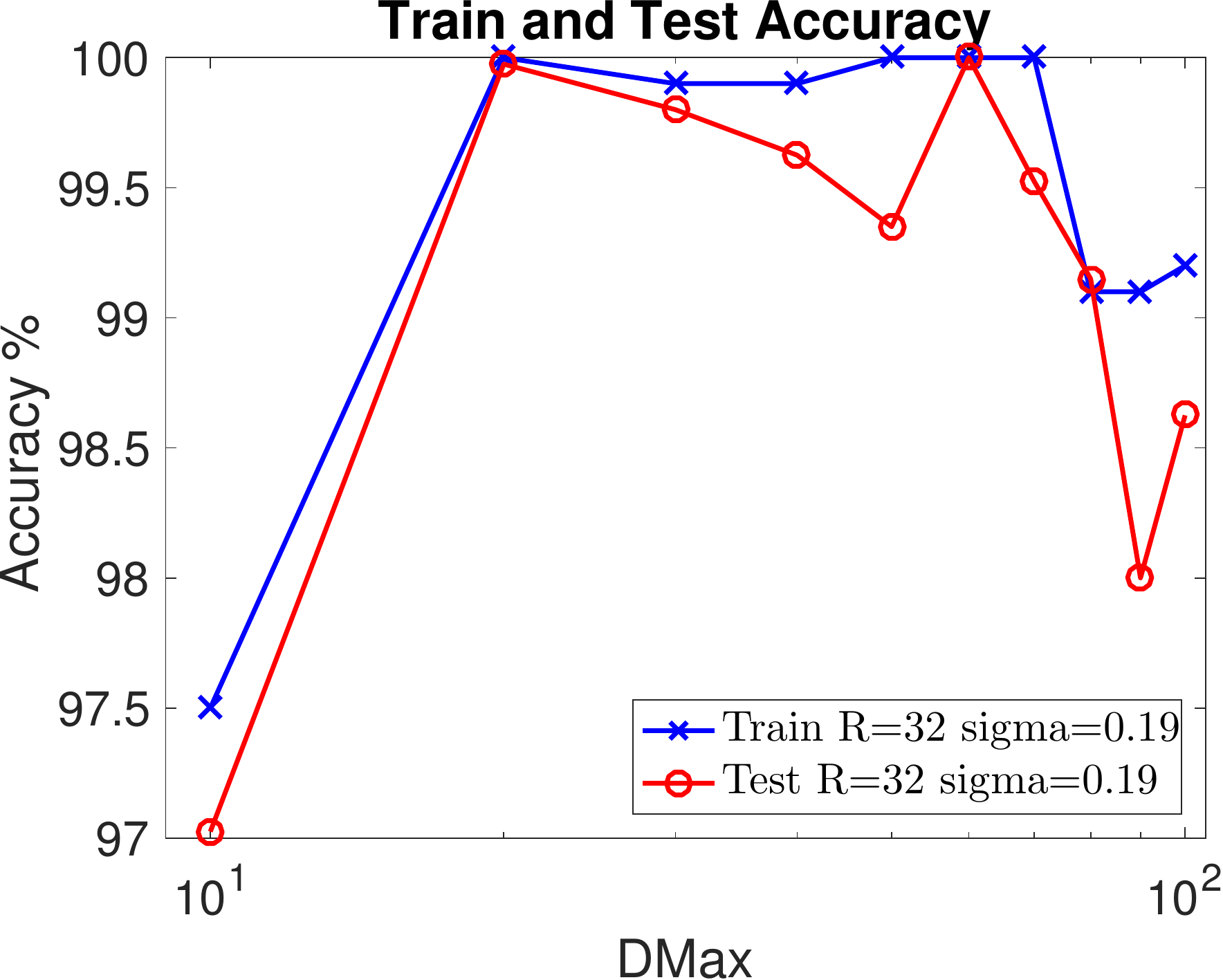}
      \caption{TWOP}
      \label{fig:exptsA_varyingD_Two_Patterns}
    \end{subfigure}
  \begin{subfigure}[b]{0.24\textwidth}
      \includegraphics[width=\textwidth]{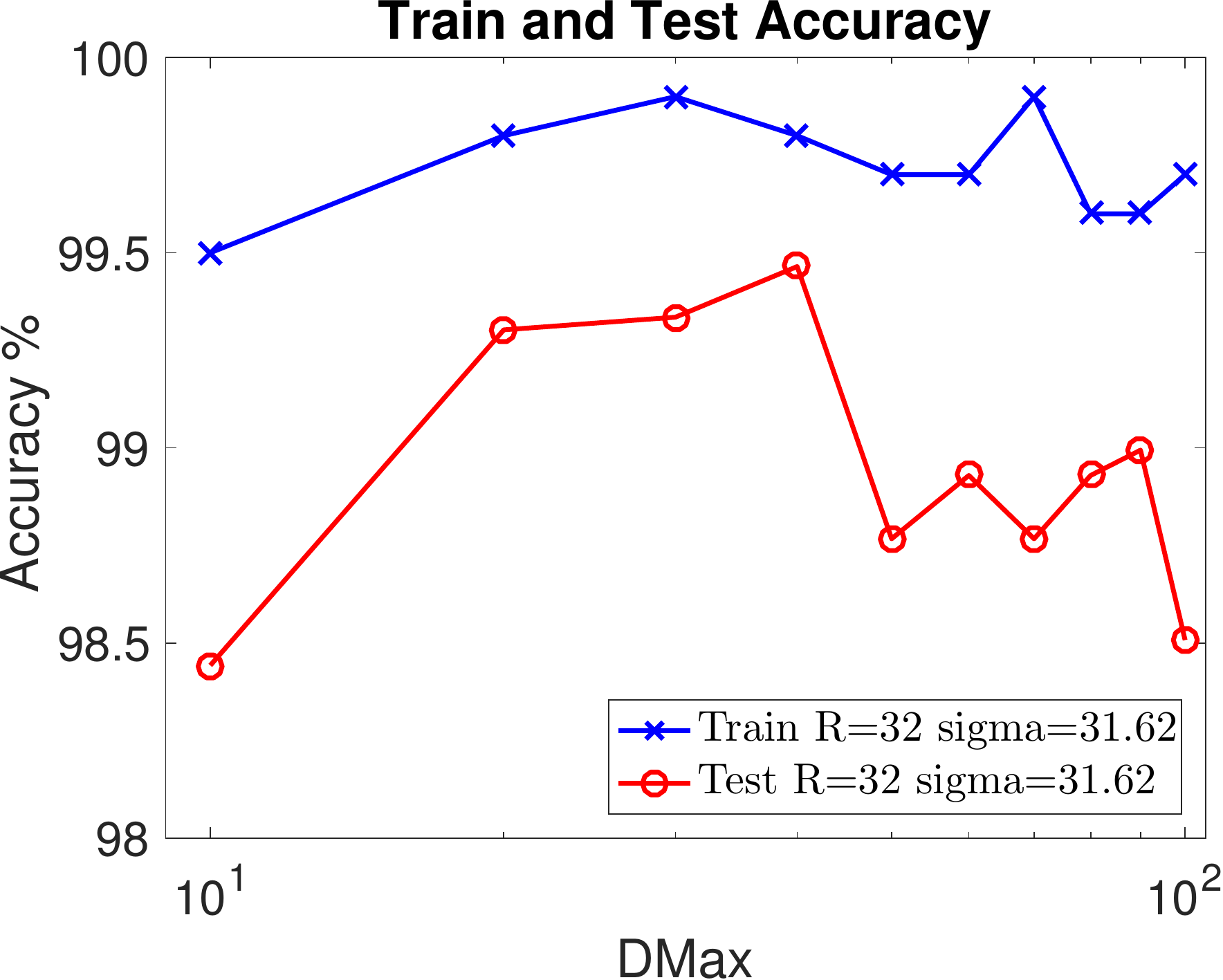}
      \caption{Wafer}
      \label{fig:exptsA_varyingD_wafer}
    \end{subfigure}
\vspace{-0mm}
\caption{Train (Blue) and test (Red) accuracy when varying $D$ with fixed $\sigma$ and $R$. We denote $D = DMax/2$.}
\vspace{-2mm}
\label{fig:exptsA_varyingD_sup}
\end{figure*}

\subsection{Parameters and Settings on Comparisons of Feature Representations}
\label{sec:Parameters and Settings on Comparisons of Feature Representations}
For TSEigen \cite{hayashi2005embedding}, we implemented this method in Matlab where we apply SVD to compute $R$ number of largest dominant components on the similar matrix computed using DTW. For TSMC \cite{QiYi2016}, we used their open source in code in Github: \url{https://github.com/cecilialeiqi/SPIRAL}. Since the default rank size of TSMC is 32, we keep all methods consistent with this setting to make a fair comparison. For all methods, we choose the parameter $C$ by 10-fold cross validation on training data in LIBLINEAR on all 16 datasets. 

\subsection{Parameters and Settings on  Comparisons for Large-Scale Classification}
\label{sec:Parameters and Settings on  Comparisons for Large-Scale Classification}
For 1NN-DTW and 1NN-DTW\textsuperscript{opt}, we implemented them using Matlab internal fitcknn with DTW using the same C Mex file \footnote{\url{https://www.mathworks.com/matlabcentral/fileexchange/43156-dynamic-time-warping–dtw-}} as our method RWS. Although our implementations may not be highly optimized, we believe the runtime comparisons among these methods are reasonably fair. For DTWF \cite{kate2016using}, we used their open source code \footnote{\url{https://people.uwm.edu/katerj/timeseries/}}. To make a fair comparison with other methods, we set the window size as $min(L/10, 40)$. The feature representation generated by DTWF combines SAX, DTW, and DTW\_R where we use recommended parameter ranges $n$ = [8 16 24 32 40 48 56 64 72 80 96 112 128 144 160], $w$ = [4 8], and $a$ = [3 4 5 6 7 8 9] for cross validation. For TGAK \cite{cuturi2011fast}, we took their open source code \footnote{\url{http://marcocuturi.net/GA.html}} for the experiments. We choose recommended window size $T = 0.25$ due to a good trade off between testing accuracy and computational time. We also perform cross validation to search for good kernel parameter $\sigma$ in the range of [0.01, 0.033, 0.066, 0.1, 0.33, 0.66, 1, 3.3, 6.6, 10] and the LIBLINEAR parameter $C$ in the range of [1e-5 1e-4 1e-3 1e-2 1e-1 1 1e1 1e2 1e3 1e4 1e5 1e6]. 

\subsection{Parameters and Settings on  Comparisons for Large-Scale Clustering}
\label{sec:Parameters and Settings on  Comparisons for Large-Scale Clustering}
For KMeans-DTW \cite{petitjean2011global}, we used the public available python code \footnote{\url{https://github.com/alexminnaar/time-series-classification-and-clustering}}, which also implements LB\_Keogh lower bound with DTW. However, the efficiency of python code may be significantly worse than C mex file of DTW we used, which could be the reason we observed larger margin speedup compared to 1NN-DTW. Nevertheless, note that the computational complexity of RWS over Kmeans-DTW reduces from quadratic complexity to linear complexity. For CLDS \cite{li2011time}, we used the open source code published by authors \footnote{\url{http://www.cs.cmu.edu/~./leili/software.html}}. We choose the parameter $C$ by cross validation while using recommended parameters for generating the representations on all datasets. For K-Shape \cite{paparrizos2015k}, we used the public available python code \footnote{\url{https://github.com/Mic92/kshape}}. Similarly, we choose the parameter $C$ by cross validation while using recommended parameters for generating the representations on all datasets.



\end{document}